\documentclass[11pt,letterpaper]{article}
\usepackage{booktabs}
\usepackage{amsmath,amsfonts,amsthm,amssymb}
\usepackage{setspace}
\usepackage{fancyhdr}
\usepackage{lastpage}
\usepackage{extramarks}
\usepackage{xspace}
\usepackage{chngpage}
\usepackage{nicefrac}
\usepackage{subcaption}
\usepackage[ruled,linesnumbered,vlined]{algorithm2e} 

\SetAlFnt{\small}
\SetAlCapFnt{\small}
\SetAlCapNameFnt{\small}
\SetAlCapHSkip{0pt}
\IncMargin{-\parindent}
\usepackage{soul,color}
\usepackage{graphicx,float,wrapfig}
\usepackage[font=small,labelfont=bf]{caption}
\usepackage[margin=1in]{geometry}
\usepackage{enumitem}
\usepackage{float}
\usepackage{xspace}
\usepackage{thmtools,thm-restate}

\setlist{nosep}

\allowdisplaybreaks

\usepackage{mdframed}
\newtheorem*{mdresult}{Result}

\newcommand{\argmax}{\arg\!\max}

\newcommand{\real}{\mathbb{R}}
\newcommand{\opt}{\mathsf{Opt}}
\newcommand{\cp}{\mathsf{CP}}
\newcommand{\pre}{\mathsf{Pre}}

\newcommand{\subr}{\mathsf{SubAlg}}

\newcommand{\alg}{\mathsf{Alg}}
\newcommand{\rev}{\mathsf{Rev}}
\newcommand{\loss}{\mathsf{Loss}}

\newcommand{\lcb}{\ensuremath{\ell}}
\newcommand{\ucb}{\ensuremath{r}}
\newcommand{\D}{\mathcal{D}}

\newcommand{\se}{{T^{-100}}}
\newcommand{\epsprm}{\delta}

\newcommand{\OTild}{\widetilde{O}}
\newcommand{\one}{\mathbf{1}}%
\newcommand{\Ex}[2][]{\mbox{\rm\bf E}_{#1}\left[#2\right]}%
\renewcommand{\Pr}[2][]{\mbox{\rm\bf Pr}_{#1}\left[#2\right]}%

\usepackage[usenames,dvipsnames]{xcolor}
\usepackage[linktocpage=true,breaklinks,colorlinks,citecolor=blue,linkcolor=BrickRed]{hyperref}
\usepackage{cleveref}
\newcommand{\IGNORE}[1]{}
\newcommand{\poly}{\ensuremath{\mathsf{poly}}}
\newcounter{note}[section]

\title{\textmd{\bf Bandit Sequential Posted Pricing via Half-Concavity }}
\date{\today}

\author{ 
    Sahil Singla\thanks{
        (ssingla@gatech.edu)
        School of Computer Science,
        Georgia Tech. Supported in part by NSF award CCF-2327010.
        }
    \and
    Yifan Wang\thanks{
        (ywang3782@gatech.edu)
        School of Computer Science,
        Georgia Tech. Supported in part by NSF award CCF-2327010.
        }
}

\newtheorem{Theorem}{Theorem}[section]
\newtheorem{main}{Main Result}
\newtheorem{Lemma}[Theorem]{Lemma}

\newtheorem{Definition}[Theorem]{Definition}

\newtheorem{Claim}[Theorem]{Claim}
\newtheorem{Corollary}[Theorem]{Corollary}

\begin{document}



\maketitle 

\begin{abstract}
\medskip 
Sequential posted pricing auctions are popular  because of their simplicity in practice and their tractability in theory.  A usual assumption in their study is that the Bayesian prior distributions of the buyers are known to the seller, while in reality these priors can only be accessed from historical data. To overcome this assumption, 
we study  sequential posted pricing in the bandit learning model, where the seller interacts with  $n$ buyers over $T$ rounds: In each round the seller posts $n$ prices for the $n$ buyers and the first buyer with a valuation higher than the price takes the item. The only feedback that the seller receives in each round is the revenue. 

\smallskip

Our main results obtain nearly-optimal  regret bounds for single-item sequential posted pricing in the bandit learning model. In particular, we achieve an $\OTild(\poly(n)\sqrt{T})$ regret  for buyers with (Myerson's) regular distributions and an $\OTild(\poly(n)T^{{2}/{3}})$ regret for buyers with general distributions, both of which are tight in the number of rounds $T$. Our result for regular distributions was previously not known even for the single-buyer setting   and relies on a new \textit{half-concavity} property of the revenue function in the value space. For  $n$ sequential buyers, our technique is to run a generalized single-buyer algorithm for all the buyers and to carefully bound the regret from the sub-optimal pricing of the suffix buyers. 
\end{abstract}


\bigskip
\setcounter{tocdepth}{1}
{\small
\begin{spacing}{0}
\vspace{-0.5cm}
   \tableofcontents
\end{spacing}
}

\clearpage

\section{Introduction}

Sequential Posted Pricing (SPP) schemes are well-studied  in  mechanism design because of their simplicity/convenience in practice and their tractability in theory. In the basic setting, a seller wants to sell a single item to a group of buyers. The buyers arrive one-by-one and the seller presents a take-it-or-leave-it price. The first buyer with a value higher than the posted price takes the item by paying the price. The benefit of SPP is that it gives approximately-optimal revenue while being ``simple''. For instance, although it is known that Myerson's mechanism \cite{myerson1981optimal} is optimal for selling a single item to $n$ buyers with regular distributions, this mechanism is often impractical. This is because the winner's payment is defined via ``complicated'' virtual valuations and all  buyer bids need to be simultaneously revealed, which is not possible in large markets. In contrast, SPP are known to give at least $1/2$-fraction of the optimal revenue for  regular buyers, while having fixed prices and buyers arriving one-by-one {\color{black}(see the books \cite{roughgarden2017twenty,HartlineBook}  for regular buyers,  and \cite{DBLP:conf/soda/Yan11} for a discussion on how ironing  extends this result to general buyers). }
For multi-item (multi-parameters) setting, the difference between optimal and SPP mechanisms is even more stark. For instance,  the optimal mechanism to sell $n$ items to a single unit-demand bidder is known to be impossible (unless $P^{NP}=P^{\#P}$) \cite{ChenDOPSY15}, but we can use SPP to obtain $1/4$-fraction of the optimal revenue~\cite{CHK-EC07,CHMS-STOC10,CMS-Games15}. 



\medskip
\noindent \textbf{Sample Complexity.} A common assumption among initial works on revenue-maximization 
was that the underlying distributions of the buyers are known to the seller.  This  is unrealistic  in many  applications since the distributions are unknown and  need to be learnt  from historical data. Inspired by this, a  recent line of research  studies the sample complexity for different types of mechanisms, i.e., how many samples are sufficient to learn an $\epsilon$-optimal mechanism. 
For instance, starting with the pioneering work in \cite{DBLP:conf/stoc/ColeR14,BalcanBHM-JCSS08}, several papers   studied the sample complexity of Myerson's mechanism \cite{HMR-SICOMP18,DBLP:conf/stoc/Devanur0P16,DBLP:conf/sigecom/RoughgardenS16,DBLP:conf/stoc/GonczarowskiN17}, and finally \cite{GHZ-STOC19}  obtained the tight sample complexity bounds for regular buyers and for $[0, 1]$ bounded-support buyers. Sample complexity of many other classes of mechanisms have been studied, e.g., \cite{BSV-EC18,JLX-ITCS23} study  second-price auctions, \cite{BSV-EC18,BDDKSV-STOC21} study posted-pricing mechanisms, \cite{GHTZ-COLT21} study strongly-monotone auctions (which includes Myerson and SPP), and  \cite{MR-COLT16,CD-FOCS17,GW21,COZ-STOC22} study the sample complexity of  multi-parameter auctions.

\medskip
\noindent \textbf{Bandit Learning.} A stronger model of learning than the sample complexity model is the well-studied bandit learning model; see books \cite{LS-Book20,BubeckC-Book12}.  In this model, the seller interacts with the buyers over $T$ days. On each day, the seller proposes a parameterized mechanism to the buyer  and  sees only the revenue as  feedback. The goal is to minimize the total regret, which is the difference between the total revenue achieved by the learning algorithm and the  optimal  mechanism over $T$ days. It is stronger than   sample complexity  in two aspects: Firstly, it requires the learner to continually play near-optimal mechanisms, whereas in  sample complexity  the learner may lose a lot of the revenue while  learning. Secondly, the feedback in the bandit model on any day is limited to only the  proposed mechanism, and not every possible mechanism that could have been proposed.

The study of learning mechanisms with bandit feedback goes back to at least \cite{DBLP:conf/focs/KleinbergL03}, where the authors provide an $\OTild(T^{{2}/{3}})$ regret bound for learning  single-buyer posted pricing mechanism. Other examples of learning auctions in this model include 
bandits with knapsacks~\cite{BKS-JACM18,ISSS-JACM22}, 
  reserve price for i.i.d buyers \cite{CGM15}, and   non-anonymous reserve prices \cite{DBLP:conf/sigecom/NiazadehGWSB21}.

Despite a lot of work on learning 
SPP mechanisms and on learning auctions with bandit feedback,  learning SPP in the bandit  model was unknown. We answer this question and provide near-optimal regret bounds for SPP with $n$  buyers having regular or general distributions.


\subsection{Model and Results}
\label{sec:prelim}


In \emph{Sequential Posted Pricing} (SPP),  there are $n$ sequential buyers with independent valuation distributions $\D_1, \cdots, \D_n$. We make the standard normalization assumption that each distribution $\D_i$ is supported in $[0, 1]$. The seller plays a group of prices $(p_1, \cdots, p_n) \in [0, 1]^n$ and then the buyers arrive one-by-one with valuation $v_i \sim \D_i$. The first buyer with  $v_i \geq p_i$ takes the item and pays the price $p_i$ as the revenue.  The goal of the seller is to play prices to maximize the expected revenue.
We define $R(p_1, \cdots, p_n)$ to be the expected revenue when playing prices  $(p_1, \cdots, p_n)$, i.e.,
\[
    \textstyle R(p_1, \cdots, p_n)~:=~ \sum_{i = 1}^n p_i \cdot \Pr{i\text{ gets the item}} ~=~ \sum_{i = 1}^n p_i \cdot (1-F_i(p_i))\prod_{j = 1}^{i-1} F_j(p_j),
\]
where $F_i(x)$ is the cumulative distribution function (CDF) of distribution $\D_i$. For SPP, the optimal prices are defined as
\[
    (p^*_1, \cdots, p^*_n) ~:=~ {\argmax}_{(p_1, \cdots, p_n)} R(p_1, \cdots, p_n).
\]
These optimal prices can be easily calculated in polynomial time using a reverse dynamic program:  price $p^*_n = \arg\!\max_p p(1-F_n(p))$, and with known optimal prices $p^*_{i+1}, \cdots, p^*_n$, we can calculate \[ p^*_i ~=~ \argmax_p p(1-F_i(p)) + F_i(p) \cdot \Ex{\text{revenue from }i+1, \cdots, n}.\]

\medskip
\noindent \textbf{Bandit Sequential Posted Pricing (BSPP).} 
In many practical applications of SPP, the valuation distributions $\D_i$  are unknown, so we will study them in the bandit learning model. Consider the following toy example to motivate the model: Suppose you want  to sell an item each day for the entire next month using SPP. For simplicity, assume that each day exactly 3 bidders arrive: one each in the morning,  afternoon, and  evening. If you know the value distributions of these 3 bidders on each day, then you can find the optimal prices $(p^*_1, p^*_2, p^*_3)$ as discussed above, and play them every day to maximize the total revenue. However, if the distributions are unknown then you need to learn the optimal prices during the month. A major challenge is that you don't even get to see the true valuations of the 3 bidders who show up each day, only whether they decide to buy the item, or not, at the played price.

Formally, BSPP is a $T$ rounds/days repeated game where
on each day $t\in [T]$ the seller  proposes a group of $n$ prices $(p_1^{(t)}, \cdots, p_n^{(t)})$. The buyers then arrive one-by-one with valuations $v_i^{(t)} \sim \D_i$ and the first buyer $i$ with valuation $v_i^{(t)} \geq  p_i^{(t)}$ takes the item. The seller receives revenue  $\text{Rev}^{(t)} = p_i^{(t)}$ as the reward for this day and sees $p_i^{(t)}$ as the feedback, or equivalently sees the identity of the buyer that takes the item\footnote{By adding an arbitrarily small noise to each price, the seller can retrieve buyer's identity from the revenue.}. Note that none of the buyer valuations are ever revealed. If there is no buyer with a valuation higher than their price, the seller receives  reward  $0$. The goal of the seller is to minimize in expectation the \textit{total regret}: 
\[
    \textstyle \text{Regret} ~:=~ T \cdot R(p^*_1, \cdots, p^*_n) - \sum_{t \in [T]} \text{Rev}^{(t)}.
\]

  Our first main result gives a near-optimal regret algorithm for BSPP with regular buyers.

\begin{main}[Informal \Cref{thm:newGRMain}] \label{mainResult1}
     For bandit sequential posted pricing with $n$ buyers having regular distribution inside $[0,1]$, there exists an algorithm with $\OTild(\poly(n)\sqrt{T})$ regret. 
\end{main}

It should be noted that this is the first bandit learning algorithm for regular distributions with optimal regret bound in $T$, even for the single-buyer   or i.i.d. buyers setting.
A very recent paper of \cite{LSTW-SODA23} 
studies ``pricing query complexity'' of this problem in the special case of a single-buyer. Their  model is weaker than our bandit learning model since it only gives a \textit{final regret} bound, i.e., one might incur a large regret during the  $T$ days of learning but the final learnt prices have a low $1$-day regret. Hence, their $\Omega({1}/{\epsilon^2})$ lower bound for single-buyer query complexity   immediately implies an $\Omega(\sqrt{T})$ lower bound for BSPP,   showing  tightness of the $\sqrt{T}$ factor in \Cref{mainResult1}. However, their single-buyer pricing query complexity upper bounds do not apply to our stronger bandit learning model.  Moreover, it's unclear how to extend their techniques beyond a single-buyer.

   Our next result gives a near-optimal regret bound  for BSPP with $n$ buyers having general distributions. The authors of \cite{LSTW-SODA23} give an $\Omega({1}/{\epsilon^3})$ query complexity lower bound for single-buyer setting with general distributions, implying an $\Omega(T^{2/3})$ regret lower bound for  BSPP with general distributions. Our result achieves a tight upper bound dependency on $T$.

\begin{main}[Informal \Cref{thm:ggmain}]
     For bandit sequential posted pricing problem 
     with $n$ buyers having value distribution inside $[0,1]$,  there exists an algorithm with $\OTild(\poly(n)T^{{2}/{3}})$ regret. 
\end{main}
 
 This result generalizes the single-buyer result of \cite{DBLP:conf/focs/KleinbergL03} to $n$ sequential buyers. Interestingly, our techniques are very different from theirs  since we are in the stochastic bandit model. Although \cite{DBLP:conf/focs/KleinbergL03} show an $\OTild(T^{{2}/{3}})$ regret algorithm   even when the buyer valuations are chosen by an adversary, we observe in \Cref{sec:lowerBound} that such a result is impossible  for  multiple buyers since already for $2$ buyers with  adversarial valuations, every online algorithm incurs  $\Omega(T)$ regret.

\subsection{Techniques}

\noindent \textbf{Single Regular Buyer via Half-Concavity.} The proof of our first main result is based on a key observation that  the revenue curve of a regular distribution is ``half-concave''. Recall that a distribution is  regular if its revenue curve is concave in the quantile space. 
Simple examples show that regular distributions need not be concave in the value space (e.g., exponential distributions).
Our half-concavity  shows that regular distributions are still concave on one side of its maximum.

\begin{main}[Informal \Cref{lma:1RHalfConSpecial} and \Cref{clm:GRHalfCon}]
    Let $R(p)$ be the revenue function of a regular distribution supported on $[0,1]$ and let $p^* = \argmax_p R(p)$. Then, $R(p)$ is concave in $[0, p^*]$.
\end{main}

 Concavity is a strong property that often allows efficient learning. We show that it's sufficient to learn a single-peaked function even if it is only half-concave. The high-level intuition for learning a half-concave follows from the standard recursive algorithm for learning concave functions: first, consider a (fully) concave function $R(x)$ defined on interval $[\ell, r]$. We set $a = \frac{2\ell + r}{3}$ and $b = \frac{\ell + 2r}{3}$, and test $R(a)$ and $R(b)$ with sufficiently many samples. There can be three cases:
 
 \begin{itemize}
     \item Case 1: $R(a) < R(b)$. We can drop $[\ell, a]$ and recurse on $[a, r]$.
     \item Case 2: $R(a) > R(b)$. We can drop $[b, r]$ and recurse on $[\ell, b]$.
     \item Case 3: $R(a) \approx R(b)$. Concavity implies that $R(x)$ is nearly a constant for all $x \in [\ell, r]$, so we are done.
 \end{itemize}
 Now suppose $R(x)$ is only half-concave. We can still perform a similar algorithm: the first two cases use single-peakedness, so they remain the same. For the third case, half-concavity still guarantees that $R(x)$ can't be too high for $x \in [b, r]$, so we can drop $[b, r]$ and recurse on $[\ell, b]$.

 Our final algorithm and proofs combine the above Case 2 and 3; see details in \Cref{sec:SingleRegular}.


 \medskip
\noindent \textbf{Generalizing Single-Buyer to Sequential Buyers.} In the case of sequential buyers, our main idea is to run the single-buyer algorithm for all $n$ buyers. A major difference compared to the single buyer setting is that the revenue function for this buyer becomes $R(p) = p\cdot (1 - F(p)) + C\cdot F(p)$, where  $C\cdot F(p)$  represents the case that this buyer does not take the item and the seller receives an expected revenue $C$ from the buyers after the current buyer. Therefore, the main for the generalization is to take care of this extra $C\cdot F(p)$ term. We show that for regular distributions the half-concavity still holds for the new function $R(p)$ in the interval $[C, 1]$, along with some other nice properties in the interval $[0, C]$. For general distributions, we show that the extra $C\cdot F(p)$ term does not make a huge difference.

\subsection{Further Related Work}

\noindent \textbf{Bandit Learning and Single-Buyer Posted Pricing.} The bandit learning model is well-established and we refer the readers to these books for classical results~\cite{LS-Book20,BubeckC-Book12,CL-Book06}. The recent book of \cite{Slivkins-Book19} is a great reference for  work at the intersection of bandits and economics. In particular, the problem of learning the single-buyer posted pricing mechanism with bandit feedback has a long history, dating back to  \cite{DBLP:conf/focs/KleinbergL03}. In their paper, the authors put forth an $\OTild(T^{2/3})$ regret bound for general buyer distributions and an $\OTild(\sqrt{T})$ regret bound under a non-standard assumption that the revenue curve is single-peaked, and its second derivative at this peak is a \textit{strictly negative constant}, independent of the parameter $T$.
This result is not directly comparable to our $\OTild(\sqrt{T})$ bound for regular distributions. This is primarily because the second derivative at the maximum of the revenue curve for a regular distribution can range from negligibly small to zero. 
Additionally, \cite{cesa2019dynamic} proposes an $\OTild(\sqrt{KT})$ regret bound, assuming the buyer's value resides within a discrete set with cardinality $K$. This result, however, is also not directly comparable to ours, as regular distributions inherently exhibit continuity, contrasting with the assumption of discrete value sets.

 \medskip \noindent \textbf{Threshold Query Model.} Recently, the threshold query structure that underpins the bandit feedback model of the single-buyer posted pricing problem is extensively studied. In this model, the learner queries a threshold $\tau$ and only observes $\one[\tau<X \sim \D]$. The goal is to determine the minimum number of queries to learn a key parameter (e.g., median, mean, or CDF) of a distribution. We refer the readers to {\color{black}\cite{DBLP:conf/alt/MeisterN21, LSTW-SODA23, okoroafor2023non, DBLP:conf/sigecom/LemeSTW23}} for   learning complexity of the threshold query model.

\medskip \noindent \textbf{Sequential Posted Pricing and Prophet Inequality.} Sequential  posted pricing (SPP) and its variants have been long popular, both in theory and in practice. 
One of the first results in their theoretical study is that  posted prices obtain 78\% of the optimal revenue  for selling a single item in a large market~\cite{BH-EC08}. Subsequently, posted prices have been successfully analyzed in both single-dimensional and multi-dimensional settings; see this survey \cite{Lucier-Survey17}  and book \cite{HartlineBook}. A beautiful paper on the popularity of posted prices in practice is  \cite{einav2018auctions}.

 A closely related problem to SPP is the  Prophet Inequality (PI) problem from optimal stopping theory. In both these problems a sequence of $n$ independent buyers arrive one-by-one. However, in SPP we want to maximize the revenue and in PI we want to maximize the welfare. 
 Interestingly, for known buyer value distributions, both these problems are  equivalent~\cite{CFPV19}, but this reduction requires virtual-value distributions and doesn't work for unknown distributions.
 In the bandit model, PI and SPP behave very differently.  For instance, \cite{GKSW-arXiv22} recently obtained $\widetilde{O}(\sqrt{T})$ regret algorithm for PI with general distributions, whereas  an $\Omega(T^{2/3})$ regret lower bound exists for SPP \cite{LSTW-SODA23}. Our $\widetilde{O}(\sqrt{T})$ regret results crucially  rely on half-concavity of regular distributions. 


\medskip \noindent \textbf{Regular Distributions and Learning.} Myerson's regularity of distributions has been greatly studied since it was introduced in \cite{myerson1981optimal}. 
In particular, it is a standard assumption in learning theory for auctions \cite{DRY-GEB15,DBLP:conf/stoc/ColeR14,HMR-SICOMP18,DBLP:conf/stoc/Devanur0P16,DBLP:conf/sigecom/RoughgardenS16,GHZ-STOC19}.
For basic properties of regular distribution, and its important subclass of Monotone Hazard Rate  distributions, see the books \cite{HartlineBook,roughgarden2017twenty}. A recent paper of \cite{LSTW-SODA23}, which studies pricing query complexity of the single-buyer single-item problem, proves an interesting ``relative flatness'' property of regular distributions in the value space. Roughly it says that  if the revenue curve has nearly 
the same value at 4 different equidistant points then there cannot be a high revenue point in between. 
Comparing it to our idea of half-concavity,  the two properties are in general incomparable as no one implies the other. However, we found our    half-concavity to be  more intuitive and 
convenient to work with, since proofs based on relative flatness often lead to a long case analysis.

\section{A Single Buyer with Regular Distribution}
\label{sec:SingleRegular}

In this section, we present an $\OTild(\sqrt{T})$ regret algorithm for BSPP in the special case of a single buyer, where we call the problem Bandit Posted Pricing. Moreover, we focus on the case when the buyer's value distribution is \textit{regular}, which is a standard assumption in economics \cite{myerson1981optimal}.

\begin{Definition}[Regularity]
    Distribution $\D$  with CDF $F(x)$ and PDF $f(x)$  is called \emph{regular} when 
    $\phi(v) ~:=~ v - \frac{1 - F(v)}{f(v)}$
    is monotone non-decreasing, or equivalently, its \emph{revenue curve} $R_q(q)$ is concave in the quantile space,  where
    \[
    R_q(q) ~:=~ q \cdot F^{-1}(1 - q).
    \]
\end{Definition}

In the \emph{Bandit Posted Pricing} problem  there is a single regular buyer with an unknown regular value distribution $\D$  having support $[0,1]$ and CDF $F(x)$. Our goal is to approach the optimal price $p^*$ that maximizes the revenue function $R(p):= p\cdot (1 - F(p))$ in the the following bandit learning game over $T$ days: On day $t \in \{1,\ldots, T\}$, we post a price $p_t \in [0,1]$ and the environment draws $v_t \sim \D$. Our reward is $p_t \cdot \textbf{1}_{v_t \geq p_t}$, and the goal is to minimize in expectation the total \emph{regret}:
$ 
T \cdot R(p^*) - \sum_{t \in [T]} p_t \cdot \textbf{1}_{v_t \geq p_t}.
$
Our $\OTild(\sqrt{T})$ regret algorithm for this problem uses ``half-concavity''.

\subsection{Half-Concavity}

Our $\OTild(\sqrt{T})$ regret algorithm works beyond regular distributions, for the class of \emph{half-concave} distributions. We first give the definition of half-concavity:

\begin{Definition}[Half-Concavity]
\label{def:HalfConcSpecial}
A function $R(x): \real \rightarrow \real$  is \emph{half-concave} in interval $[\lcb, \ucb] \subseteq [0,1]$ if the following conditions hold:
\begin{enumerate}[label=(\roman*)]
    \item $R(x)$ is single-peaked in $[\lcb, \ucb]$, i.e., $\exists p^* \in [\lcb, \ucb]$ satisfying that $R(x)$ is non-decreasing in $[\lcb, p^*]$ and non-increasing in $[p^*, \ucb]$.
    \item $R(x)$ is 1-Lipschitz in $[\lcb, p^*]$. 
    \item $R(x)$ is concave   in $[\lcb, p^*]$.
\end{enumerate}
\end{Definition}

\begin{figure}[H]
    \centering
    \includegraphics[width=2.5in]{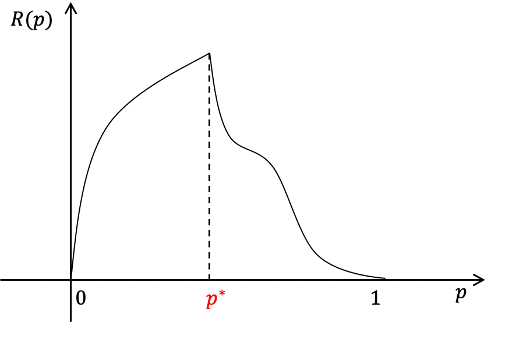}
    \caption{An example of a half-concave function, where the function is Lipschitz and concave before $p^*$.}
    \label{fig:regular}
\end{figure}

The concept of half-concavity implies that a function has nice properties properties on one side of the maximum (see \Cref{fig:regular}), rendering it learnable from this advantageous side. In the definition of half-concavity, both Lipschitzness and some kind of concavity are vital to ensure learnability. Without Lipschitzness, a function becomes unlearnable when $p^*$ is exceedingly close to $0$ due to the inability of accurately detecting the value of $p^*$. On the other hand, without concavity, simple examples show that a regret of $\Omega(T^{2/3})$ is inevitable. For instance, a revenue function characterized by multiple peaks where we can only ascertain the value of $p^*$ when we examine each peak with a sufficient number of samples. {We note that related works, such as \cite{DBLP:conf/icml/CombesP14, DBLP:journals/corr/MagureanuCP14}, address bandit problems with assumptions of unimodularity or Lipschitzness. However, these works only provide instance dependent regret bounds, which are at least $\omega(\sqrt{T})$ in the worst case, as they lack at least one of the key assumptions of Lipschitzness and concavity.}

Given both Lipschitzness  and (half) concavity, the following theorem gives $\OTild(\sqrt{T})$ regret.

\begin{Theorem}
    \label{thm:srmain}
    For Bandit Posted Pricing  with a single buyer,
    if the revenue function $R(p) := p\cdot (1 - F(p))$ is half-concave, then there exists an algorithm with $O(\sqrt{T} \log T)$ regret.
\end{Theorem}

The main application of  \Cref{thm:srmain} is for regular distributions due to the following lemma.

\begin{Lemma}
\label{lma:1RHalfConSpecial}
Let $F(x)$ be the CDF of a regular distribution with support in $[0, 1]$, and 
$R(p) := p \cdot (1 - F(p)) $
be the expected revenue on playing price $p$. Then, $R(p)$ is half-concave in $[0, 1]$.
\end{Lemma}

Before proving this lemma, we observe the following immediate corollary.

\begin{Corollary}[Corollary of \Cref{thm:srmain} and \Cref{lma:1RHalfConSpecial}]
\label{thm:sr-regmain}
    For Bandit Posted Pricing  with a single buyer having regular valuation distribution, there exists an algorithm with $O(\sqrt{T} \log T)$ regret.
\end{Corollary}

\begin{proof}[Proof of \Cref{lma:1RHalfConSpecial}] 
We first provide a proof assuming the PDF $f(p)$ is non-zero and differentiable.
We will prove the three properties of half-concavity one-by-one.

\noindent (i) \emph{Single-peakedness}: By definition, revenue function  of a regular distribution  is  concave in the quantile space. Hence, it's single-peaked in the quantile space. Suppose $q^*$ is the quantile that achieves the maximum, then choosing price  $p^* = F^{-1}(1-q^*)$ proves single-peakedness of $R(p)$.

\medskip

\noindent (ii) \emph{Lipschitzness}: For $0 \leq a \leq b \leq p^*$, we have 
\[
   0~\leq~  R(b) - R(a) ~=~ (b-a) - F(b)\cdot b + F(a) \cdot a ~ \leq~ b - a,
\]
where the first inequality follows from single-peakedness and that $a\leq b \leq p^*$, and the second
inequality uses $F(b) \geq  F(a) \geq 0$ and that $b \geq a \geq 0$.

\medskip
\noindent (iii) \emph{Half-concavity}:  We first provide a direct proof via the monotonicity of virtual valuation $\phi$.

Recall that $\phi(p) := p - \frac{1 - F(p)}{f(p)}$ is non-decreasing for regular distribution, so we have
\[
\phi'(p) ~=~ 1 - \frac{-f^2(p) - (1-F(p)) \cdot f'(p)}{f^2(p)} ~\geq~ 0. 
\]
Rearranging the above inequality, 
\begin{align}
\label{eq:hc1}
    2f^2(p) + f'(p) \cdot (1 - F(p)) ~\geq~ 0.
\end{align}

Consider the derivative and the second order derivative of $R(p) = p \cdot (1 - F(p))$, i.e.,
\begin{align*}
    R'(p) ~=~ 1 - F(p) - p \cdot f(p) \qquad \text{and} \qquad 
    R''(p) ~=~ -2f(p) - p \cdot f'(p).
\end{align*}

For half-concavity, we prove that $R''(p) \leq 0$ for every $p \in [0, p^*]$. Consider  two cases:
\begin{itemize}
    \item Case 1: $f'(p) > 0$. In this case, $R''(p) = -2f(p) - p \cdot f'(p) \leq 0$ immedately holds.
    \item Case 2: $f'(p) \leq 0$. In this case, we have
    \begin{align*}
        R''(p) ~&=~ -2f(p) - p \cdot f'(p)  \\
        ~&=~ \frac{1}{f(p)} \cdot \left(-2f^2(p) - f(p) \cdot p \cdot f'(p)\right) \\
        ~&\leq~ \frac{1}{f(p)} \cdot \left(f'(p) \cdot (1 - F(p)) - f(p) \cdot p \cdot f'(p)\right) \\
        ~&=~ \frac{f'(p)}{f(p)} \cdot \big(1 - F(p) - p \cdot f(p)\big) ~\leq ~0,
    \end{align*}
    where the first inequality is from \eqref{eq:hc1} and the second inequality uses $f'(p) \leq 0$ along with the fact that $R'(p) \geq 0$ when $p \in [0, p^*]$. 
\end{itemize}

The above proof requires the PDF $f(p)$ to be non-zero and differentiable. Below we provide an alternate proof via the concavity in the quantile space, which bypasses this assumption for $f(p)$. We keep both the proofs here since the reader may find one proof more instructive than the other.

\noindent \textbf{An alternate proof for half-concavity.} We prove half-concavity by contradiction. Assume  there exist $0\leq a< b < c \leq p^*$ satisfying  $c = b+t(b-a)$   and   $R(c) > R(b) + t\cdot \big( R(b) - R(a) \big)$. 

Define $q(x):=1 - F(x)$ and $\overline{q} := q(b) + t\cdot \big(q(b) - q(a) \big)$.
We first show that $\overline{q} >  q(c)$: 
Regularity implies that the revenue function is concave in the quantile space. Hence, for all $q' \geq \overline{q}$,  
\[ R(q^{-1}(q')) ~\leq~ R(b) + \frac{q(b) - q'}{q(a) - q(b)}\cdot (R(b) - R(a)) ~\leq~ R(b) + t\cdot (R(b) - R(a)) ~<~ R(c). \] 
For every $q' \geq \bar q$, we have $q' \neq q(c)$. So there must be $q(c) < \bar q$.
This gives the desired contradiction: 
\begin{align*}
    R(c) &~>~ R(b) + t\cdot (R(b) - R(a)) \\
    & ~=~ b\cdot q(b) + t\cdot \big(b \cdot q(b) - a\cdot q(a) \big) \\
    & ~\geq~ b\cdot q(b) + t\cdot \big(b \cdot q(b) - a\cdot q(a) \big) - t(t+1)(b-a)(q(a)-q(b)) \\
    &~=~ \big((t+1)b-ta \big) \cdot \big((t+1)q(b)-tq(a) \big)   \\
    &~=~ c\cdot \overline{q} ~\geq~ c \cdot q(c)   ~=~ R(c).   \qedhere
\end{align*}    
\end{proof}


\subsection{$\OTild(\sqrt{T})$ Regret  for Half-Concave Functions: Proof of \Cref{thm:srmain}}
\label{sec:1RMainAlg}

Our algorithm relies on a subroutine that takes a parameter $\epsilon$ and a confidence interval $[\lcb, \ucb] \ni p^*$ as input, and then after $\OTild(\epsilon^{-2})$ rounds it generates with high probability a new confidence interval $[\lcb', \ucb']$ that contains the optimal price $p^*$ and every price in $[\lcb', \ucb']$ has $1$-day regret bounded by $\epsilon$. Our final algorithm runs in $O(\log T)$ phases, where in each phase we call the sub-routine with the parameter $\epsilon$, and $\epsilon$ is halved in the next phase. The  sub-routine is captured by the
following lemma, which is the heart of the proof and will be proved in \Cref{sec:1RSubR} using half-concavity of $R(p)$. 

\begin{Lemma}
\label{thm:1RMain}
Let $R(p)$ be a half-concave revenue function defined in $[0, 1]$. Given  $\epsilon \geq {1}/{\sqrt{T}}$ and $[\lcb, \ucb] \ni p^*$, there exists an algorithm that tests $O\big(\epsilon^{-2}{\log^2 T} \big)$ rounds with prices inside $[\lcb, \ucb-\frac{1}{T^{100}}]$ and outputs an interval $[\lcb', \ucb'] \subseteq [\lcb, \ucb]$ satisfying  with probability $1 - T^{-5}$ that 
\begin{itemize}
    \item $p^* \in [\lcb', \ucb']$
    \item   $R(x) \geq R(p^*) - \epsilon$ for any $x \in [\lcb', \ucb'- \frac{1}{T^{100}}]$.  
\end{itemize}
\end{Lemma}


\begin{algorithm}[tbh]
\caption{$O(\sqrt{T} \log T)$ Regret Algorithm}
\label{alg:sr-gen}
\KwIn{Hidden Revenue Function $R(p)$, time horizon $T$}
Let $\epsilon_1 \gets 1, [\lcb_1, \ucb_1] \gets [0, 1]$, and $i \gets 1$. \\
\While{$\epsilon_i > \frac{\log T}{\sqrt{T}}$}
{
    Run the algorithm described in \Cref{thm:1RMain} with $\epsilon_i, [\lcb_i, \ucb_i]$ as input, and get $[\lcb', \ucb']$. \\
    Let $[\lcb_{i+1}, \ucb_{i+1}] \gets [\lcb', \ucb']$ and $\epsilon_{i+1} \gets \frac{1}{2}\epsilon_i$. \\
    Let $i \gets i+1$.
}
Finish remaining rounds with any price in $[\lcb_i, \ucb_i - \se]$.
\end{algorithm}


Given \Cref{thm:1RMain},   our  algorithm for \Cref{thm:srmain} is  natural and has a simple proof.
 
\begin{proof}[Proof of \Cref{thm:srmain}]
    We claim that with probability at least $1 - T^{-4}$, the regret of \Cref{alg:sr-gen} is $O(\sqrt{T} \log T)$. \Cref{alg:sr-gen} uses \Cref{thm:1RMain} for multiple times with a halving error parameter. Assume \Cref{alg:sr-gen} ends with $i = k +1$, i.e., the while loop runs $k$ times. Since we obtain $\epsilon_{i+1} = \epsilon_i /2$ and $\epsilon_k > \frac{\log T}{\sqrt{T}}$ holds, there must be $k = O(\log T)$.
    
    For $i \in [k]$, let $\alg_i$ represent the corresponding algorithm we call when using \Cref{thm:1RMain} with $\epsilon_i, [\lcb_i, \ucb_i]$ as the input. To use \Cref{thm:1RMain} we need to verify that $p^* \in [\lcb_i, \ucb_i]$ for all $i \in [k]$. The condition $p^* \in [\lcb_1, \ucb_1] = [0, 1]$ clearly holds. For $i =2, 3, \cdots, k$, the condition $p^* \in [\lcb_i, \ucb_i]$ is guaranteed by \Cref{thm:1RMain} when calling $\alg_{i-1}$. The failing probability of \Cref{thm:1RMain} is $T^{-5}$. By the union bound, with probability $1 - k\cdot T^{-5} > 1 - T^{-4}$, \Cref{thm:1RMain} always holds in \Cref{alg:sr-gen}.

    Now we prove the regret of \Cref{alg:sr-gen}. \Cref{thm:1RMain}  guarantees that $R(p^*) - R(p) \leq \epsilon_i$ holds for all $p \in [\lcb_{i+1}, \ucb_{i+1} - \se]$, while when calling $\alg_{i+1}$, only the prices in $[\lcb_{i+1}, \ucb_{i+1} - \se]$ is tested. Therefore, the total regret of \Cref{alg:sr-gen} can be bounded by
    \[ \textstyle
    1\cdot O\big(\epsilon^{-2}_1\log^2 T\big) ~+~ \sum_{i = 2}^k \epsilon_{i-1} \cdot O(\epsilon^{-2}_i\log^2 T) ~+~ T \cdot \epsilon_k  ~=~ O(\sqrt{T} \log T).
    \]
    In this expression, the first term is the regret of $\alg_1$. The second term is the sum of the regret from $\alg_2, \cdots, \alg_k$. The third term is the regret for the remaining rounds.

    Finally, we also need to check that \Cref{alg:sr-gen} uses no more than $T$ rounds. \Cref{thm:1RMain} suggests that $\alg_i$ runs $O({\epsilon^{-2}_i}{\log^2 T})$ rounds. So the total number of rounds in the while loop is
    $
    \sum_{i \in [k]} O({\epsilon^{-2}_i}{\log^2 T}) ~=~ O(T).
    $
    Therefore, \Cref{alg:sr-gen} is feasible.
\end{proof}

\subsection{Main Sub-Routine: Proof of \Cref{thm:1RMain} via Half-Concavity}
\label{sec:1RSubR}

In this section, we present the main sub-routine of our bandit algorithm. It utilizes half-concavity to generate a confidence interval for the optimal price $p^*$ in $\OTild(\frac{1}{\epsilon^2})$ rounds while ensuring that every price in the interval is $\epsilon$-optimal. 


\paragraph{Algorithm Overview.} The algorithm contains two major steps (also see \Cref{fig:sr-twosteps}):
\begin{itemize}
    \item Step 1: Find an approximation $\hat p$ such that $R(p^*) - R(\hat p) \leq \frac{\epsilon}{2}$ via half-concavity. This step gives a sufficiently precise price estimate approximating $p^*$.
    
    \item Step 2: Given $\hat p$, construct a new confidence interval $[\lcb', \ucb']$ via binary search. As $R(p)$ exhibits a single-peak property, we can use $R(\hat p)$ as a benchmark and implement a standard binary search algorithm to identify the leftmost point $\lcb'$ and the rightmost point that are $\frac{\epsilon}{2}$-close to $R(\hat p)$. These two endpoints describe the desired new confidence interval. 
\end{itemize}
We note an important detail in Step 2: the Lipschitzness assumption in the range $[p^*, \ucb]$ is absent. Consequently, the binary search algorithm fails to provide a satisfactory loss guarantee for a small tail of the new confidence interval. This introduces the non-standard error factor $\se$ as discussed in \Cref{thm:1RMain}.

\begin{figure}
    \centering
        \includegraphics[width=2.7in]{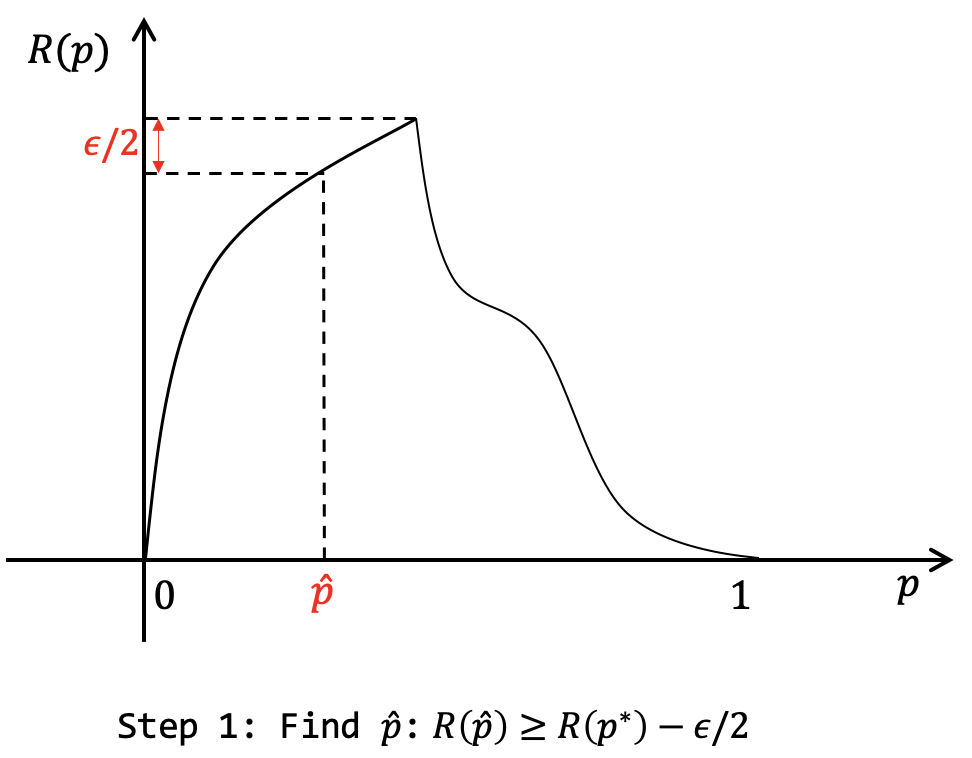}
        \includegraphics[width=2.7in]{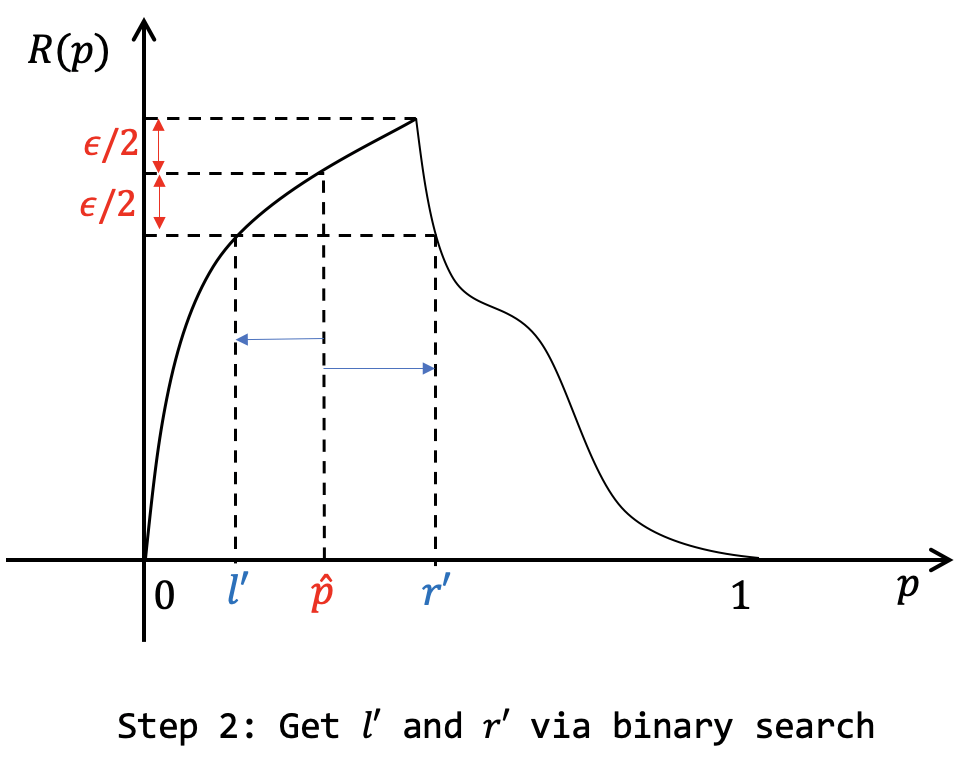}
    \caption{The two main steps of \Cref{thm:1RMain}}
    \label{fig:sr-twosteps}
\end{figure}

\subsubsection{Step 1: Find $\hat p$ to Approximate $p^*$}

The first step of the sub-routine is to find $\hat p$ that approximates $p^*$. Specifically, we want to prove the following lemma:

\begin{Lemma}
\label{lma:Step1}
    Assume $R(p)$ is a half-concave revenue function defined in $[0, 1]$. Given interval $[\lcb, \ucb]$ and error parameter $\epsilon > {1}/{\sqrt{T}}$, there exists an algorithm that tests $O({\epsilon^{-2}}{\log^2 T})$ rounds with prices inside $[\lcb, \ucb - {T^{-100}}]$ and outputs $\hat p \in [\lcb, \ucb - \se]$ satisfying with probability at least $1 - T^{-6}$ that $\max_{p \in [\lcb, \ucb]} R(p) - R(\hat p) < \frac{\epsilon}{2}$.
\end{Lemma}

When we have the assumption $p^* \in [\lcb, \ucb]$ from \Cref{thm:1RMain}, \Cref{lma:Step1} implies we are able to find $\hat p$ such that $R(p^*) - R(\hat p) \leq \frac{\epsilon}{2}$ in the first step. We first give the pseudo-code of the algorithm required by \Cref{lma:Step1} and its high-level idea.

\begin{algorithm}[tbh]
\caption{Finding $\hat p$}
\label{alg:sr-esp}
\KwIn{Hidden Revenue Function $R(p)$, Interval $[\lcb, \ucb]$, Error Parameter $\epsilon$}
Let $\epsprm = \frac{\epsilon}{100}$ be the scaled error parameter, $\lcb_1 \gets \lcb, \ucb_1 \gets  \ucb$, $i = 1$. \\
\While{$\ucb_i - \lcb_i > \epsprm$}
{
    Let $a_i \gets ({2\lcb_i + \ucb_i})/{3}$  and $b_i \gets ({\lcb_i + 2\ucb_i})/{3}$. \\
    Test $C \cdot {\epsprm^{-2}}{\log T}$ rounds with price $p=a_i$. Let $\hat R(a_i)$ be the average reward. \\
    Test $C \cdot {\epsprm^{-2}}{\log T}$ rounds with price $p=b_i$. Let $\hat R(b_i)$ be the average reward. \\
    \eIf{$\hat R(a_i) < \hat R(b_i) - 2\epsprm$}{Let $\lcb_{i+1} \gets a_i, \ucb_{i+1} \gets \ucb_i$}{Let $\lcb_{i+1} \gets \lcb_i, \ucb_{i+1} \gets b_i$}
    $i \gets i+1$
}
Test $C \cdot {\epsprm^{-2}}{\log T}$ rounds with price $p=\lcb_i$. Let $\hat R(\lcb_i)$ be the average reward. \\
Let $\hat p = \arg \max_{p \in P} \hat R(p)$, where $P = \{p:p \text{ is tested}\}$. \\
\KwOut{$\hat p$. }
\end{algorithm}

\Cref{alg:sr-esp} is a recursive algorithm that runs $O(\log \frac{1}{\epsprm})$ rounds. In each round, the algorithm tests the two points $a, b$ that divide the current interval $[\lcb, \ucb]$ into thirds. Each price is tested for $\OTild(\frac{1}{\epsprm^2})$ rounds, where $\epsprm = {\epsilon}/{100}$ is the scaled error parameter. A standard concentration inequality guarantee both $|R(a) - \hat R(a)| \leq \epsprm$ and $|R(b) - \hat R(b)| \leq \epsprm$ hold. Then, the algorithm drops one third of the interval according to the test results. There are two different cases (see \Cref{fig:sr-twocases}):

\begin{figure}
    \centering
        \includegraphics[width=2.7in]{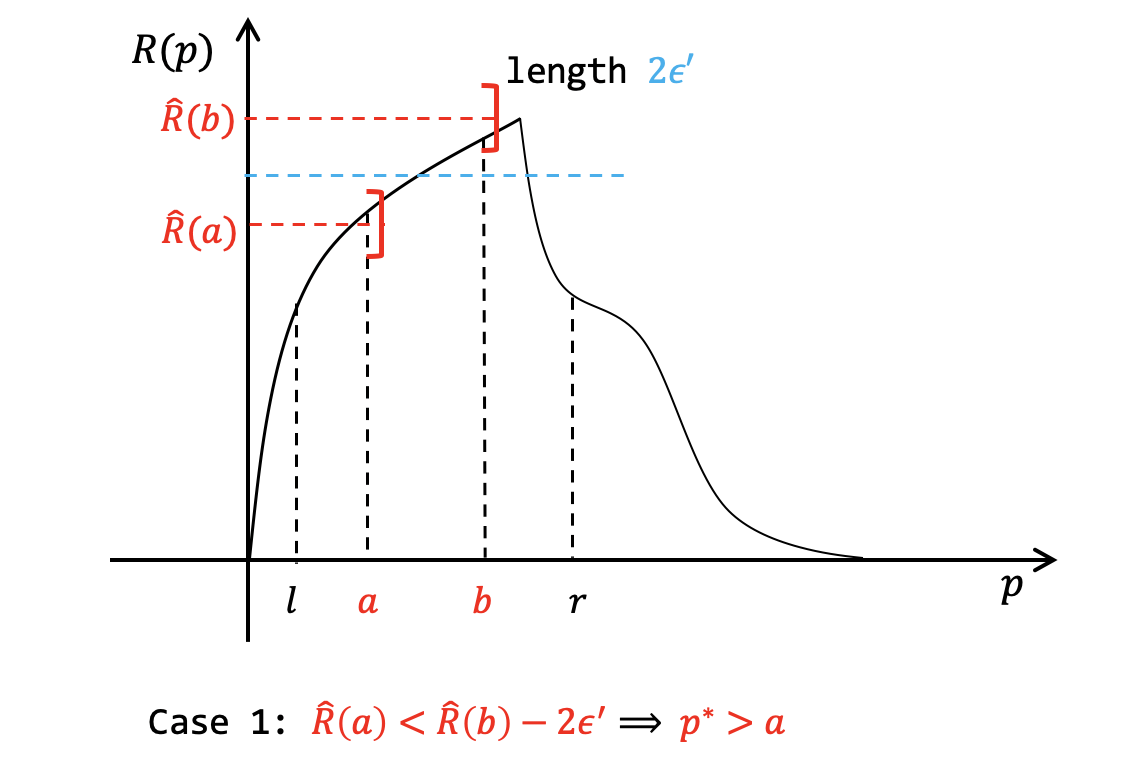}
        \includegraphics[width=2.7in]{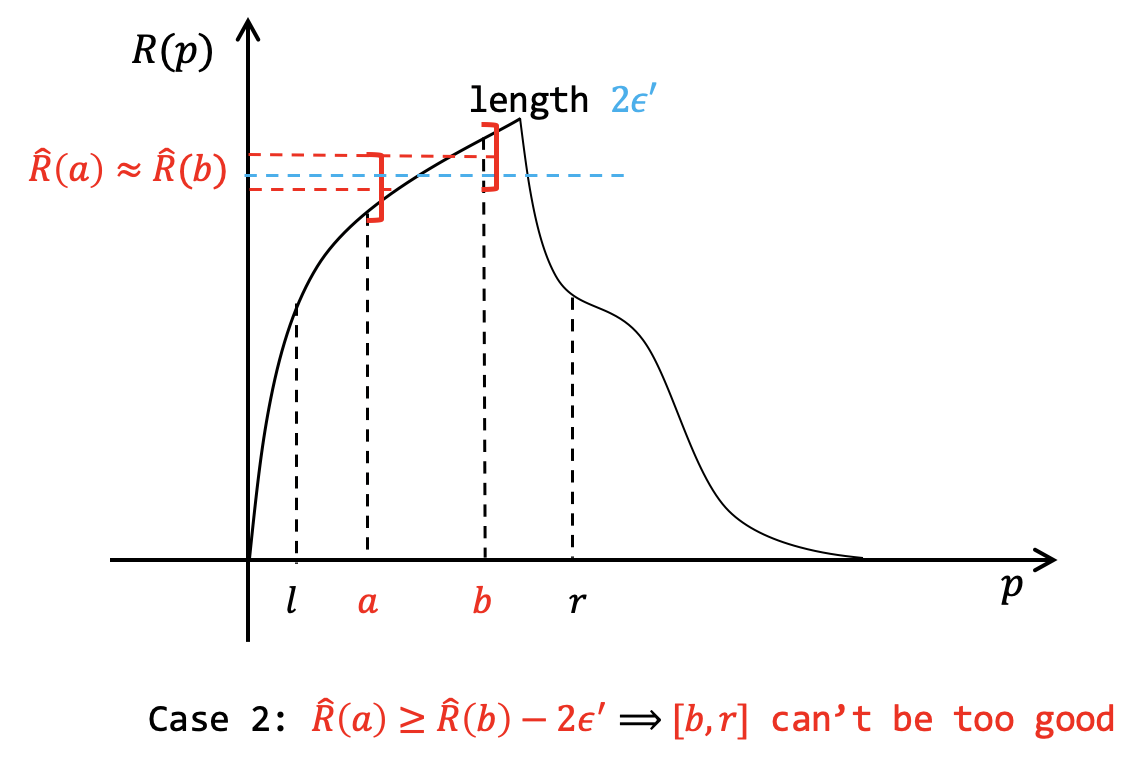}
    \caption{Two cases of Step 1 as discussed in \Cref{alg:sr-esp}.}
    \label{fig:sr-twocases}
\end{figure}

\noindent \textit{Case 1}: $\hat R(a) < \hat R(b) - 2\epsprm$. The inequality is sufficient to show that $R(a) < R(b)$. Since function $R(p)$ is single-peaked, there must be $p^* \geq a$ and the sub-interval $[\lcb, a]$ can be dropped.

\noindent \textit{Case 2}: $\hat R(a) \geq \hat R(b) - 2\epsprm$. This inequality implies $R(b) - R(a)$ is sufficiently small. In this case, observe that if $p^* \in [b, \ucb]$, the value of $R(p^*)$ can't be much better than $R(b)$, because the concavity in $[a, p^*]$ guarantees that $R(p^*) - R(b) \leq R(b) - R(a)$. Therefore, the sub-interval $[b, \ucb]$ can be dropped because we don't need to further test it.

Finally, the recursive algorithm stops when the length of the remaining interval is bounded by $\epsprm$, and the Lipschitzness guarantees the final interval is good.

Now, we give the formal proof of \Cref{lma:Step1}.

\begin{proof}[Proof of \Cref{lma:Step1}]

 We first show that \Cref{alg:sr-esp}  tests $O(\frac{\log^2 T}{\epsilon^2})$ prices in $[\lcb, \ucb - \se]$.
 
 The main body of \Cref{alg:sr-esp} is a while loop that maintains an interval to be tested. Assume the while loop stops when $i = k + 1$. Then, for $i \in [k]$, the length of $[\lcb_{i+1}, \ucb_{i+1}]$ is two-thirds of $[\lcb_i, \ucb_i]$, and we have $\ucb_k - \lcb_k > \epsprm$. Therefore, $k = O(\log \frac{1}{\epsprm}) = O(\log T)$. For interval $[\lcb_i, \ucb_i]$, two prices $a_i, b_i$ are tested for $O({\epsprm^{-2}}{\log T})$ times, so in total \Cref{alg:sr-esp} tests $O({\epsprm^{-2}}{\log^2 T})$ rounds. To show every tested price is in $[\lcb, \ucb- \se]$, let $p$ be a price tested in \Cref{alg:sr-esp}. $p \in [\lcb, \ucb]$ directly follows from the algorithm. For the inequality $p \leq \ucb - \se$, observe that the largest tested price must be $p = b_i$ for some $i \in [k]$, and $b_i \leq \ucb - \se$ is guaranteed by the fact that $\ucb_i - b_i \geq \frac{1}{3} \epsprm \gg \se$.

 It only remains to show $\max_{p \in [\lcb, \ucb]} - R(\hat p) \leq \frac{\epsilon}{2}$. Define $p^*_{\lcb, \ucb} := \arg \max_{p \in [\lcb, \ucb]} R(p)$ to be the optimal price in $[\lcb, \ucb]$. We will show that  $R(p^*_{\lcb, \ucb}) - R(\hat p) \leq 6\epsprm < \frac{\epsilon}{2}$ holds with probability $1 - T^{-6}$. We first need the following claim, which follows from standard concentration inequalities:
    \begin{restatable}{Claim}{clmcona}
    \label{clm:sr-step1-con}
    In \Cref{alg:sr-esp}, $|\hat R(p) - R(p)| \leq \epsprm$ simultaneously holds for every tested price $p$ with probability at least $1 - T^{-6}$.
    \end{restatable}
    \begin{proof}
    For a single tested price $p$, $\hat R(p)$ is estimated by calculating the average of $N = C \cdot \frac{\log T}{\epsprm^2}$ samples. By Hoeffding's Inequality, 
    \[
    \Pr{|\hat R(p) - R(p)| > \epsprm} ~\leq 2\exp\left(-2 N \epsprm^2\right) ~=~ 2T^{-2C} ~<~ T^{-7}.
    \]
    The last inequality holds when $C$ is a constant greater than $4$. Then, we have $|\hat R(p) - R(p)| \leq \epsprm$ holds with probability $1 - T^{-7}$ for a single tested price $p$.
    
    Notice that \Cref{alg:sr-esp} can't test more than $T$ prices. By the union bound, $|\hat R(p) - R(p)| \leq \epsprm$ simultaneously holds for all tested prices with probability $1 - T^{-6}$.
\end{proof}
  
    
    Next, we use \Cref{clm:sr-step1-con} to show $R(p^*_{\lcb, \ucb}) - R(\hat p) \leq 6\epsprm$. We consider the following two cases:
    
    The first case is that there exists $i \in [k]$, such that $p^*_{\lcb, \ucb}$ falls in $[\lcb_i, \ucb_i]$ but not in $[\lcb_{i+1}, \ucb_{i+1}]$. In this case, there must be $\hat R(a_i) \geq \hat R(b_i) - 2\epsprm$ and $[b_i, \ucb_i]$ is dropped. If not, we have $\hat R(a_i) \geq \hat R(b_i) - 2\epsprm$ and $[\lcb_i, a_i]$ is dropped. However, the inequality $\hat R(a_i) \geq \hat R(b_i) - 2\epsprm$ together with \Cref{clm:sr-step1-con} gives $R(a_i) < R(b_i)$. Then, the single-peakedness of $R(p)$ gives $p^*_{\lcb, \ucb} \geq a_i$, which is in contrast to the assumption that $p^*_{\lcb, \ucb} \notin [\lcb_{i+1}, \ucb_{i+1}]$.  Therefore, we have $\hat R(a_i) \geq \hat R(b_i) - 2\epsprm$, and the assumption $p^*_{\lcb, \ucb} \notin [\lcb_{i+1}, \ucb_{i+1}]$ implies $p^*_{\lcb, \ucb} \in [b_i, \ucb_{i}]$. Then, the concavity of $R(p)$ in $[a_i, p^*_{\lcb, \ucb}]$ guarantees 
    \[R(p^*_{\lcb, \ucb}) - R(b_i) ~\leq~ R(b_i) - R(a_i) ~\leq~ (\hat R(b_i) + \epsprm) - (\hat R(a_i) - \epsprm) ~\leq~ 4\epsprm.\]

    The second case is that the desired $i$ in the first case doesn't exist. In this case, the only possibility is that $p^*_{\lcb, \ucb} \in [\lcb_{k+1}, \ucb_{k+1}]$. Then, the 1-Lipschitzness of function $R(p)$ guarantees that $R(p^*_{\lcb, \ucb}) - R(\lcb_k) \leq  \epsprm$.

    In both cases, we find a tested price $p'$ satisfying $R(p') \geq R(p^*_{\lcb, \ucb}) - 4\epsprm$. Then,
\[
R(\hat p) ~\geq~ \hat R(\hat p) - \epsprm ~\geq \hat R(p') - \epsprm ~\geq ~ R(p') - 2\epsprm ~\geq R(p^*_{\lcb, \ucb}) - 6\epsprm.     \qedhere
\]
\end{proof}

\IGNORE{
\paragraph{Comparision to \cite{LSTW-SODA23}} We note that a similar $\OTild(\frac{1}{\epsilon^2})$ threshold query complexity result is also given in \cite{LSTW-SODA23}. Comparing to the algorithm in \cite{LSTW-SODA23}, we make the following improvements:
\begin{itemize}
    \item \Cref{alg:sr-esp} has a lower dependency on the logarithmic factors.
    \item The algorithm in \cite{LSTW-SODA23} is built on a non-standard flatness observation for regular distributions. The algorithm requires a case discussion with 7 different cases. Our algorithm is based on a cleaner idea of half-concavity and avoids involved case discussions.
    \item A generalization of \Cref{alg:sr-esp} works for sequential posted pricing setting. It's unclear whether the algorithm in \cite{LSTW-SODA23} generalizes to sequential buyers.
\end{itemize}
}

\subsubsection{Step 2: Generating New Confidence Interval}

Upon executing \Cref{alg:sr-esp}, we obtain an approximate optimal price $\hat p$. The subsequent stage of the algorithm uses $R(\hat p)$ as a benchmark to establish the new upper and lower bounds of the confidence interval. Given that the function $R(p)$ is single-peaked, two independent binary search algorithms suffice to independently determine these new upper and lower bounds, leading to the following \Cref{lma:Step2}:

\begin{restatable}{Lemma}{lmasrbs}
\label{lma:Step2}
    Assume $R(p)$ is a half-concave revenue function defined in $[0, 1]$. Given interval $[\lcb, \ucb] \ni p^*$, error parameter $\epsilon > \frac{1}{\sqrt{T}}$, near-optimal price $\hat p \in [\lcb, \ucb-\se]$ that satisfies $R(p^*) - R(\hat p) \leq \frac{\epsilon}{2}$, there exists an algorithm that tests $O(\epsilon^{-2} \log^2 T )$ rounds with prices in $[\lcb, \ucb - \se]$ and outputs $[\lcb',\ucb'] \subseteq [\lcb, \ucb]$ satisfying with probability $1 - T^{-6}$ that:
    \begin{itemize}
        \item $p^* \in [\lcb', \ucb']$.
        \item For $p \in [\lcb', \ucb'-\se]$, we have $R(p^*)  - R(p) \leq \epsilon$.
    \end{itemize}
\end{restatable}

Combining \Cref{lma:Step1} and \Cref{lma:Step2} proves \Cref{thm:1RMain}:

\begin{proof}[Proof of \Cref{thm:1RMain}]
    We first run \Cref{alg:sr-esp} to get $\hat p$, and then run the algorithm described in \Cref{lma:Step2} with $\hat p$ being part of the input. Combining \Cref{lma:Step1} and \Cref{lma:Step2} with union bound guarantees that we have the desired output with success probability at least $1 - 2T^{-6} > 1 - T^{-5}$.
\end{proof}

It only remains to prove \Cref{lma:Step2}. We first give the following \Cref{alg:sr-bs}, and then show \Cref{alg:sr-bs} is the desired algorithm for \Cref{lma:Step2}.

\begin{algorithm}[tbh]
\caption{Getting $[\lcb', \ucb']$ via Binary Search}
\label{alg:sr-bs}
\KwIn{Hidden Revenue Function $R(p)$, Interval $[\lcb, \ucb]$, Error Parameter $\epsilon$, Near-Optimal price $\hat p$}
Let $\epsprm = \frac{\epsilon}{100}$ be the scaled error parameter. \\
\tcp{First binary search to determine $\lcb'$}
Test $C \cdot \frac{\log T}{\epsprm^2}$ rounds with price $p=\hat p$. Let $\hat R(\hat p)$ be the average reward. \\
Let $\lcb_{b1} \gets \lcb, \ucb_{b1} \gets \hat p$. \\
\While{$\ucb_{b1} - \lcb_{b1} \geq \se$}
{
    Let $m = \frac{\lcb_{b1}+\ucb_{b1}}{2}$. \\
    Test $C \cdot \frac{\log T}{\epsprm^2}$ rounds with price $p=m$. Let $\hat R(m)$ be the average reward. \\
    \leIf{$\hat R(m) < \hat R(\hat p) - 2\epsprm$}{Update $\ucb_{b1} \gets m$}{Update $\lcb_{b1} \gets m$}
}
Let $\lcb' \gets \lcb_{b1}$ \\
\tcp{Second binary search to determine $\ucb'$}
Test $C \cdot \frac{\log T}{\epsprm^2}$ rounds with price $p=\ucb - \se$. Let $\hat R(\ucb - \se)$ be the average reward. \\
\eIf{$\hat R(\ucb - \se) \geq \hat R(\hat p) - 2\epsprm$}
{
    \tcp{Special Case: original upper bound is good}
    Let $\ucb' \gets \ucb$
}
{
Let $\lcb_{b2} \gets \hat p, \ucb_{b2} \gets \ucb - \se$. \\
\While{$\ucb_{b2} - \lcb_{b2} \geq \se$}
{
    Let $m = \frac{\lcb_{b2}+\ucb_{b2}}{2}$. \\
    Test $C \cdot \frac{\log T}{\epsprm^2}$ rounds with price $p=m$. Let $\hat R(m)$ be the average reward. \\
    \leIf{$\hat R(m) < \hat R(\hat p) - 2\epsprm$}{Update $\ucb_{b2} \gets m$}{Update $\lcb_{b2} \gets m$}
}
Let $\ucb' \gets \ucb_{b2}$.
}
\KwOut{$[\lcb', \ucb']$. }
\end{algorithm}

We first give the following concentration bound for \Cref{alg:sr-bs}:

We first give the following claim:
    \begin{Claim}
    \label{clm:sr-step2-con}
    In \Cref{alg:sr-bs}, $|\hat R(p) - R(p)| \leq \epsprm|$ simultaneously holds for every tested price $p$ with probability at least $1 - T^{-6}$.
    \end{Claim}
    \begin{proof}
        For a single tested price $p$, $\hat R(p)$ is estimated by calculating the average of $N = C \cdot \frac{\log T}{\epsprm^2}$ samples. By Hoeffding's Inequality, 
        \[
        \Pr{|\hat R(p) - R(p)| > \epsprm} ~\leq~ 2\exp\left(-2 N \epsprm^2\right) ~=~ 2T^{-2C} ~<~ T^{-7}.
        \]
        The last inequality holds when $C$ is a constant greater than $4$. Then, we have $|\hat R(p) - R(p)| \leq \epsprm$ holds with probability $1 - T^{-7}$ for a single tested price $p$.
        
        Notice that \Cref{alg:sr-bs} can't test more than $T$ prices. By the union bound, $|\hat R(p) - R(p)| \leq \epsprm$ simultaneously holds for all tested prices with probability $1 - T^{-6}$.
    \end{proof}

    Now we prove \Cref{lma:Step2} assuming \Cref{clm:sr-step2-con} holds.

\begin{proof}[Proof of \Cref{lma:Step2}]
    We show \Cref{alg:sr-bs} satisfies the statements in \Cref{lma:Step2}. Our proof starts from verifying \Cref{alg:sr-bs} runs $O(\frac{\log T}{\epsprm^2})$ rounds with prices in $[\lcb, \ucb - \se]$. For the number of tested rounds, \Cref{alg:sr-bs} runs two standard binary searches. In one binary search, the subroutine tests one price with $C \cdot \frac{\log T}{\epsprm^2}$ rounds, and the binary search stops when the length of the interval is less than $\se$. Therefore, The total number of rounds is $C \cdot \frac{\log T}{\epsprm^2} \cdot O(\log T^{100}) = O(\frac{\log^2 T}{\epsilon^2})$. For the range constraint of all tested prices, \Cref{alg:sr-bs} directly guarantees that every tested price falls in $[\lcb, \ucb - \se]$.

    Next, we show $p^* \in [\lcb', \ucb']$. We first prove $p^* \geq \lcb'$. In \Cref{alg:sr-bs}, the only way we update $\lcb'$ is that we observe $\hat R(m) < \hat R(\hat p) - 2\epsprm$ in the first binary search, and $\lcb'$ is updated to at least $m$. This is feasible because  \Cref{clm:sr-step2-con} gives $R(m) < R(\hat p) \leq R(p^*)$, and the single-peakedness of $R(p)$ guarantees $m \leq p^*$. The proof of $p^* \leq \ucb'$ is symmetric.

    Finally, we show $R(p) \geq R(p^*) - \epsilon$ for all $p \in [\lcb', \ucb' - \se]$. Function $R(p)$ is single-peaked, so it is sufficient to prove that $R(\hat p) - R(\lcb') \leq 5\epsprm < \frac{\epsilon}{2}$ and $R(\hat p) - R(\ucb' - \se) \leq 5\epsprm < \frac{\epsilon}{2}$. Then, combining these two inequalities with the assumption $R(p^*) - R(\hat p) \leq \frac{\epsilon}{2}$ gives the desired statement.
 
    We first prove that $R(\hat p) - R(\lcb') \leq 5\epsprm$. The binary search subroutine guarantees that $\hat R(\ucb_{b1}) \geq \hat R(\hat p) - 2\epsprm$ after the binary search loop is finished. Then, \Cref{clm:sr-step2-con} ensures that $R(\hat p) - R(\ucb_{b1}) \leq 4\epsprm$. Since $\lcb'$ is set to be $\lcb_{b1} > \ucb_{b1} - \se$, the 1-Lipschitzness of $R(p)$ in $[\lcb_{b1}, \hat p]$ guarantees that 
    \[
    R(\hat p) - R(\lcb') ~=~ R(\hat p) - R(\ucb_{b1}) + R(\ucb_{b1}) - R(\lcb_{b1}) ~\leq~ 4\epsprm + \se ~\leq~ 5\epsprm. 
    \]
    For the inequality $R(\hat p) - R(\ucb' - \se) \leq 5\epsprm$,  a symmetric proof for the second binary search subroutine gives $R(\hat p) - R(\lcb_{b2}) \leq 4\epsprm$, while $\ucb'$ is set to be $\ucb_{b2} < \lcb_{b2} + \se$. If $p^* < \ucb' - \se$, the single-peakedness of $R(p)$ gives 
    \[
    R(\ucb' - \se) ~\geq~ R(\lcb_{b2}) ~\geq~ R(\hat p) - 4\epsprm.\]
    Otherwise, the 1-Lipschitzness gives 
    \[
    R(\ucb' - \se) ~\geq~ R(\lcb_{b2}) - (\lcb_{b2} - \ucb_{b2} + \se) ~\geq~ R(\hat p) - 4\epsprm - \se > R(\hat p) - 5\epsprm.    \qedhere
    \]
\end{proof}

\section{$n$ Buyers with Regular  Distributions}
\label{sec:GeneralRegular}

In this section, we provide the $\OTild(\poly(n)\sqrt{T})$ regret algorithm for BSPP with $n$ buyers having regular distributions. In this problem we have $n$ buyers with independent regular value distributions $\D_1, \cdots, \D_n$. The seller posts prices $p_1, \cdots, p_n$, and the first buyer with value $v_i \sim \D_i \geq p_i$ gets the item and pays $p_i$. Our goal is to approach the optimal prices $p^*_1, \cdots, p^*_n$ that maximizes the expected revenue $R(p_1, \cdots, p_n):=p_i \cdot \Pr{i \text{ gets the item}}$ via a bandit learning game over $T$ days: On day $t \in [T]$ we post a price vector $(p^{(t)}_1, \cdots, p^{(t)}_n)$ and the environment draws $v^{(t)}_1 \sim \D_1, \cdots, v^{(t)}_n \sim \D_n$. We only observe reward $\text{Rev}_t=\sum_{i \in [n]} p^{(t)}_i \cdot \one[p^{(t)}_i \geq v^{(t)}_i \land \forall j<i,p^{(t)}_j < v^{(t)}_j ]$. Our goal is to minimize the total regret $T \cdot R(p^*_1, \cdots, p^*_n) - \sum_{t \in [T]} \text{Rev}_t$ in expectation. Our main result is:

\begin{Theorem}
    \label{thm:newGRMain}
     For BSPP with 
     $n$ buyers having regular distributions,  there exists an algorithm with $O(n^{2.5}\sqrt{T} \log T)$ regret.
\end{Theorem}

\subsection{Proof Overview}
\label{sec:pfov}

In this section, we provide an overview of our algorithm for sequential buyers.
The main structure of our algorithm follows the structure of the single-buyer algorithm. We maintain $n$ confidence intervals $[\lcb_1, \ucb_1], \cdots, [\lcb_n, \ucb_n]$ for the optimal prices $p^*_1, \cdots, p^*_n$. The core of the algorithm is to design a sub-routine that uses $\OTild({\poly(n)}/{\epsilon^2})$ rounds to update the confidence intervals to $[\lcb'_1, \ucb'_1], \cdots, [\lcb'_n, \ucb'_n]$, such that playing any prices inside new confidence intervals has regret at most $\epsilon$. Then, calling this sub-routine $O(\log T)$ times with halving error parameter $\epsilon$ suffices to obtain $\OTild(\poly(n)\sqrt{T})$ regret.

The core sub-routine contains two steps. The first step is to get a group of near-optimal prices $\hat p_1, \cdots, \hat p_n$. The second step is to use prices $\hat p_1, \cdots, \hat p_n$ as a benchmark to find new upper- and lower- confidence bounds. 

\medskip \noindent \textbf{Step 1: Finding $\hat p_i$ with a Different Revenue Function.} Comparing to the single-buyer setting, the revenue function for sequential buyers is  different. In particular, when we want to determine $\hat p_i$, the revenue function for buyer $i$ looks like
\[
R(p) = p\cdot (1 - F(p)) + C \cdot F(p),
\]
where the constant $C$ represents the expected revenue from the buyers $i+1, \cdots, n$. Our goal is to find a near-optimal price $\hat p_i$ for this new revenue function $R(p)$.

The major challenge in this step is that function $R(p)$ is not always half-concave: it is half-concave in $[C, 1]$, while in $[0, C]$ it is increasing but not necessary concave. This prevents us from directly applying the single-buyer algorithm. To fix this, an idea is to estimate the value of $C$ to determine the half-concave region. We show that $C$ can be estimated accurately enough in some of the cases, where then running the single-buyer algorithm is sufficient.  On the other hand, in the cases where $C$ cannot be estimated, we observe an extra property of $R(p)$ that it is growing fast enough in $[0, C]$, leading to a ``generalized half-concavity" property in $[0, p^*]$. This is sufficient to show the learnability of $R(p)$. Combining the two cases completes Step 1 of our algorithm.

\medskip \noindent \textbf{Step 2: Binary Search with Prices $\hat p_1, \cdots, \hat p_n$.} The second step of the sub-routine is to update confidence intervals via binary searches. In this step, the main new challenge is the error from $\hat p_1, \cdots \hat p_n$. For instance, while trying to  update $[\lcb_i, \ucb_i]$ with revenue function $R(p) = p\cdot (1 - F(p)) + C \cdot F(p)$, the constant $C$  comes from non-optimal prices $\hat p_{i+1}, \cdots, \hat p_n$. 
Notice that $p^*_i$ is not the optimal price of $R(p)$. This creates an issue  that comparing $R(\hat p_i)$ with $R(p)$ for a price $p \in [\lcb_i, \ucb_i]$ is not evident to show $p^*_i < p$ or $p^*_i > p$, preventing us from directly doing binary search.

To fix this issue, our main idea is to look at function $R^*(p):=p\cdot (1 - F(p)) + C^* \cdot F(p)$, where $C^*$ represents the revenue from the optimal prices $p^*_{i+1}, \cdots p^*_n$, since $p^*_i$ is now the optimal price for $R^*(p)$. Although function $R^*(p)$ is not directly accessible, we can bound the difference between $C$ and $C^*$, and use the value of $R(p)$ to estimate $R^*(p)$. Running binary search with respect to this virtual function $R^*(p)$ completes Step 2 of our algorithm.

\subsection{Notation}
\label{sec:notation}

For the clarity of our algorithms, in this subsection we first introduce the following notations we will frequently use in our proofs.

\begin{itemize}
     \item $\epsprm$: We define $\epsprm = \frac{\epsilon}{100n^2}$ to be the scaled error parameter. $\epsprm$ also represents the unit length of those confidence intervals given by concentration inequalities.
     \item $R(p_1, \cdots, p_n)$ and $p^*_1, \cdots, p^*_n$: Recall that we defined $R(p_1, \cdots, p_n)$ to be the expected revenue when playing prices $(p_1, \cdots, p_n)$, and $(p^*_1, \cdots, p^*_n)$ represents the optimal prices that maximizes $R(p_1, \cdots, p_n)$.
     \item $(\hat p_1, \cdots, \hat p_n)$: The approximately optimal price vector we find in first-step algorithm. 
     \item $S_i$: The suffix buyers $i, i+1, \cdots, n$. When finding $\hat p_{i-1}$, the performance of buyers $i, i+1, \cdots, n$ can be seen as a whole. $S_i$ represents the whole of this suffix starting from buyer $i$. 
     \item $\ucb^{s}_i$: We define $\ucb^{s}_i ~:=~ \ucb_i - \se.$
     \Cref{lma:newGRMain} suggests that $\ucb_i - \se$ is the highest price with regret guarantee. Therefore, $\ucb^{s}_i$ represents this highest ``safe" price for buyer $i$.
     \item $P_i$: We define $ P_i ~:=~ \prod_{j < i} F_j(\ucb^{s}_j),$
     where $F_j(\cdot)$ is the CDF of $\D_j$. $P_i$ is the probability that the item remains unsold for buyer $i$ when we set  $p_j$ to be the highest safe price $\ucb^{s}_j$ for all $j < i$, i.e., $P_i$ represents the maximum probability of testing buyer $i$.
     \item $\loss(S_i)$: We define $\loss(S_i) := R(1, \cdots, 1, p^*_i, \cdots, p^*_n) - R(1, \cdots,1,\hat p_i, \cdots, \hat pn)$, i.e., $\loss(S_i)$ represents the loss from $S_i$ when playing prices $\hat p_i, \cdots, \hat p_n$.
     \item $R_i(p)$: We define $R_i(p)~:=~R\big(\ucb^{s}_1, \cdot, \ucb^{s}_{i-1}, p_i = p, \hat p_{i+1}, \cdots, \hat p_n\big).$\\
     The intuition of function $R_i(p)$ is: When finding $\hat p_i$, naturally we hope to fix other prices to avoid interference from other buyers. For a buyer $j$ before $i$, we set $p_j$ to be $\ucb^{s}_j$ to maximize the probability of testing $i$. For buyer $k$ from suffix $S_{i+1}$, we set $p_k$ to be the settled $\hat p_k$ to minimize the error from the suffix. 
     \item $\hat p^*_i$: We define $\hat p^*_i := \arg\max_{p \in [\lcb_i, \ucb_i]} R_i(p)$. It should be noted that $\hat p^*_i$ only represents the optimal price in the interval $[\lcb_i, \ucb_i]$.
     \item $R^*_i(p)$:  We define $R^*_i(p) ~:=~ R(\ucb^s_1, \cdots, \ucb^s_{i-1}, p_i = p, p^*_{i+1}, \cdots, p^*_n)$. \\
     The intuition of $R^*_i(p)$ is: notice that $\arg \max_p r^*_i(p) = p^*_i$. A key idea of our algorithm is to approach $p^*_i$ via estimating $R^*_i(p)$. It should be noted that $R^*_i(p)$ can be also written as $R^*_i(p) = R_i(p) + P_i\cdot F_i(p) \cdot\loss(S_{i+1}).$ We will mainly use the second definition in our proofs.
 \end{itemize}

\subsection{$\OTild(\poly(n)\sqrt{T})$ Algorithm for $n$ Buyers: Proof of \Cref{thm:newGRMain}}

We first prove our main \Cref{thm:newGRMain}   for BSPP assuming the following lemma for main sub-routine:

\begin{Lemma}
    \label{lma:newGRMain}
    Given a BSPP instance with $R(p_1, \cdots, p_n)$ being the expected revenue function and $p^*_1, \cdots, p^*_n$ being the optimal prices. 
    Given error parameter $\epsilon \geq {1}/{\sqrt{T}}$ and intervals $[\lcb_1, \ucb_1], \cdots, [\lcb_n, \ucb_n]$ satisfying $p_i^* \in [\lcb_i, \ucb_i]$ for all $i \in [n]$, there exists an algorithm  that tests $O\big(\frac{n^5\log^2 T}{\epsilon^2} \big)$ rounds, where each tested price vector $(p_1, \cdots, p_n)$ satisfies $p_i \in [\lcb_i, \ucb_i - \se]$ for all $i \in [n]$, and  outputs with probability $1 - T^{-3}$ new confidence intervals $[\lcb'_1, \ucb'_1] \subseteq [\lcb_1, \ucb_1], \cdots, [\lcb'_n, \ucb'_n] \subseteq [\lcb_n, \ucb_n]$ satisfying 
\begin{itemize}
    \item For $i \in [n]$, $p_i^* \in [\lcb'_i, \ucb'_i]$.
    \item For every price vector $(p_1, \cdots, p_n)$ that satisfies $p_i \in [\lcb'_i, \ucb'_i - \se]$ for all $i \in [n]$, we have
    \[
    R(p^*_1, \cdots, p^*_n) ~-~ R(p_1, \cdots, p_n) ~\leq~ \epsilon.
    \]
\end{itemize}
\end{Lemma}

The algorithm in \Cref{lma:newGRMain} uses some generalized ideas for single regular buyer. We defer the detailed algorithm and proofs to \Cref{sec:gr32} and  first give the $\OTild(\poly(n)\sqrt{T})$ regret algorithm assuming \Cref{lma:newGRMain} holds.

\begin{algorithm}[tbh]
\caption{$O(n^{2.5}\sqrt{T} \log T)$ Regret Algorithm}
\label{alg:gr-gen}
\KwIn{Hidden revenue function $R(p_1, \cdots, p_n)$, time horizon $T$}
Let $\epsilon_1 \gets 1, [\lcb^{(1)}_j, \ucb^{(1)}_j] \gets [0, 1]$ for all $j \in [n]$, and $i \gets 1$. \\
\While{$\epsilon_i > \frac{n^{2.5}\log T}{\sqrt{T}}$}
{
    Run the algorithm described in \Cref{lma:newGRMain} with $\epsilon_i, [\lcb^{(i)}_1, \ucb^{(i)}_1], \cdots, [\lcb^{(i)}_n, \ucb^{(i)}_n]$ as input, and get $[\lcb'_1, \ucb'_1], \cdots, [\lcb'_n, \ucb'_n]$. Let $[\lcb^{(i+1)}_{j}, \ucb^{(i+1)}_{j}] \gets [\lcb'_j, \ucb'_j]$ for all $j \in [n]$. \\
    Let $\epsilon_{i+1} \gets \frac{1}{2}\epsilon_i$ and $i \gets i+1$. \\
}
Finish remaining rounds with any $(p_1, \cdots, p_n)$ satisfying $ p_j \in [\lcb^{(i)}_j, \ucb^{(i)}_j - \se]$ for all $j \in [n]$.
\end{algorithm}

\begin{proof}[Proof of \Cref{thm:newGRMain}]
    We claim that the regret of \Cref{alg:gr-gen} is $O(n^{2.5}\sqrt{T} \log T)$ with probability $1 - T^{-2}$. \Cref{alg:gr-gen} uses \Cref{lma:newGRMain} for multiple times with a halving error parameter. Assume \Cref{alg:gr-gen} ends with $i = q +1$, i.e., the while loop runs $q$ times. Since we obtain $\epsilon_{i+1} = \epsilon_i /2$ and $\epsilon_q > \frac{n^{2.5}\log T}{\sqrt{T}}$ holds, there must be $q = O(\log T)$.
    
    For $i \in [q]$, let $\alg_i$ represent the corresponding algorithm we call when using \Cref{lma:newGRMain} with $\epsilon_i$ and $\{[\lcb^{(i)}_j, \ucb^{(i)}_j]\}$ as the input. To use \Cref{lma:newGRMain} we need to verify that $p^*_j \in [\lcb^{(i)}_j, \ucb^{(i)}_j]$ for all $i \in [q], j \in [n]$. The condition $p^*_j \in [\lcb^{(1)}_j, \ucb^{(1)}_j] = [0, 1]$ clearly holds. For $i =2, 3, \cdots, q$, the condition $p^*_j \in [\lcb^{(i)}_j, \ucb^{(i)}_j]$ is guaranteed by \Cref{lma:newGRMain} when calling $\alg_{i-1}$. The failing probability of \Cref{lma:newGRMain} is $T^{-3}$. By the union bound, with probability $1 - T^{-2}$, \Cref{lma:newGRMain} always holds.

    Now we give the regret of \Cref{alg:gr-gen}. \Cref{lma:newGRMain}  guarantees that for $i \geq 2$, every price vector we test in $\alg_i$ has regret at most $\epsilon_{i-1}$. Hence, the total regret of \Cref{alg:gr-gen} can be bounded by
    \[
    \textstyle 1\cdot O\big(\epsilon_1^{-2} n^{5}\log^2 T\big) ~+~ \sum_{i = 2}^q \epsilon_{i-1} \cdot O\big(\epsilon^{-2}_i n^{5}\log^2 T\big) ~+~ T \cdot \epsilon_q  ~=~ O(n^5 \log^2 T + n^{2.5}\sqrt{T} \log T).
    \]
    In this paper, we consider the setting that $n$ is a fixed parameter while $T$ goes to infinity, and further assume that $n^{5}=o(\sqrt{T})$, otherwise an $O(\poly(n))$ regret algorithm is trivial. Then, $O(n^5 \log^2 T)$ can be upper bounded by $o(n^{2.5}\sqrt{T} \log T)$ and the regret bound of \Cref{alg:gr-gen} is $O(n^{2.5}\sqrt{T} \log T)$. 
    
    Finally, we also need to check that \Cref{alg:gr-gen} uses no more than $T$ rounds. \Cref{lma:newGRMain} suggests that $\alg_i$ runs $O(\epsilon^{-2}_i n^5\log^2 T)$ rounds. So the total number of rounds in the while loop is
    $
    \sum_{i \in [q]} O(\epsilon^{-2}_i n^5\log^2 T) ~=~ O(T).
    $
    Therefore, \Cref{alg:gr-gen} is feasible.
\end{proof}

\subsection{Two-Steps Sub-Routine: Proof of \Cref{lma:newGRMain}}
\label{sec:gr32}

In this section, we present the main sub-routine for sequential buyers and prove \Cref{lma:newGRMain}. Our algorithm uses a two-steps algorithm similar to the single-buyer setting to update all intervals: the first step is to find a group of approximately optimal prices $(\hat p_1, \cdots, \hat p_n)$, and the second step is to use $(\hat p_1, \cdots, \hat p_n)$ as a benchmark to update the upper and lower bounds via binary search. We show that in general this approach works after overcoming additional challenges. Comparing to the single-buyer algorithm, the main extra challenge  is that the half-concavity does not always hold for sequential buyers. To fix this issue, we define a generalized version of half-concavity and give a new algorithm that works for this generalized half-concavity.

As discussed in \Cref{sec:pfov}, the main sub-routine defined in \Cref{lma:newGRMain} contains two steps: the first step is to find a group of approximately optimal prices $(\hat p_1, \cdots, \hat p_n)$, and the second step is to use $(\hat p_1, \cdots, \hat p_n)$ as a benchmark to update the upper and lower bounds via binary search. We first give corresponding key lemmas for these two steps.

\begin{Lemma}[Step 1]
\label{lma:gr-esp}
Given $\epsprm := \frac{\epsilon}{100n^2} \geq \frac{1}{\sqrt{T}}$, intervals $[\lcb_1, \ucb_1], \cdots, [\lcb_n, \ucb_n]$ that satisfies $p^*_i \in [\lcb_i, \ucb_i]$ for all $i \in [n]$, there exists an algorithm that tests $O(\frac{n\log^2 T}{\epsprm^2})$ rounds with price vectors $(p_1, \cdots, p_n)$ that satisfy $p_i \in [\lcb_i, \ucb^{s}_i]$ for all $i \in [n]$ and outputs $(\hat p_1, \cdots, \hat p_n)$. With probability $1 - T^{-4}$, we have the following for all $i \in [n]$:
\begin{itemize}
    \item $R_i(\hat p^*_i) - R_i(\hat p_i) \leq 4\epsprm$.
    \item $P_i \cdot \loss_i \leq 5(n-i+1) \epsprm$.
\end{itemize}
\end{Lemma}

\begin{Lemma}[Step 2]
\label{lma:gr-bsgen}
    Given $\epsprm := \frac{\epsilon}{100n^2} \geq \frac{1}{\sqrt{T}}$, intervals $[\lcb_1, \ucb_1], \cdots, [\lcb_n, \ucb_n]$ that satisfies $p^*_i \in [\lcb_i, \ucb_i]$ for all $i \in [n]$, and near-optimal prices $\hat p_1, \cdots, \hat p_n$ that satisfy the conditions in \Cref{lma:gr-esp}, there exists an algorithm that tests $O(\frac{n \log^2 T}{\epsprm^2})$ rounds with price vectors $(p_1, \cdots, p_n)$ that satisfy $p_i \in [\lcb_i, \ucb^{s}_i]$ for all $i \in [n]$ and outputs $[\lcb'_1, \ucb'_1], \cdots, [\lcb'_n, \ucb'_n]$. With probability $1 - T^{-4}$, we have the following:
    \begin{itemize}
        \item $[\lcb'_i, \ucb'_i] \subseteq [\lcb_i, \ucb_i]$.
        \item $p^*_i \in [\lcb'_i, \ucb'_i]$.
        \item For every price vector $(p_1, \cdots, p_n)$ that satisfies $p_i \in [\lcb'_i, \ucb'_i - \se]$ for all $i \in [n]$, we have
    \[
    R(p^*_1, \cdots, p^*_n) ~-~ R(p_1, \cdots, p_n) ~\leq~ \epsilon.
    \]
    \end{itemize}
\end{Lemma}

Combining \Cref{lma:gr-esp} and \Cref{lma:gr-bsgen} immediately proves \Cref{lma:newGRMain}: we run the algorithms defined in \Cref{lma:gr-esp} and \Cref{lma:gr-bsgen} sequentially. Both algorithms use $O(\epsprm^{-2} \cdot n \log T) = O(\epsilon^{-2} \cdot n^5 \log T)$ rounds, so the constraint on number of rounds in \Cref{lma:newGRMain} is satisfied. The feasibility of each tested price vector is guaranteed by both \Cref{lma:gr-esp} and \Cref{lma:gr-bsgen}. Finally, for the required conditions of the output, each algorithm succeeds with probability $1 - T^{-4}$. By the union bound, with probability $1 - 2T^{-4} > 1 - T^{-3}$, both algorithms succeed simultaneously. Then, the statements in \Cref{lma:newGRMain} is guaranteed by \Cref{lma:gr-bsgen}.

It only remains to prove \Cref{lma:gr-esp} and \Cref{lma:gr-bsgen}. Next, we prove \Cref{lma:gr-esp} in \Cref{sec:grstep1}, and \Cref{lma:gr-bsgen} in \Cref{sec:grstep2}.

\subsection{Step 1: Finding Near-Optimal $(\hat p_1, \cdots, \hat p_n)$ to Prove \Cref{lma:gr-esp}}
\label{sec:grstep1}

In this section, we give the algorithm defined in \Cref{lma:gr-esp}. Our algorithm is based on a key sub-algorithm that finds $\hat p_i$ such that $R_i(\hat p^*_i) - R_i(\hat p_i) \leq 4\epsprm$. The following lemma states the existence of such a sub-algorithm:

\begin{restatable}{Lemma}{GRSubr}
    \label{lma:gr-subr}
    There exists a sub-algorithm $\subr$ that takes hidden function $R_i(p)$, interval $[\lcb_i, \ucb_i]$, and parameter $\epsprm$ as input, tests $O(\epsprm^{-2} \log^2 T)$ rounds with price vectors $(\ucb^{s}_1, \cdots, \ucb^{s}_{i-1}, p \in [\lcb_i, \ucb^s_i],$ $ \hat p_{i+1}, \cdots, \hat p_n)$,  and outputs $\hat p_i \in [\lcb_i, \ucb^s_i]$ that satisfies 
    $R_i(\hat p^*_i) - R_i(\hat p_i) \leq 4\epsprm $ with probability $1 - T^{-5}$.
\end{restatable}

The design of $\subr$ requires the idea of generalized half-concavity discussed in \Cref{sec:pfov}. 
Since the  proof requires some involved technical details,   we defer  it separately to \Cref{sec:subrproof}. 
With the help of \Cref{lma:gr-subr}, we give our algorithm and the proof of \Cref{lma:gr-esp}.

\begin{algorithm}[tbh]
\caption{Finding $\hat p_1, \cdots, \hat p_n$}
\label{alg:gr-esp}
\KwIn{Hidden Revenue Function $R(p_1, \cdots, p_n)$, Intervals $[\lcb_1, \ucb_1], \cdots, [\lcb_n, \ucb_n]$, Error Parameter $\epsilon$}
Let $\epsprm = \frac{\epsilon}{100n^2}$ be the scaled error parameter. \\
\For{$i = n \to 1$}
{
    Call $\subr$ with input $R_i(p), [\lcb_i, \ucb_i], \epsprm$ and get $\hat p_i$ that satisfies $R_i(\hat p^*_i) - R_i(\hat p_i) \leq 4\epsprm$.
}
\KwOut{$\hat p_1, \cdots, \hat p_n$}
\end{algorithm}

\begin{proof}[Proof of \Cref{lma:gr-esp}]
    We show that \Cref{alg:gr-esp} is the desired algorithm  for \Cref{lma:gr-esp}. \Cref{lma:gr-subr} guarantees that calling $\subr$ one time runs $O(\frac{\log^2 T}{\epsprm^2})$ rounds. Therefore, \Cref{alg:gr-esp} tests $O(\frac{n\log^2 T}{\epsprm^2})$ rounds in total.
    For the succeed probability, \Cref{alg:gr-esp} calls $\subr$ $n$ times. Each call fails with probability $T^{-5}$. By the the union bound\footnote{Recall that we assumed $n= o(T)$.}, with probability $1 - T^{-4}$, all calls of $\subr$ in \Cref{alg:gr-esp} succeed. Then, the constraints  $p_i \in [\lcb_i, \ucb_i - \se]$ and $R_i(\hat p^*_i) - R_i(\hat p_i) \leq 4\epsprm$ are directly guaranteed by \Cref{lma:gr-subr}. 
    
    It only remains to prove $P_i \cdot \loss(S_i) \leq 5(n-i+1)\epsprm$. Now we prove this statement assuming $R_i(\hat p^*_i) - R_i(\hat p_i) \leq 4\epsprm$ holds.

    We do the prove via induction from $i = n$ to $i = 1$. The base case is $i = n$. We have
        \[
        P_n \cdot \loss(S_n) ~=~ R_n(p^*_n) - R_n(\hat p_n) ~=~ R_n(\hat p^*_n) - R_n(\hat p_n) \leq 4\epsprm < 5\epsprm.
        \]
        For the induction step, we prove $P_i \cdot \loss(S_i) \leq 5(n-i+1)\epsprm$ under the induction hypothesis $P_{i+1} \cdot \loss(S_{i+1}) \leq 5(n-i)\epsprm$. We prove the induction step by discussing the following two cases.
        
        \noindent \textbf{Case 1:} $p^*_i \in [\lcb_i, \ucb^s_i]$. In this case, we have
        \begin{align*}
            P_i \cdot \loss(S_i) ~=~& R(\ucb^s_1, \cdots, \ucb^s_{i-1}, p^*_{i}, p^*_{i+1} \cdots, p^*_n) - R(\ucb^s_1, \cdots, \ucb^s_{i-1}, \hat p_{i}, \hat p_{i+1} \cdots, \hat p_n) \\
            =~& R(\ucb^s_1, \cdots, \ucb^s_{i-1}, p^*_{i},p^*_{i+1} \cdots, p^*_n) -  R(\ucb^s_1, \cdots, \ucb^s_{i-1}, p^*_{i}, \hat p_{i+1} \cdots, \hat p_n)\\
            &+ R(\ucb^s_1, \cdots, \ucb^s_{i-1}, p^*_{i}, \hat p_{i+1} \cdots, \hat p_n) -R(\ucb^s_1, \cdots, \ucb^s_{i-1}, \hat p_{i}, \hat p_{i+1} \cdots, \hat p_n) \\
        =~& P_i \cdot F_i(p^*_i) \cdot \loss(S_{i+1}) + R_i(p^*_i) - R_i(\hat p_i) \\
        \leq ~& P_i \cdot F_i(\ucb^s_i) \cdot \loss(S_{i+1}) + R_i(\hat p^*_i) - R_i(\hat p_i) \\
        =~& P_{i+1}\loss(S_{i+1}) + R_i(\hat p^*_i) - R_i(\hat p_i)\\
        \leq~& 5(n-i)\epsprm + 4\epsprm ~<~ 5(n-i+1)\epsprm.
        \end{align*}

    \noindent \textbf{Case 2:} $p^*_i \in [\ucb^s_i, \ucb_i]$.  For this case, we have
    \begin{align*}
        P_i \cdot \loss(S_i) ~=~& R(\ucb^s_1, \cdots, \ucb^s_{i-1}, p^*_{i}, p^*_{i+1} \cdots, p^*_n) - R(\ucb^s_1, \cdots, \ucb^s_{i-1}, \hat p_{i}, \hat p_{i+1} \cdots, \hat p_n) \\
        \leq ~& R(\ucb^s_1, \cdots, \ucb^s_{i-1}, \ucb^s_i, p^*_{i+1} \cdots, p^*_n) - R(\ucb^s_1, \cdots, \ucb^s_{i-1}, \hat p_{i}, \hat p_{i+1} \cdots, \hat p_n) + \epsprm \\
        =~& R(\ucb^s_1, \cdots, \ucb^s_{i-1}, \ucb^s_i,p^*_{i+1} \cdots, p^*_n) -  R(\ucb^s_1, \cdots, \ucb^s_{i-1}, \ucb^s_i, \hat p_{i+1} \cdots, \hat p_n)\\
            &+ R(\ucb^s_1, \cdots, \ucb^s_{i-1}, \ucb^s_i, \hat p_{i+1} \cdots, \hat p_n) -R(\ucb^s_1, \cdots, \ucb^s_{i-1}, \hat p_{i}, \hat p_{i+1} \cdots, \hat p_n) +\epsprm \\
        =~& P_{i+1} \loss(S_{i+1}) + R_i(\ucb^s_i) - R_i(\hat p_i) + \epsprm \\
        \leq~& P_{i+1} \loss(S_{i+1}) + R_i(\hat p^*_i) - R_i(\hat p_i) + \epsprm \\
        \leq~& 5(n-i)\epsprm + 4\epsprm + \epsprm ~=~ 5(n-i+1)\epsprm,
    \end{align*}
    where the first inequality is from the following fact: let $C := R(1, \cdots, 1, p_i = 1, p^*_{i+1}, \cdots, p^*_n)$ be the optimal revenue from suffix $S_i$, we have
    \begin{align*}
         &R(\ucb^s_1, \cdots, \ucb^s_{i-1}, p^*_{i}, p^*_{i+1} \cdots, p^*_n)  - R(\ucb^s_1, \cdots, \ucb^s_{i-1}, \ucb^s_i, p^*_{i+1} \cdots, p^*_n) \\
         ~=~& P_i \cdot (1 - F_i(\ucb^s_i)) \cdot (p^*_i - \ucb^s_i) + P_i \cdot (F_i(p^*_i) - F_i(\ucb^s_i)) \cdot (C - p^*_i) \\
         ~\leq~& (p^*_i - \ucb^s_i) + 0 ~\leq~ \se < \epsprm,
    \end{align*}
    where the second inequality is from the fact that $p^*_i \geq C$, because $p^*_i$ is the maximizer of function $R^*_i(p)$, and $R^*_i(p^*_i) - R^*_i(C) = P_i \cdot (1 - F_i(p^*_i)) \cdot (p^*_i - C) \geq 0$ implies $p^*_i \geq C$.
    
    Therefore, we have $P_i \loss(S_i) \leq 5(n-i+1)\epsprm$ in both cases, which finishes the induction step.
\end{proof}

\subsection{Step 2: Determining New Intervals via Binary Search to Prove \Cref{lma:gr-bsgen}}
\label{sec:grstep2}

The second step of our subroutine is to use prices $\hat p_1, \cdots, \hat p_n$ as a benchmark to determine new confidence intervals. Comparing to the single-buyer setting, there exists an extra challenge for sequential buyers: We only have access to the function $R_i(p)$. However, the optimal price for $R_i(p)$ is $\hat p^*_i$, instead of the desired $p^*_i$. At a first glance, it is unreasonable to determine the confidence interval for $p^*_i$ only with access to function $R_i(p)$. To fix this issue, we introduce
\[
R^*_i(p) ~:=~ R(\ucb^s_1, \cdots, \ucb^s_{i-1}, p_i = p, p^*_{i+1}, \cdots, p^*_n) ~=~ R_i(p) + P_i \cdot F_i(p) \cdot \loss(S_{i+1}).
\]
Then, $R^*_i(p)$ acts as a  bridge between $R_i(p)$ and $p^*_i$, because we have $p^*_i = \arg \max_p R^*_i(p)$, and the difference between $R_i(p)$ and $R^*_i(p)$ can be bounded via the following claim:
\begin{restatable}{Claim}{clmbs}
    \label{clm:gr-loc-opt}
    If \Cref{lma:gr-esp} holds, we have $R_i(p) \leq R^*_i(p) \leq R_i(p) + 5(n-i)\epsprm$ for all $i \in [n], p \in [\lcb_i, \ucb^s_i]$.
\end{restatable}
\begin{proof}
    $R_i(p) \leq R^*_i(p)$ holds simply because $\loss(S_{i+1}) \geq 0$. For $R^*_i(p) \leq R_i(p) + 5(n-i)\epsprm$, we have $P_i \cdot \loss(S_i) \leq 5(n-i+1)\epsprm$ when \Cref{lma:gr-esp} holds. Notice that $P_i\cdot F_i(p) \loss(S_{i+1}) \leq P_{i+1} \loss(S_{i+1}) \leq 5(n-i)\epsprm$ when $p \in [\lcb_i, \ucb^s_i]$. This gives the inequality $R^*_i(p) \leq R_i(p) + 5(n-i)\epsprm$.
\end{proof}

With \Cref{clm:gr-loc-opt}, the idea of our binary search algorithm is clear: We first set $R_i(\hat p_i)$  be the benchmark. For a price $p$, we can estimate $R_i(p)$ and compare $R_i(p) + 5(n-i) \epsprm$ with $R_i(\hat p_i)$. If $R_i(p) + 5(n-i) \epsprm < R_i(\hat p_i)$, it is evident to say $R^*_i(p) < R^*_i(\hat p_i)$, and therefore $p$ can't be $p^*_i$. Otherwise, we keep $p$ as a candidate of $p^*_i$. Since function $R_i(p)$ is single-peaked, the experiment above can be done via two binary search algorithms, and extra steps are needed to take care of the errors from the uncertainty - we only have estimates $\hat R_i(p)$ and $\hat R_i(\hat p_i)$,  instead of $R_i(p)$ and $R_i(\hat p_i)$. Now, we state the following lemma that achieves our idea:

\begin{restatable}{Lemma}{GRBin}
\label{lma:gr-bs}
    Assume \Cref{lma:gr-esp} holds. Given hidden function $R_i(p)$, interval $[\lcb_i, \ucb_i] \ni p^*_i$, scaled $\epsprm:= \frac{\epsilon}{100n^2} > 1/\sqrt{T}$, price $\hat p_i$ that satisfies $R_i(\hat p^*_i) - R_i(\hat p_i) \leq 4\epsprm$, there exists a binary search algorithm that runs $O(\frac{\log^2 T}{\epsprm^2})$ rounds with price vectors $(\ucb^s_1, \cdots, \ucb^s_{i-1}, p_i \in [\lcb_i, \ucb^s_i], \hat p_{i+1}, \cdots, \hat p_n)$ and outputs with probability $1 - T^{-5}$  $[\lcb'_i, \ucb'_i] \subseteq [\lcb_i, \ucb_i]$ that satisfies:
    \begin{itemize}
        \item $p^*_i \in [\lcb'_i, \ucb'_i]$.
        \item For $p \in [\lcb'_i, \ucb'_i - \se]$, $R_i(\hat p^*_i) - R_i(p) \leq 5(n-i)\epsprm + 15\epsprm$.
    \end{itemize}
\end{restatable}

To prove \Cref{lma:gr-bs}, we first give the following \Cref{alg:gr-bs}, and then show \Cref{alg:gr-bs} is the desired algorithm for \Cref{lma:gr-bs}.

\begin{algorithm}[tbh]
\caption{Getting $[\lcb'_i, \ucb'_i]$ via Binary Search}
\label{alg:gr-bs}
\KwIn{Hidden Revenue Function $R_i(p)$, Interval $[\lcb_i, \ucb_i]$, Error Parameter $\epsprm$}
Fix $p_j = \ucb^s_j$ for $j < i$ and $p_j = \hat p_j$ for $j > i$.
\tcp{First binary search to determine $\lcb'_i$}
Test $C \cdot \frac{\log T}{\epsprm^2}$ rounds with price $p_i=\hat p_i$. Let $\hat R_i(\hat p_i)$ be the average reward. \\
Let $\lcb_{b1} \gets \lcb_i, \ucb_{b1} \gets \hat p_i$. \\
\While{$\ucb_{b1} - \lcb_{b1} \geq \se$}
{
    Let $m = \frac{\lcb_{b1}+\ucb_{b1}}{2}$. \\
    Test $C \cdot \frac{\log T}{\epsprm^2}$ rounds with price $p_i=m$. Let $\hat R_i(m)$ be the average reward. \\
    \leIf{$\hat R_i(m) < \hat R_i(\hat p_i) - 2\epsprm - 5(n-i)\epsprm$}{Update $\ucb_{b1} \gets m$}{Update $\lcb_{b1} \gets m$}
}
Test $p_i=\lcb_{b1}$ and $p_i=\ucb_{b1}$ each for $C \cdot \frac{\log T}{\epsprm^2}$ rounds. Let $\hat R_i(\lcb_{b1}), \hat R_i(\ucb_{b1})$ be the average reward. \\
\leIf{$\hat R_i(\lcb_{b1}) < \hat R_i(\ucb_{b1}) - 3\epsprm$}{Let $\lcb'_i \gets \ucb_{b1}$}{Let $\lcb'_i \gets \lcb_{b1}$}
\tcp{Second binary search to determine $\ucb'_i$}
Test $C \cdot \frac{\log T}{\epsprm^2}$ rounds with price $p_i=\ucb_i - \se$. Let $\hat R_i(\ucb_i - \se)$ be the average reward. \\
\eIf{$\hat R_i(\ucb - \se) \geq \hat R_i(\hat p_i) - 2\epsprm$}
{
    \tcp{Special Case: original upper bound is good}
    Let $\ucb'_i \gets \ucb_i$
}
{
Let $\lcb_{b2} \gets \hat p_i, \ucb_{b2} \gets \ucb_i - \se$. \\
\While{$\ucb_{b2} - \lcb_{b2} \geq \se$}
{
    Let $m = \frac{\lcb_{b2}+\ucb_{b2}}{2}$. \\
    Test $C \cdot \frac{\log T}{\epsprm^2}$ rounds with price $p_i=m$. Let $\hat R_i(m)$ be the average reward. \\
    \leIf{$\hat R_i(m) < \hat R_i(\hat p) - 2\epsprm - 5(n-i)\epsprm$}{Update $\ucb_{b2} \gets m$}{Update $\lcb_{b2} \gets m$}
}
Let $\ucb'_i \gets \ucb_{b2}$.
}
\KwOut{$[\lcb'_i, \ucb'_i]$. }
\end{algorithm}

\begin{proof}[Proof of \Cref{lma:gr-bs}]
    We show \Cref{alg:gr-bs} satisfies the statements in \Cref{lma:gr-bs}. Our proof starts from verifying \Cref{alg:gr-bs} runs $O(\frac{\log^2 T}{\epsprm^2})$ rounds with $p_i \in [\lcb_i, \ucb_i - \se]$. For the number of tested rounds, \Cref{alg:gr-bs} runs two standard binary searches. In one binary search, the subroutine tests one price with $C \cdot \frac{\log T}{\epsprm^2}$ rounds, and the binary search stops when the length of the interval is less than $\se$. Therefore, The total number of rounds is $C \cdot \frac{\log T}{\epsprm^2} \cdot O(\log T^{100}) = O(\frac{\log^2 T}{\epsilon^2})$. For the range constraint of all tested prices, \Cref{alg:gr-bs} directly guarantees that every tested price falls in $[\lcb_i, \ucb_i - \se]$.

    Now we move to the two main statements. We first give the following claim:
    \begin{Claim}
    \label{clm:gr-step2-con}
    In \Cref{alg:gr-bs}, $|\hat R_i(p) - R_i(p)| \leq \epsprm|$ simultaneously holds for every tested price $p$ with probability at least $1 - T^{-5}$.
    \end{Claim}
    \begin{proof}
        For a single tested price $p$, $\hat R(p)$ is estimated by calculating the average of $N = C \cdot \frac{\log T}{\epsprm^2}$ samples. By Hoeffding's Inequality, 
        \[
        \Pr{|\hat R_i(p) - R_i(p)| > \epsprm} ~\leq~ 2\exp\left(-2 N \epsprm^2\right) ~=~ 2T^{-2C} ~<~ T^{-6}.
        \]
        The last inequality holds when $C$ is a constant greater than $4$. Then, we have $|\hat R_i(p) - R_i(p)| \leq \epsprm$ holds with probability $1 - T^{-6}$ for a single tested price $p$.
        
        Notice that \Cref{alg:gr-bs} can't test more than $T$ prices. By the union bound, $|\hat R(p) - R(p)| \leq \epsprm$ simultaneously holds for all tested prices with probability $1 - T^{-5}$.
    \end{proof}
    We will give the remaining parts of the proof assuming  $|\hat R_i(p) - R_i(p)| \leq \epsprm$ always holds.
    
    Now we show $p^*_i \in [\lcb'_i, \ucb'_i]$. The key idea is to use \Cref{clm:gr-loc-opt}. For clarity we first restate \Cref{clm:gr-loc-opt}:
    \clmbs*

    Our goal is to use binary search to determine $I' := \{p \in [\lcb'_i, \ucb'_i]: R_i(p) + 5(n-i)\epsprm \geq R_i(\hat p_i)\}$. For $p \in [\lcb'_i, \ucb'_i] \setminus I'$, there must be 
    \[
    R^*_i(p) ~\leq~ R_i(p) + 5(n-i)\delta ~<~ R_i(\hat p_i) ~\leq~ R^*_i(\hat p_i) ~\leq R^*_i(p^*_i).
    \]
    Recall that $p^*_i$ is the optimal price for $R^*_i(p)$. Therefore, the price $p \in [\lcb'_i, \ucb'_i] \setminus I'$ can't be $p^*_i$, i.e.,  $p^*_i$ falls in $I'$.

    Now we prove $p^*_i \in [\lcb'_i, \ucb'_i]$. We first show that $p^*_i \geq \lcb_{b1}$. Here, $\lcb_{b1}$ refers to the value after the while loop is finished. $p^*_i \geq \lcb_{b1}$ holds because $\lcb_{b1}$ is updated to be at least $m$ when we observe $\hat R_i(m) < \hat R_i(\hat p_i) - 2\epsprm - 5(n-i)\epsprm$. Combining this with \Cref{clm:gr-step2-con} implies $R_i(m) < R_i(\hat p_i) - 5(n-i)\epsprm$, i.e., $m \notin I'$. Since $R_i(p)$ is single-peaked, it's sufficient to update $\lcb_{b1}$ to at least $m$ and $I' \subseteq [\lcb_{b1}, 1]$ holds. Symmetrically, a similar argument gives $p^*_i \leq \ucb_{b2} = \ucb'_i$.

    It remains to show $p^*_i \geq \lcb'_i$. $\lcb'_i$ and $\lcb_{b1}$ are different when we have $\hat R(\lcb_{b1}) < \hat R(\ucb_{b1}) - 3\epsprm$. Combining this inequality with \Cref{clm:gr-step2-con} gives $R(\ucb_{b1}) - R(\lcb_{b1}) > \epsprm > \ucb_{b1} - \lcb_{b1}$. This is in contrast to the constrained Lipschitzness from \Cref{clm:GRHalfCon}. Then, there must be $\ucb_{b1} < \rev(S_{i+1}) < \opt(S_{i+1}) \leq p^*_i$, where the last inequality $p^*_i \geq \opt(S_{i+1})$ is implied by the fact that $R^*_i(p^*_i) - R^*_i(\opt(S_{i+1})) = P_i \cdot (1 - F_i(p^*_i)) \cdot (p^*_i - \opt(S_{i+1})) \geq 0$. Therefore, it's feasible to update $\lcb'_i$ to $\ucb_{b1}$ in this case.

    Finally, we show $R_i(\hat p^*_i) - R_i(p) \leq 5(n-i)\epsprm + 15\epsprm$ for all $p \in [\lcb'_i, \ucb'_i - \se]$. Function $R(p)$ is single-peaked, so it is sufficient to prove $R_i(\hat p^*_i) - R_i(\lcb'_i) \leq 5(n-i)\epsprm + 11\epsprm$ and $R_i(\hat p^*_i) - R_i(\ucb'_i - \se) \leq 5(n-i)\epsprm + 11\epsprm$. Then, combining these two inequalities with the assumption $R_i(\hat p^*_i) - R_i(\hat p_i) \leq 4 \epsprm$ gives the desired statement.

    We first prove that $R_i(\hat p_i) - R(\lcb'_i) \leq 5(n-i)\epsprm + 11\epsprm$. The binary search subroutine guarantees that $\hat R_i(\ucb_{b1}) \geq \hat R_i(\hat p_i) - 2\epsprm - 5(n-i)\epsprm$ after the binary search loop is finished. Then, \Cref{clm:gr-step2-con} ensures that $R_i(\hat p_i) - R_i(\ucb_{b1}) \leq 5(n-i)\epsprm + 4\epsprm$. Then, if $\hat R_i(\lcb_{b1}) < \hat R_i(\ucb_{b1}) - 3\epsprm$, we set $\lcb'_i$ to be $\ucb_{b1}$. Otherwise, we have $\hat R_i(\lcb_{b1}) \geq \hat R_i(\ucb_{b1}) - 3\epsprm$ and $\lcb'_i$ is set to be $\lcb_{b1}$. Then,
    \begin{align*}
        R_i(\hat p_i) - R_i(\lcb'_i) ~\leq&~ R_i(\ucb_{b1}) - R_i(\lcb_{b1}) + 5(n-i)\epsprm + 4\epsprm \\
        ~\leq&~ \hat R_i(\ucb_{b1}) - \hat R_i(\lcb_{b1}) + 5(n-i)\epsprm + 6\epsprm \\
        ~\leq&~ 5(n-i)\epsprm + 9\epsprm~<~5(n-i)\epsprm + 11\epsprm.
    \end{align*}
    Therefore, in both cases $R_i(\hat p_i) - R(\lcb'_i) \leq 11(n-i)\epsprm + 6\epsprm$ holds.

    For the inequality $R_i(\hat p_i) - R_i(\ucb'_i - \se) \leq 5(n-i)\epsprm + 11\epsprm$,  a symmetric proof for the second binary search subroutine gives $R_i(\hat p) - R_i(\lcb_{b2}) \leq 5(n-i)\epsprm + 4\epsprm$, while $\ucb'_i$ is set to be $\ucb_{b2} < \lcb_{b2} + \se$. If $\hat p^*_i < \ucb'_i - \se$, the single-peakedness of $R(p)$ gives 
    \[
    R_i(\ucb'_i - \se) ~\geq~ R_i(\lcb_{b2}) ~\geq~ R_i(\hat p) - 5(n-i)\epsprm -  4\epsprm.\]
    Otherwise, the constrained Lipschitzness of $R_i(p)$ gives
    \[
    R_i(\ucb'_i - \se) ~\geq~ R_i(\lcb_{b2}) - (\lcb_{b2} - \ucb_{b2} + \se) ~\geq~ R_i(\hat p) - 5(n-i)\epsprm - 4\epsprm - \se > R_i(\hat p) - 5(n-i)\epsprm - 5\epsprm.    \qedhere
    \]
\end{proof}

Finally, we end this section proving  \Cref{lma:gr-bsgen} via \Cref{lma:gr-bs}.

\begin{proof}[Proof of \Cref{lma:gr-bsgen}]
    The algorithm for \Cref{lma:gr-bsgen} is to use \Cref{lma:gr-bs} $n$ times and get $n$ new confidence intervals. \Cref{lma:gr-bs} succeeds with probability $1 - T^{-5}$. By the union bound, with probability $1 - n \cdot T^{-5} > 1 - T^{-4}$, our algorithm for \Cref{lma:gr-bsgen} succeeds, and the conditions $[\lcb'_i \ucb'_i] \subseteq [\lcb_i, \ucb_i]$ and $p^*_i \in [\lcb'_i, \ucb'_i]$ are guaranteed by \Cref{lma:gr-bs}. It only remains to show that each price vector in the new intervals has regret no more than $\epsilon$.

    We instead prove the following statement via induction, assuming the statements in \Cref{lma:gr-esp} and \Cref{lma:gr-bs} hold: For $p_i, \cdots, p_n$ that satisfies $p_j \in [\lcb'_j, \ucb'_j - \se]$ for all $j \in \{i, \cdots, n\}$, we have
     \[
     R(\ucb^s_1, \cdots, \ucb^s_{i-1}, p^*_i, \cdots, p^*_n) ~-~ R(\ucb^s_1, \cdots, \ucb^s_{i-1}, p_i, \cdots, p_n) ~\leq~ 50(n-i+1)^2\epsprm.
     \]

     The base case is $i = n$. \Cref{lma:gr-bs} gives
    \[
    R(\ucb^s_1, \cdots, \ucb^s_{n-1}, p^*_n) ~-~ R(\ucb^s_1, \cdots, \ucb^s_{n-1}, p_n) ~\leq~ 15\epsprm < 50\epsprm.
    \]

    For the induction step, we hope to show 
    \[
     R(\ucb^s_1, \cdots, \ucb^s_{i-1}, p^*_i, \cdots, p^*_n) ~-~ R(\ucb^s_1, \cdots, \ucb^s_{i-1}, p_i, \cdots, p_n) ~\leq~ 50(n-i+1)^2\epsprm
     \]
     under the induction hypothesis 
     \[
     R(\ucb^s_1, \cdots, \ucb^s_{i-1}, \ucb^s_i, p^*_{i+1},\cdots, p^*_n) ~-~ R(\ucb^s_1, \cdots, \ucb^s_{i-1}, \ucb^s_i, p_{i+1},\cdots, p_n) ~\leq~ 50(n-i)^2\epsprm.
     \]
     We first observe that in the induction hypothesis, if we change the price of buyer $i$ from $\ucb^s_i$ to $p_i$, it only reduces the probability of incurring a loss. Therefore, the induction hypothesis also gives
     \[
     R(\ucb^s_1, \cdots, \ucb^s_{i-1}, p_i, p^*_{i+1},\cdots, p^*_n) ~-~ R(\ucb^s_1, \cdots, \ucb^s_{i-1}, p_i, p_{i+1},\cdots, p_n) ~\leq~ 50(n-i)^2\epsprm.
     \]
     Now we finish the induction step:
    \begin{align*}
        R(\ucb^s_1, \cdots, \ucb^s_{i-1}, p^*_i, p^*_{i+1}, \cdots, p^*_n) ~\leq&~ R(\ucb^s_1, \cdots, \ucb^s_{i-1}, \hat p_i, \hat p_{i+1}, \cdots, \hat p_n) + 5(n-i+1)\epsprm \\
        ~\leq&~ R(\ucb^s_1, \cdots, \ucb^s_{i-1}, \hat p^*_i, \hat p_{i+1}, \cdots, \hat p_n) + 5(n-i+1)\epsprm + 4\epsprm \\
        ~\leq&~ R(\ucb^s_1, \cdots, \ucb^s_{i-1}, p_i, \hat p_{i+1}, \cdots, \hat p_n) + 10(n-i+1)\epsprm + 14\epsprm \\
        ~\leq&~ R(\ucb^s_1, \cdots, \ucb^s_{i-1}, p_i, p^*_{i+1}, \cdots, p^*_n) + 10(n-i+1)\epsprm + 14\epsprm  \\
        ~\leq&~ R(\ucb^s_1, \cdots, \ucb^s_{i-1}, p_i, p_{i+1}, \cdots, p_n) \\
        &+ 10(n-i+1)\epsprm + 14\epsprm +50(n-i)^2\epsprm \\
        ~\leq&~ R(\ucb^s_1, \cdots, \ucb^s_{i-1}, p_i, p_{i+1}, \cdots, p_n) + 50(n-i+1)^2\epsprm.
    \end{align*}
    Here, the first inequality is from the inequality $P_i \loss(S_i) \leq 5(n-i+1)\epsprm$ in \Cref{lma:gr-esp}. The second inequality is from \Cref{lma:gr-esp} that guarantees $R_i(\hat p^*_i) - R_i(\hat p_i) \leq 4\epsprm$. The third inequality is from \Cref{lma:gr-bs} that guarantees $R_i(\hat p^*_i) - R_i(p_i) \leq 5(n-i)\epsprm + 15\epsprm$. The fourth inequality is from the fact that prices $p^*_{i+1}, \cdots, p^*_n$ are optimal for suffix $S_{i+1}$. The fifth inequality is from the corollary of the induction hypothesis. Comparing the start and the end of the above inequality finishes the induction step.

    Finally, taking $i = 1$ with the fact that $\epsprm = \frac{\epsilon}{100n^2}$ finishes the proof of \Cref{lma:gr-bsgen}. 
\end{proof}

\subsection{Sub-Algorithm $\subr$: Proof of Lemma \ref{lma:gr-subr}}
\label{sec:subrproof}

We now introduce the sub-algorithm $\subr$ that finds $\hat p_i$ satisfying $R_i(\hat p^*_i) - R_i(\hat p_i) \leq 4\epsprm$, and prove the missing \Cref{lma:gr-subr} that will complete the proof of Step 1. We restate the lemma for convenience.

\GRSubr*

We start from introducing the following new notations that we need in this subsection:
\begin{itemize}
    \item $\pre(S_i)$: We define $\pre(S_i) := R(\ucb^{s}_1, \cdots, \ucb^{s}_{i-1}, 1, \cdots, 1)$ to be the expected revenue from the buyers before $S_i$ with prices $\ucb^{s}_1, \cdots, \ucb^{s}_{i-1}$.
    \item $\opt(S_i)$: We define $\opt(S_i):=R(1, \cdots, 1, p^*_i, \cdots, p^*_n)$ to be the optimal revenue from $S_i$.
    \item $\rev(S_i)$: We define $\rev(S_i):=R(1, \cdots,1,\hat p_i, \cdots, \hat pn)$ to be the near-optimal revenue from $S_i$ with near-optimal prices $\hat p_i, \cdots, \hat p_n$.
\end{itemize}

With the help of these new notations, we can rewrite $R_i(p)$ as
\[
R_i(p) ~=~ \pre(S_i) + P_i(p \cdot (1 - F_i(p)) + \rev(S_{i+1}) \cdot F_i(p)).
\]
Comparing to the revenue function of a single buyer, function $R_i(p)$ has an extra $\rev(S_{i+1}) \cdot F_i(p)$ term, and therefore   half-concavity is not directly applicable. The good news is that $R_i(p)$ is still half-concave in part of the interval:

\begin{restatable}{Claim}{GRHalfCon}
\label{clm:GRHalfCon}
    Function $R_i(p)$ satisfies the following properties:
    \begin{itemize}
        \item $R_i(p)$ is single-peaked in $[0, 1]$ with $\arg \max_{p \in [0, 1]} \geq \rev(S_{i+1})$.
        \item $R_i(p)$ has constrained Lipschitzness: For $b \in [\rev(S_{i+1}), 1]$ and $a \in [0, b]$, we have $R(b) - R(a) \leq P_i(b - a)$.
        \item $R_i(p)$ is half-concave in $[\rev(S_{i+1}), 1]$.
    \end{itemize}
\end{restatable}

\begin{proof}

By the definition of $R_i(p)$, we have $R_i(p) = P_i(p\cdot (1-F_i(p) + C_i\cdot F_i(p))$, where $P_i := \prod_{j < i} F_j(\ucb_j)$ is the probability of reaching buyer $i$. Now we prove the three properties one-by-one:
\noindent (i) \emph{Single-peakedness}: Let $R_q(q)$ be the revenue function in the quantile space that correspond to $R_i(p)$. By definition we have $R_q(q) = P_i(q^{-1}(q) \cdot (1-q) + C_i \cdot q)$. The key observation is that comparing to the original concave revenue function $q^{-1}(q) \cdot q$, $R_q(q)$ has an extra \textit{linear} term, so it still remains concave. Therefore, function $R_i(p)$ is single-peaked in the value space.

We also need to show that $\arg \max_p R_i(p) \geq \rev(S_{i+1})$. This is true because for all $p \in [0, \rev(S_{i+1})]$, 
\begin{align*}
    R_i(p) ~=&~ \pre(S_i) + P_i(p \cdot (1 - F_i(p)) + \rev(S_{i+1}) \cdot F_i(p)) \\
    ~\leq&~ \pre(S_i) + P_i(\rev(S_{i+1}) \cdot (1 - F_i(p)) + \rev(S_{i+1}) \cdot F_i(p)) \\
    ~=&~ \pre(S_i) + P_i(\rev(S_{i+1}) \cdot (1 - F_i(\rev(S_{i+1}))) + \rev(S_{i+1}) \cdot F_i(\rev(S_{i+1}))) \\
    ~=&~ R_i(\rev(S_{i+1})).
\end{align*}
The above inequality shows $R_i(\rev(S_{i+1}))$ is not worse than any price in $[0, \rev(S_{i+1})]$. Therefore, there must be $\arg \max_p R_i(p) \geq \rev(S_{i+1})$.

\noindent (ii) \emph{Constrained Lipschitzness}: For $b \geq \rev(S_{i+1})$ and $a \leq b$, we have
\[
R_i(b) - R_i(a) ~=~ P_i((b - a) - (b - \rev(S_{i+1}))F_i(b) + (a - \rev(S_{i+1}))F_i(a)) ~\leq~ P_i(b-a).
\]
Here, the inequality is true because we have $b - \rev(S_{i+1}) \geq a - \rev(S_{i+1})$, $b - \rev(S_{i+1}) \geq 0$, and $F_i(b) \geq F_i(a) \geq 0$. Then, for quantities $A, B, C, D$ satisfying $A \geq 0, A \geq B, C \geq D \geq 0$, it's easy to verify that $A \cdot C \geq B \cdot D$.

\noindent (iii) \emph{Half-Concavity in $[\rev(S_{i+1}), 1]$}: The single-peakedness and the Lipschitzness required by the definition of half-concavity are both proved. Let $p^* = \argmax_{p \in [0, 1]} R_i(p)$. We proved that $p^* \geq \rev(S_{i+1})$, then it only remains to show $R_i(p)$ is concave in $[\rev(S_{i+1}), p^*]$.

We prove by contradiction. Assume there exists $\rev(S_{i+1}) \leq a < b < c \leq p^*$ satisfying $c = b + t\cdot (b - a)$ and $R(c) > R(b) + t\cdot (R(b) - R(a))$. Define $q(x):=1 - F(x)$ and $\overline{q} := q(b) + t\cdot \big(q(b) - q(a) \big)$. We first show that $\overline{q} \geq  q(c)$: 
Since the revenue function is concave in the quantile space, hence, for all $q' \geq \overline{q}$,  
\[ R(q^{-1}(q')) ~\leq~ R(b) + \frac{q(b) - q'}{q(a) - q(b)}\cdot (R(b) - R(a)) ~\leq~ R(b) + t\cdot (R(b) - R(a)) ~<~ R(c). \] 
For every $q' \geq \bar q$, we have $q' \neq q(c)$. So. we must have $q(c) < \bar q$.
This gives the desired contradiction: 
\begin{align*}
    \frac{R(c)}{P_i} ~>~& \frac{R(b)}{P_i} + \frac{t\cdot (R(b) - R(a))}{P_i} \\
     ~=~ &b\cdot q(b) + t\cdot \big(b \cdot q(b) - a\cdot q(a) \big) + \rev(S_{i+1})\cdot (1 - q(b) + t\cdot (q(a) - q(b))) \\
     ~\geq~& b\cdot q(b) + t\cdot \big(b \cdot q(b) - a\cdot q(a) \big) + \rev(S_{i+1})\cdot (1 - q(b) + t\cdot (q(a) - q(b))) \\
     &- t(t+1)(b-a)(q(a)-q(b)) \\
    ~=~ &\big((t+1)b-ta \big) \cdot \big((t+1)q(b)-tq(a) \big) \rev(S_{i+1})\cdot (1 - q(b) + t\cdot (q(a) - q(b)))   \\
    ~=~ &c\cdot \overline{q} + \rev(S_{i+1})\cdot (1 - \overline q) ~\geq~ c \cdot q(c) + \rev(S_{i+1})\cdot (1-q(c))  ~=~ \frac{R(c)}{P_i}.   \qedhere
\end{align*}
\end{proof}

The claim shows that $R_i(p)$ is half-concave, and therefore learnable in $[\rev(S_{i+1}), 1]$. However, in $[0, \rev(S_{i+1})]$ the function is increasing but not necessarily concave. The  sub-algorithm $\subr$ will involve different cases to handle  non-concavity   inside interval $[0, \rev(S_{i+1})]$. Depending on the value of $P_i$ and $F_i(\ucb^{s}_i)$, we run different algorithms. See \Cref{alg:gr-subr} for the pseudo-code of $\subr$.

\begin{algorithm}[tbh]
\caption{Sub-algorithm $\subr$}
\label{alg:gr-subr}
\KwIn{Hidden function $R_i(p)$, interval $[\lcb_i, \ucb_i]$, error parameter $\epsprm$}
Fix $p_j = \ucb^{s}_j$ for $j < i$ and $p_j = \hat p_j$ for $j > i$. \\
Test $N = C \cdot \frac{\log T}{\epsprm^2}$ rounds with $p_i = \ucb^s_i$ to estimate $\hat P_i = \frac{\sum_{i \in [N]} \one[\text{Item not sold before }i]}{N}$. \\
\If{$\hat P_i < \frac{3}{4} \epsprm$}
{
    \textbf{Output} $\hat p_i = \lcb_i$ and \textbf{Return.} \tcp{Case 1: $P_i$ too small, any price in $[\lcb_i, \ucb^s_i]$ works.}
}

Test $C \cdot \frac{\log T}{\epsprm^2}$ rounds with $p_i = \ucb^s_i$ to estimate $\hat F_i(\ucb^s_i) := \frac{\sum \one[\text{Item not sold before }i+1]}{\sum \one[\text{Item not sold before }i]}$.\\
\eIf{$\hat F_i(\ucb^s_i) \geq 0.4$}
{
    \tcp{Case 2: $F_i(\ucb^s_i) \geq 0.3$, estimate $\rev(S_{i+1})$ accurately} 
    Test $C\cdot \frac{\log T}{\epsprm^2}$ rounds with $p_i = \ucb^s_i$ to get $\widehat \rev(S_{i+1}) := \frac{\sum \text{Revenue if item not sold before }i+1}{\sum \one[\text{Item not sold before }i + 1]}$. \\
    Call \Cref{alg:sr-esp} with input $R_i(p)$, interval $[\max\{\lcb_i, \widehat \rev(S_{i+1}) + \frac{\epsprm}{\hat P_i}\}, \ucb_i]$, $\epsprm$, and receive output $\hat p^{(1)}_i$. \\
    Test $C\cdot \frac{\log T}{\epsprm^2}$ rounds with $p_i = \hat p^{(1)}_i$. Let $\hat R_i(\hat p^{(1)}_i)$ be the average reward. \\
    Let $\hat p^{(2)}_i = \max\{\lcb_i, \widehat \rev(S_{i+1}) - \frac{\epsprm}{\hat P_i}\}$. \\
    Test $C\cdot \frac{\log T}{\epsprm^2}$ rounds with $p_i = \hat p^{(2)}_i$. Let $\hat R_i(\hat p^{(2)}_i)$ be the average reward. \\
    Let $\hat p_i = \arg \max_{p \in \{\hat p^{(1)}_i, \hat p^{(2)}_i\}} \hat R_i(p)$.
}
{
    \tcp{Case 3: $F_i(\ucb^s_i) \leq 0.5$, either $R_i(p)$ is half 2-concave or $R_i(\ucb^s_i)$ is good.}
    Call $\alg$ and receive output $\hat p^{(3)}_i$. \\
    Test $C\cdot \frac{\log T}{\epsprm^2}$ rounds with $p_i = \hat p^{(3)}_i$. Let $\hat R_i(\hat p^{(3)}_i)$ be the average reward. \\
    Let $\hat p^{(4)}_i = \ucb^s_i$. \\
    Test $C\cdot \frac{\log T}{\epsprm^2}$ rounds with $p_i = \hat p^{(4)}_i$. Let $\hat R_i(\hat p^{(4)}_i)$ be the average reward. \\
    Let $\hat p_i = \arg \max_{p \in \{\hat p^{(3)}_i, \hat p^{(4)}_i\}} \hat R_i(p)$.
}
\KwOut{$\hat p_i$}
\end{algorithm}

There are three cases in \Cref{alg:gr-subr}. We first briefly discuss the intuition behind each case.

\noindent \textbf{Case 1} ($P_i$ is small): Since the difference between $R_i(\hat p^*_i)$ and any other $R_i(p)$ is upper bounded by $P_i$, playing any price in $[\lcb_i, \ucb^s_i]$ does not incur a big regret. Therefore, taking $\hat p_i$ to be an arbitrary price in $[\lcb_i, \ucb^s_i]$ is feasible.

\noindent \textbf{Case 2} ($F_i(\ucb^s_i)$ is large): In this case,  the value of $\rev(S_{i+1})$ can be estimated accurately. Then, we divide interval $[\lcb_i, \ucb^s_i]$ from the upper bound of $\rev(S_{i+1})$ into two parts.  For the part on the right, \Cref{clm:GRHalfCon} guarantees that $R_i(p)$ is half-concave, and we call the single-buyer algorithm to find a candidate price. For the part on the left, the constrained Lipschitzness from \Cref{clm:GRHalfCon} suggests that we can mark the lower bound of $\rev(S_{i+1})$ as a candidate price. Finally, we take a better one from two candidate prices to be $\hat p_i$.

\noindent \textbf{Case 3} ($F_i(\ucb^s_i)$ is small): In this case, we find two candidate prices according to the value of $\ucb^s_i$ and $\rev(S_{i+1})$. If $\ucb^s_i \geq \rev(S_{i+1})$, the following lemma gives the first candidate price:

\begin{restatable}{Lemma}{grcase}
    \label{lma:gr-case2}
    If $F_i(\ucb^s_i) \leq 0.5$ and $\ucb^s_i \geq \rev(S_{i+1})$, there exists an algorithm $\alg$ that runs $O(\epsprm^{-2} \log^2 T)$ rounds with price vectors $(\ucb^s_1, \cdots, \ucb^s_{i-1}, p_i \in [\lcb_i, \ucb^s_i], \hat p_{i+1}, \cdots, \hat p_n)$ and outputs $\hat p_i$ that satisfies with probability $1 - T^{-6}$ that $R_i(\hat p^*_i) - R_i(\hat p_i) \leq 4\epsprm$.
\end{restatable}

\Cref{lma:gr-case2} suggests that $R_i(p)$ is learnable when $F_i(\ucb^s_i)$ is small and $\ucb^s_i \geq \rev(S_{i+1})$ holds. The proof of \Cref{lma:gr-case2} requires ideas of generalized half-concavity. For consistency, we first assume the correctness of \Cref{lma:gr-case2}, and defer the details to \Cref{sec:gr334}.

On the other hand, if instead $\rev(S_{i+1}) > \ucb^s_i$ holds, we use $\ucb^s_i$ to be the second candidate price. The main reason is from the following claim:

\begin{Claim}
\label{clm:gr-hat-star-rev}
    Assume $p^*_i \in [\lcb_i, \ucb_i]$, then there must be $\hat p^*_i \geq \rev(S_{i+1})$.
\end{Claim}
\begin{proof}
    Suppose $\hat p^*_i < \rev(S_{i+1})$. Then, we have $p^*_i \geq \opt(S_{i+1}) \geq \rev(S_{i+1}) > \hat p^*_i$, where the first inequality is implied by the observation that 
    \[
    R^*_i(p^*_i) - R^*_i(\opt(S_{i+1})) ~=~ P_i \cdot (1 - F_i(p^*_i)) \cdot (p^*_i - \opt(S_{i+1})) ~\geq~ 0.
    \] Since we have $\hat p^*_i, p^*_i \in [\lcb_i, \ucb_i]$, there must be $\rev(S_{i+1}) \in [\lcb_i, \ucb_i]$. This leads to  a contradiction: \Cref{clm:GRHalfCon} guarantees $R_i(p)$ is increasing in $[0, \rev(S_{i+1})]$, which is in contrast to the assumption that $\hat p^*_i < \rev(S_{i+1})$. Therefore, there must be $\hat p^*_i \geq \rev(S_{i+1})$.
\end{proof}
\Cref{clm:gr-hat-star-rev} together with the inequality $\rev(S_{i+1}) > \ucb^s_i$ implies $\hat p^*_i \in [\ucb^s_i, \ucb_i]$. Now the constrained Lipschitzness of $R_i(p)$ implies $R_i(\hat p^*_i) - R_i(\ucb^s_i) \leq \hat p^*_i - \ucb^s_i \leq \se$. Therefore, it's sufficient to take $\ucb^s_i$ as the second candidate price. Then, the algorithm takes  better of the two  prices for Case~3.

Now, we prove \Cref{lma:gr-subr}. Our proof starts from the following concentration bounds for $\subr$:

\begin{restatable}{Claim}{grsubrcon}
        \label{clm:gr-subrcon}
        In \Cref{alg:gr-subr}, with probability $1 - T^{-6}$ the following bounds simultaneously hold:
        \begin{itemize}
            \item $|\hat P_i - P_i| \leq \frac{1}{4} \epsprm$.
            \item $|\hat F_i(\ucb^s_i) - F_i(\ucb^s_i)| \leq 0.1$.
            \item $|\widehat \rev(S_{i+1}) - \rev(S_{i+1})| \leq \frac{\epsprm}{\hat P_i}$.
            \item $|\hat R_i(\hat p^{(j)}_i) -R_i(\hat p^{(j)}_i)| \leq \frac{1}{2}\epsprm$ for all $j \in \{1, 2, 3, 4\}$.
        \end{itemize}
\end{restatable}

\begin{proof}
We show each statement holds with probability $1 - T^{-7}$ and the union bound guarantees all statements hold with probability $1 - T^{-6}$.

\noindent (i) $|\hat P_i - P_i| \leq \frac{1}{4} \epsprm$: $\hat P_i$ is estimated with $N = C \cdot \epsprm^{-2}\log T$ samples. By Hoeffding's Inequality, we have
\[
\Pr{|\hat P_i - P_i| > \frac{1}{4} \epsprm} ~\leq~ 2\exp(-2N\cdot \frac{\epsprm^2}{16}) ~=~ 2T^{C/8} ~\leq T^{-7},
\]
where the last inequality holds when $C$ is a constant greater than $60$.

\noindent (ii) $|\hat F_i(\ucb^s_i) - F_i(\ucb^s_i)| \leq 0.1$. We assume $|\hat P_i - P_i| \leq \frac{1}{4} \epsprm$ holds. Then, the algorithm estimates  $\hat F_i(\ucb^s_i)$ only when $P_i \geq \hat P_i - \frac{1}{4}\epsprm \geq \frac{1}{2}\epsprm$.

In the algorithm, we run $N' = C \frac{\log T}{\epsprm^2}$ rounds to estimate $\hat F_i(\ucb^s_i)$. However, a single round becomes a valid sample only when the item is unsold until buyer $i$. We claim that with probability $1 - T^{-7}$ we have at least $N = 400\log T$ samples. This is true because by Hoeffding's Inequality, we have
\begin{align*}
    \Pr{N < 400 \log T} ~<&~ \Pr{\left|\frac{N}{N'} - P_i\right| > P_i - \frac{400\log T}{N'}} \\
~\leq&~ \Pr{\left|\frac{N}{N'} - P_i\right| > \frac{1}{4} \epsprm}\\
~\leq&~ 2\exp(-2N' \cdot \frac{\epsprm^2}{16}) ~=~ 2T^{-C/8} ~<~ T^{-7}. 
\end{align*}
Here, the second inequality is from $P_i \geq \frac{1}{2} \epsprm$ and $\frac{400 \log T}{N'} = \frac{400\epsprm^2}{C}< \frac{1}{4} \epsprm$ when $C$ is a constant greater than $1600$, and the last inequality also holds when $C>1600$. 

Next, we show $|\hat F_i(\ucb^s_i) - F_i(\ucb^s_i)| \leq 0.1$ assuming we have at least $N = 400\log T$ samples. By Hoeffding's Inequality, we have
\[
\Pr{|\hat F_i(\ucb^s_i) - F_i(\ucb^s_i)| > 0.1} ~\leq~ 2\exp(-2N\cdot 0.01) ~=~ 2T^{-8} ~<~ T^{-7}.
\]

\noindent (iii)  $|\widetilde \rev(S_{i+1}) - \rev(S_{i+1})| \leq \frac{\epsprm}{\hat P_i}$. We assume $|\hat P_i - P_i| \leq \frac{1}{4} \epsprm$ and $|\hat F_i(\ucb^s_i) - F_i(\ucb^s_i)| \leq 0.1$ holds. The algorithm estimates $\widetilde \rev(S_{i+1}) $ only when $\hat P_i \geq \frac{3}{4} \epsprm$ and $\hat F_i(\ucb^s_i) \geq 0.4$ hold. In this case,  $|\hat P_i - P_i| \leq \frac{1}{4} \epsprm$ and $\hat P_i \geq \frac{3}{4} \epsprm$ give $\hat P_i \in [\frac{3}{4} P_i, \frac{3}{2} P_i]$. Besides, we have $P_i \geq \hat P_i - \frac{1}{4}\epsprm \geq \frac{1}{2}\epsprm$ and  $F_i(\ucb^s_i) \geq \hat F_i(\ucb^s_i) - 0.1 \geq 0.3$.

In the algorithm, we run $N' = C \frac{\log T}{\epsprm^2}$ rounds to estimate $\widetilde \rev(S_{i+1})$. However, a single round becomes a valid sample only when the item is unsold until buyer $i+1$. We claim that with probability $1 - T^{-7}$ we have at least $N = \frac{9\log T P^2_i}{\epsprm^2}$ samples. This is true because by Hoeffding's Inequality, we have
\begin{align*}
    \Pr{N < \frac{9\log T P^2_i}{\epsprm^2}} ~\leq&~ \Pr{\left|\frac{N}{N'} - P_{i+1}\right| > P_{i+1} - \frac{9\log T P^2_i}{N' \cdot \epsprm^2}} \\
    ~\leq&~\Pr{\left|\frac{N}{N'} - P_{i+1}\right| > 0.2P_i} \\
    ~\leq&~ 2\exp(-2N' \cdot 0.01 P_i^2) ~\leq~ 2T^{-0.01C} ~<~ T^{-7}.
\end{align*}
Here, the second inequality is from $P_{i+1} = P_i F_i(\ucb^s_i) \geq 0.3P_i$, and $\frac{9\log T P_i^2}{N' \epsprm^2} = \frac{9P_i^2}{C} < 0.1P_i$, when $C$ is a constant greater than $90$. The second to the last inequality uses the fact that $P_i \geq \frac{1}{2} \epsprm$, and the last inequality is true when $C$ is a constant greater than $800$.

Next, we show $|\widetilde \rev(S_{i+1}) - \rev(S_{i+1})| \leq \frac{\epsprm}{\hat P_i}$ assuming we have at least $N = 400\log T$ samples. By Hoeffding's Inequality, we have
\begin{align*}
    \Pr{|\widetilde \rev(S_{i+1}) - \rev(S_{i+1})| > \frac{\epsprm}{\hat P_i}} ~\leq&~ \Pr{|\widetilde \rev(S_{i+1}) - \rev(S_{i+1})| > \frac{\epsprm}{ 1.5P_i}} \\
    ~\leq&~  2\exp\left(-2N\cdot \frac{\epsprm^2}{2.25P_i^2}\right) ~=~ 2T^{-8} ~<~ T^{-7}.
\end{align*}

\noindent (iii)  $|\hat R_i(\hat p^{(j)}_i) - R_i(\hat p^{(j)}_i)| \leq \frac{1}{2}\epsprm$ for $j \in \{1, 2, 3, 4\}$. We only prove the inequality when $j = 1$ and the remaining cases are identical. $\hat R_i(\hat p^{(1)}_i)$ is estimated via $N = C\cdot \frac{\log T}{\epsprm^2}$ samples. By Hoeffding's Inequality, we have
\[
\Pr{|\hat R_i(\hat p^{(j)}_i) - R_i(\hat p^{(j)}_i)| > \frac{1}{2}\epsprm} ~\leq~ 2\exp(-2N \cdot 0.25 \epsprm^2) ~=~ 2T^{-0.5 C} ~<~ T^{-7},
\]
where the last inequality holds when $C$ is a constant greater than $16$.
\end{proof}

With the help of \Cref{clm:gr-subrcon}, we can give the formal proof of \Cref{lma:gr-subr}:

\begin{proof}[Proof of \Cref{lma:gr-subr}]
    The bound of number of tested rounds in and the constraint that $p_i \in [\lcb_i, \ucb^s_i]$ are directly guaranteed by \Cref{alg:gr-subr}. It remains to show $R_i(\hat p^*_i) - R_i(\hat p_i) \leq 4\epsprm$. We prove the inequality separately for the following cases:

\noindent \textbf{Case 1:} $\hat P_i < \frac{3}{4} \epsprm$ (Line 3-4). In this case, $\hat P_i < \frac{3}{4} \epsprm$ implies $P_i \leq \epsprm$, and $R_i(\hat p^*_i) - R_i(p) \leq P_i \leq \epsprm$ holds for any price $p$, so taking $\hat p_i = \lcb_i$ is feasible.

\noindent \textbf{Case 2:} $\hat P_i \geq \frac{3}{4} \epsprm$ while $\hat F_i(\ucb^s_i) \geq 0.4$ (Line 7-12). Consider the value of $\hat p^*_i$. If $\hat p^*_i \in [\max\{\lcb_i, \widehat \rev(S_{i+1}) + \frac{\epsprm}{\hat P_i}\}, \ucb_i]$, there must be $R_i(\hat p^*_i) - R_i(\hat p^{(1)}_i) \leq \epsprm$. This is true because \Cref{clm:GRHalfCon} guarantees $[\max\{\lcb_i, \widehat \rev(S_{i+1}) + \frac{\epsprm}{\hat P_i}\}, \ucb_i] \subseteq [\rev(S_{i+1}), 1]$ is half-concave, and \Cref{lma:Step1} guarantees that the single-buyer algorithm outputs $\hat p^{(1)}_i$ that satisfies $R_i(\hat p^*_i) - R_i(\hat p^{(1)}_i) \leq \epsprm$ with probability $1 - T^{-6}$. 

If $\hat p^*_i \notin [\max\{\lcb_i, \widehat \rev(S_{i+1}) + \frac{\epsprm}{\hat P_i}\}, \ucb_i]$, there must be $\hat p^*_i \in [ \max\{\lcb_i, \widehat \rev(S_{i+1}) - \frac{\epsprm}{\hat P_i}\}, \widehat \rev(S_{i+1}) + \frac{\epsprm}{\hat P_i}\}]$. This is true because simultaneously we have $\hat p^*_i \geq \lcb_i$ from definition of $\hat p^*_i$, $\hat p^*_i \geq \rev(S_{i+1}) \geq \widehat \rev(S_{i+1}) - \frac{\epsprm}{\hat P_i}$ from \Cref{clm:gr-hat-star-rev} and \Cref{clm:gr-subrcon}, and $\hat p^*_i \leq \widehat \rev(S_{i+1}) + \frac{\epsprm}{\hat P_i}$ from the assumption  $\hat p^*_i \notin [\max\{\lcb_i, \widehat \rev(S_{i+1}) + \frac{\epsprm}{\hat P_i}\}, \ucb_i]$. Then, the constrained Lipschitzness from \Cref{clm:GRHalfCon} gives
\[
R_i(\hat p^*_i) - R_i(\hat p^{(2)}_i) \leq P_i(\hat p^*_i - \hat p^{(2)}_i) \leq P_i\big((\widehat \rev(S_{i+1}) + \frac{\epsprm}{\hat P_i}) - (\widehat \rev(S_{i+1}) - \frac{\epsprm}{\hat P_i})\big) \leq  \frac{2P_i\epsprm}{\hat P_i} \leq 3\epsprm,
\]
where the last inequality holds because $\hat P_i \geq \frac{3}{4} \epsprm$ and $|\hat P_i - P_i|\leq \frac{1}{4}\epsprm$ imply $\frac{P_i}{\hat P_i} \leq \frac{4}{3} < 1.5$.

Therefore, at least one of the inequalities $R_i(\hat p^*_i) - R_i(\hat p^{(1)}_i) \leq 3\epsprm$ or $R_i(\hat p^*_i) - R_i(\hat p^{(2)}_i) \leq 3\epsprm$ holds. If $R_i(\hat p^*_i) - R_i(\hat p^{(1)}_i) \leq 3\epsprm$ holds, we have
\[
R_i(\hat p^*_i) - R_i(\hat p_i) ~\leq~ R_i(\hat p^*_i) - \hat R_i(\hat p_i) + \frac{\epsprm}{2}  ~\leq~ R_i(\hat p^*_i) - \hat R_i(\hat p^{(1)}_i) + \frac{\epsprm}{2} ~\leq~ R_i(\hat p^*_i) -  R_i(\hat p^{(1)}_i) + \epsprm ~\leq~ 4\epsprm.
\]
Symmetrically, if $R_i(\hat p^*_i) - R_i(\hat p^{(2)}_i) \leq 3\epsprm$, the same argument still gives $R_i(\hat p^*_i) - R_i(\hat p_i) \leq 4\epsprm$. Therefore, $R_i(\hat p^*_i) - R_i(\hat p_i) \leq 4\epsprm$ holds in Case 2.

\noindent \textbf{Cases 3:} $\hat P_i \geq \frac{3}{4} \epsprm$ while $\hat F_i(\ucb^s_i) < 0.4$ (Line 14-18). This is sufficient to show $F_i(\ucb^s_i) < 0.5$. 

In this case, the algorithm prepares two candidate prices. If $\rev(S_{i+1}) \leq \ucb^s_i$, \Cref{lma:gr-case2} guarantees that with probability $1 - T^{-6}$ we get $\hat p^{(3)}_i$ such that $R_i(\hat p^*_i) - R_i(\hat p^{(3)}_i) \leq \epsprm$. Otherwise, there must be $\rev(S_{i+1}) > \ucb^s_i$. In this case, \Cref{clm:gr-hat-star-rev} guarantees $\ucb^s_i < \rev(S_{i+1}) \leq \hat p^*_i$, and the constrained Lipschitzness of $R_i(p)$ guarantees 
\[
R_i(\hat p^*_i) - R_i(\hat p^{(4)}_i) ~\leq~ \hat p^*_i - \hat p^{(4)}_i ~=~ \hat p^*_i - \ucb^s_i ~\leq~ \se ~<~ \epsprm.
\]
Therefore, at least one of the inequalities $R_i(\hat p^*_i) - R_i(\hat p^{(3)}_i) \leq \epsprm$ or $R_i(\hat p^*_i) - R_i(\hat p^{(4)}_i) \leq \epsprm$ holds. An argument similar to Case 1 is sufficient to show $R_i(\hat p^*_i) - R_i(\hat p_i) \leq 4\epsprm$.

To conclude, $R_i(\hat p^*_i) - R_i(\hat p_i) \leq 4\epsprm$ holds in all cases.
The final step is to check the success probability of our proof. In this proof, we require the correctness of \Cref{clm:gr-subrcon}, \Cref{alg:sr-esp} in Case 2, and $\alg$ in Case 3. Each of them fails with probability $T^{-6}$. By the union bound, the whole proof succeeds with probability $ 1 - 3T^{-6} > 1 - T^{-5}$.
\end{proof}

\subsection{\Cref{lma:gr-case2} via Generalized Half-Concavity}
\label{sec:gr334}

In this subsection, we prove the missing \Cref{lma:gr-case2}, which is restated for convenience.

\grcase*

Our algorithm begins from the following observation: Recall that $R_i(p) = \pre(S_i) + P_i(p \cdot (1 - F_i(p)) + \rev(S_{i+1}) \cdot F_i(p))$. Then,
\[
R'_i(p) ~=~ P_i\big(1 - F_i(p) + (\rev(S_{i+1}) - p) \cdot f_i(p)\big ).
\]
When $F_i(\ucb^{s}_i) \leq 0.5$, we have $R'_i(p) \geq P_i(1 - F_i(p)) \geq  P_i(1 - F_i(\ucb^s_i)) \geq\frac{1}{2}P_i$ for $p \in [0, \rev(S_{i+1})]$, while  $R'_i(p) \leq P_i(1 - F_i(p)) \leq  P_i$ for $p \in [\rev(S_{i+1}), 1]$. This observation shows $R_i(p)$ is growing sufficiently fast in $[0, \rev(S_{i+1})]$ as compared to $[\rev(S_{i+1}), 1]$, leading to some kinds of ``concavity" property in $[0, \rev(S_{i+1})]$. Inspired by this observation, we give the following definition of the generalized half-concavity and show how it relates to $R_i(p)$:

\begin{restatable}{Definition}{GRLGood}[Half $\lambda$-Concavity]
Given function $R(p)$ and interval $[\lcb, \ucb]$. For $\lambda \geq 1$, we say $R(p)$ is \emph{half $\lambda$-concave} on $[\lcb, \ucb]$ if the following conditions hold:
\begin{itemize}
    \item $R(p)$ is single-peaked in $[\lcb, \ucb]$.
    \item Let $p^* = \argmax_{[\lcb, \ucb]} R(p)$. For $p \in [\lcb, p^*]$, $R(p^*) - R(p) \leq p^* - p$.
    \item For $a, b \in [\lcb, p^*)$ satisfying $a < b$, we have 
    \[
    \frac{R(b)-R(a)}{b - a} \geq \frac{1}{\lambda} \cdot \frac{R(p^*) - R(b)}{p^* - b}.
    \]
\end{itemize}
\end{restatable}

The intuition of defining half $\lambda$-concavity is to capture the case when $R(p)$ is concave, or the case when $R(p)$ has slope between $[k, \lambda k]$, or a mixture of these two. The following lemma suggests $R_i(p)$ is captured by this half $\lambda$-concavity when $F_i(\ucb^s_i)$ is not too large.

\begin{Lemma}
\label{lma:gr-2good}
If $F_i(\ucb^{s}_i) < \frac{1}{2}$ and $\ucb^s_i > \rev(S_{i+1})$,  $R_i(p)$ is half 2-concave in $[\lcb_i, \ucb_i]$.
\end{Lemma}

\begin{proof}
    We prove three properties one by one.
    
    \noindent (i) \emph{Single-peakedness:} This is guaranteed by \Cref{clm:GRHalfCon}.

    \noindent (ii) \emph{Lipschitzness with respect to $p^*$:} \Cref{clm:gr-hat-star-rev} gives $\hat p^*_i \geq \rev(S_{i+1})$. Then, the constrained Lipschitzness from \Cref{clm:GRHalfCon} is sufficient to give the inequality $R_i(\hat p^*_i) - R_i(p) \leq \hat p^*_i - p$ for all $p \in [\lcb_i, \hat p^*_i]$.

    \noindent (iii) \emph{Weakened Concavity:} The proof is based on the following three facts:
    \begin{itemize}
        \item Fact 1: For $x< y \leq \rev(S_{i+1})$, we have
        \[
        \frac{R_i(y) - R_i(x)}{y - x} ~=~ \frac{1}{y-x} \cdot \int_x^y R'_i(p)dp ~\geq ~ \frac{1}{y-x} \cdot \int_x^y P_i(1 - F_i(\ucb^s_i))dp ~\geq ~ \frac{1}{2}P_i.
        \]
        \item Fact 2: For $x \leq \rev(S_{i+1})$ and $y > \rev(S_{i+1})$, the constrained Lipschitzness in \Cref{clm:GRHalfCon} gives
        \[
        \frac{R_i(y) - R_i(x)}{y - x} ~\leq~ P_i.
        \]
        \item Fact 3: For $x, y \in [\lcb_i, \ucb_i]$ that satisfies $\rev(S_{i+1}) \leq x < y < \hat p^*_i$, we claim that
        \[
        \frac{R_i(y) - R_i(x)}{y - x} ~\geq~ \frac{R_i(\hat p^*_i) - R_i(y)}{\hat p^*_i - y}.
        \]
        If a feasible pair of $(x, y)$ that satisfies $x, y \in [\lcb_i, \ucb_i]$ and $\rev(S_{i+1}) \leq x < y < \hat p^*_i$ exists, there must be $\hat p^*_i > \lcb_i$, which implies that $R_i(p)$ is not decreasing in $[\lcb_i, \ucb_i]$, and there must be $\rev(S_{i+1}) \leq \hat p^*_i \leq \arg \max_{p \in [0, 1]} R_i(p)$. Then, the half-concavity of $R_i(p)$ in $[\rev(S_{i+1}), 1]$ guarantees that $R_i(p)$ is concave in $[\rev(S_{i+1}), \hat p^*_i]$. This gives the inequality in Fact 3.
    \end{itemize}
    Now, we divide the proof into the following three cases according to the value of $a$ and $b$:
    \begin{itemize}
        \item Case 1: $a < b \leq \rev(S_{i+1})$. In this case, we have 
        \[
        \frac{R_i(b) - R_i(a)}{b - a} ~\geq~ \frac{1}{2}P_i ~\geq~  \frac{1}{2} \cdot \frac{R_i(\hat p^*_i) - R_i(b)}{\hat P^*_i - b},
        \]
        where the first inequality is from Fact 1 and the second inequality is from Fact 2.
        \item Case 2: $a < \rev(S_{i+1})$, while $b > \rev(S_{i+1})$. In this case, Fact 1 and 2 gives
        \[
        \frac{R_i(\rev(S_{i+1})) - R_i(a)}{\rev(S_{i+1}) - a} ~\geq~ \frac{1}{2}P_i ~\geq~ \frac{1}{2} \cdot \frac{R_i(\hat p^*_i) - R_i(b)}{\hat p^*_i - b}.
        \]
        On the other hand, Fact 3 gives
        \[
        \frac{R_i(b) - R_i(\rev(S_{i+1}))}{b - \rev(S_{i+1})}~\geq~ \frac{R_i(\hat p^*_i) - R_i(b)}{\hat p^*_i - b} ~\geq~ \frac{1}{2} \cdot \frac{R_i(\hat p^*_i) - R_i(b)}{\hat p^*_i - b}.
        \]
        Therefore,
        \[
        \frac{R_i(b) - R_i(a)}{b - a} \geq \min\left\{\frac{R_i(\rev(S_{i+1})) - R_i(a)}{\rev(S_{i+1}) - a},  \frac{R_i(b) - R_i(\rev(S_{i+1}))}{b - \rev(S_{i+1})}\right\} \geq \frac{1}{2} \cdot \frac{R_i(\hat p^*_i) - R_i(b)}{\hat p^*_i - b}.
        \]
        \item Case 3: $\rev(S_{i+1}) \leq a < b < \hat p^*_i$. In this case, Fact 3 gives
        \[
        \frac{R_i(b) - R_i(a)}{b - a}~\geq~ \frac{R_i(\hat p^*_i) - R_i(b)}{\hat p^*_i - b} ~\geq~ \frac{1}{2} \cdot \frac{R_i(\hat p^*_i) - R_i(b)}{\hat p^*_i - b}.   \qedhere
        \]
    \end{itemize}
\end{proof}

\Cref{lma:gr-2good} shows that $R_i(p)$ is generalized half-concave. Then, the following lemma suggests that the function is learnable with the help of half $\lambda$-concavity:

\begin{restatable}{Lemma}{GRGenR}
\label{lma:gr-genalg}
    If $R_i(p)$ is half $\lambda$-concave, then there exists an algorithm  that takes parameter $\lambda$, interval $[\lcb, \ucb]$, error $\epsilon \geq \frac{1}{\sqrt{T}}$ as input, tests $O(\frac{\lambda^2 \log^2 T}{\epsilon^2})$ rounds with price vectors $(\ucb^{s}_1, \cdots, \ucb^{s}_{i - 1}, p_i \in [\lcb, \ucb - \se], \hat p_{i+1}, \cdots, \hat p_n)$, and outputs $\hat p_i$ that satisfies with probability $1 - T^{-6}$ that $\max_{p \in [\lcb, \ucb]} R_i(p) - R_i(\hat p_i) \leq \epsilon$.
\end{restatable}

\Cref{lma:gr-genalg} is a generalization of \Cref{alg:sr-esp} for a single buyer. Observe that \Cref{alg:sr-esp}  only needs a weaker inequality $\frac{R(b) - R(a)}{b - a} \geq \frac{R(p^*) - R(b)}{p^* - b}$. Then, generalizing \Cref{alg:sr-esp} by taking $\lambda$ into consideration gives the desired algorithm for \Cref{lma:gr-genalg}. 

Observe that  combining \Cref{lma:gr-genalg} with \Cref{lma:gr-2good} directly proves \Cref{lma:gr-case2}. It only remains to prove \Cref{lma:gr-genalg}.



\begin{algorithm}[tbh]
\caption{Generalized Algorithm for Half $\lambda$-Concave Functions}
\label{alg:GenAlg}
\KwIn{Hidden Function $R_i(p)$, Interval $[\lcb, \ucb]$, Concavity Parameter $\lambda$, Error Parameter $\epsilon$}
Fix $p_j = \ucb^s_j$ for $j < i$ and $p_j = \hat p_j$ for $j > i$. \\
Let $\epsprm = \frac{\epsilon}{100\lambda}$ be the scaled error parameter, $\lcb_1 \gets \lcb, \ucb_1 \gets  \ucb$, $j = 1$. \\
\While{$\ucb_j - \lcb_j > \epsprm$}
{
    Let $a_j \gets \frac{2\lcb_j + \ucb_j}{3}, b_j \gets \frac{\lcb_j + 2\ucb_j}{3}$. \\
    Test $C \cdot \frac{\log T}{\epsprm^2}$ rounds with price $p_i=a_j$. Let $\hat R_i(a_j)$ be the average reward. \\
    Test $C \cdot \frac{\log T}{\epsprm^2}$ rounds with price $p_i=b_j$. Let $\hat R_i(b_j)$ be the average reward. \\
    \eIf{$\hat R_i(a_j) < \hat R_i(b_j) - 2\epsprm$}{Let $\lcb_{j+1} \gets a_j, \ucb_{j+1} \gets \ucb_j$}{Let $\lcb_{j+1} \gets \lcb_j, \ucb_{j+1} \gets b_j$}
    $j \gets j+1$
}
Test $C \cdot \frac{\log T}{\epsprm^2}$ rounds with price $p_i=\lcb_j$. Let $\hat R_i(\lcb_j)$ be the average reward. \\
Let $\hat p_i = \arg \max_{p \in P} \hat R_i(p)$, where $P = \{p:p \text{ is tested}\}$. \\
\KwOut{$\hat p_i$. }
\end{algorithm}

\begin{proof}[Proof of \Cref{lma:gr-genalg}]
\Cref{alg:GenAlg} is the desired algorithm for \Cref{lma:gr-genalg}. The algorithm is almost the same as \Cref{alg:sr-esp}. The only difference is that extra care is needed for the extra parameter $\lambda$.

 We first show that \Cref{alg:GenAlg} is feasible: it tests $O(\frac{\lambda^2\log^2 T}{\epsilon^2})$ rounds with $p_i \in [\lcb, \ucb - \se]$.
 
 The main body of \Cref{alg:GenAlg} is a while loop that maintains an interval to be tested. Assume the while loop stops when $j = k + 1$. Then, for $j \in [k]$, the length of $[\lcb_{j+1}, \ucb_{j+1}]$ is two-thirds of $[\lcb_j, \ucb_j]$, and we have $\ucb_k - \lcb_k > \epsprm$. Therefore, $k = O(\log \frac{1}{\epsprm}) = O(\log T)$. For interval $[\lcb_j, \ucb_j]$, two prices $a_j, b_j$ are tested for $O(\frac{\log T}{\epsprm^2})$ times, so in total \Cref{alg:GenAlg} tests $O(\frac{\lambda^2\log^2 T}{\epsilon^2})$ rounds. To show every tested price is in $[\lcb, \ucb- \se]$, let $p_i$ be a price tested in \Cref{alg:GenAlg}. $p_i \in [\lcb, \ucb]$ directly follows from the algorithm. For the inequality $p_i \leq \ucb - \se$, observe that the largest tested price must be $p_i = b_j$ for some $j \in [k]$, and $b_j \leq \ucb - \se$ is guaranteed by the fact that $\ucb_j - b_j \geq \frac{1}{3} \epsprm \gg \se$.

 It only remains to show $\max_{p \in [\lcb, \ucb]} - R_i(\hat p_i) \leq \epsilon$. Define $p^*_{\lcb, \ucb} := \arg \max_{p \in [\lcb, \ucb]} R_i(p)$ to be the optimal price in $[\lcb, \ucb]$. We will show that  $R_i(p^*_{\lcb, \ucb}) - R_i(\hat p_i) \leq 6\lambda \epsprm < \epsilon$ holds with probability $1 - T^{-6}$. We start the proof by introducing the following claim:
    \begin{Claim}
    \label{clm:gr-gen-con}
    In \Cref{alg:GenAlg}, $|\hat R_i(p) - R_i(p)| \leq \epsprm$ simultaneously holds for every tested price $p$ with probability at least $1 - T^{-6}$.
    \end{Claim}
    
    \begin{proof}
        For a single tested price $p$, $\hat R_i(p)$ is estimated by calculating the average of $N = C \cdot \frac{\log T}{\epsprm^2}$ samples. By Hoeffding's Inequality, 
        \[
        \Pr{|\hat R_i(p) - R_i(p)| > \epsprm} ~\leq~ 2\exp\left(-2 N \epsprm^2\right) ~=~ 2T^{-2C} ~<~ T^{-7}.
        \]
        The last inequality holds when $C$ is a constant greater than $4$. Then, we have $|\hat R_i(p) - R_i(p)| \leq \epsprm$ holds with probability $1 - T^{-7}$ for a single tested price $p$.
        
        Notice that \Cref{alg:GenAlg} can't test more than $T$ prices. By the union bound, $|\hat R_i(p) - R_i(p)| \leq \epsprm$ simultaneously holds for all tested prices with probability $1 - T^{-6}$.
    \end{proof}

    Now, we prove $R_i(p^*_{\lcb, \ucb}) - R_i(\hat p_i) \leq 6\lambda \epsprm$ assuming \Cref{clm:gr-gen-con} holds. We demonstrate the proof by discussing the following two cases. The first case is that there exists $j \in [k]$, such that $p^*_{\lcb, \ucb}$ falls in $[\lcb_j, \ucb_j]$ but not in $[\lcb_{j+1}, \ucb_{j+1}]$. In this case, there must be $\hat R_i(a_j) \geq \hat R_i(b_j) - 2\epsprm$ and $[b_j, \ucb_j]$ is dropped. If not, we have $\hat R_i(a_j) \geq \hat R_i(b_j) - 2\epsprm$ and $[\lcb_j, a_j]$ is dropped. However, the inequality $\hat R_i(a_j) \geq \hat R_i(b_j) - 2\epsprm$ together with \Cref{clm:gr-gen-con} gives $R(a_j) < R(b_j)$. Then, the single-peakedness of $R_i(p)$ gives $p^*_{\lcb, \ucb} \geq a_j$, which is in contrast to the assumption that $p^*_{\lcb, \ucb} \notin [\lcb_{j+1}, \ucb_{j+1}]$.  Therefore, we have $\hat R_i(a_j) \geq \hat R_i(b_j) - 2\epsprm$, and the assumption $p^*_{\lcb, \ucb} \notin [\lcb_{j+1}, \ucb_{j+1}]$ implies $p^*_{\lcb, \ucb} \in [b_j, \ucb_j]$. Then, the half $\lambda$-concavity of $R_i(p)$ in $[a_j, p^*_{\lcb, \ucb}]$ guarantees 
    \[R_i(p^*_{\lcb, \ucb}) - R_i(b_j) ~\leq~ \frac{R_i(b_j) - R_i(a_j)}{b_j - a_j} \cdot \lambda (\hat p^*_{\lcb, \ucb} - b_j) ~\leq~ \lambda\big((\hat R_i(b_j) + \epsprm) - (\hat R_i(a_j) - \epsprm)\big) ~\leq~ 4\lambda\epsprm.\]

    The second case is that the desired $j$ in the first case doesn't exist. In this case, the only possibility is that $p^*_{\lcb, \ucb} \in [\lcb_{k+1}, \ucb_{k+1}]$. Then, the Lipschitzness of function $R_i(p)$ w.r.t. $p^*_{\lcb, \ucb}$ guarantees that $R_i(p^*_{\lcb, \ucb}) - R_i(\lcb_k) \leq  \epsprm < 4\lambda \epsprm$.

    In both cases, we find a tested price $p'$ satisfying $R_i(p') \geq R_i(p^*_{\lcb, \ucb}) - 4\lambda\epsprm$. Then,
\[
R_i(\hat p) ~\geq~ \hat R_i(\hat p) - \epsprm ~\geq \hat R_i(p') - \epsprm ~\geq ~ R_i(p') - 2\epsprm ~\geq~  R_i(p^*_{\lcb, \ucb}) - 2\epsprm - 4\lambda \epsprm ~\geq~ 6\lambda \epsprm.     \qedhere
\]
\end{proof}

\section{$n$ Buyers with General Distributions}\label{sec:GG}

In this section, we provide the $\OTild(\poly(n)T^{\frac{2}{3}})$ regret algorithm for BSPP problem with general distributions.
The setting in this section is the same as \Cref{sec:GeneralRegular}, but with no extra regularity assumption. Our result generalizes the $\OTild(T^{2/3})$ single-buyer result studied in \cite{DBLP:conf/focs/KleinbergL03}. 

\begin{restatable}{Theorem}{GGMainThm}
    \label{thm:ggmain}
    For BSPP problem with $n$ buyers, there is an algorithm with $O(n^{5/3}T^{2/3} \log T)$ regret.
\end{restatable}

\paragraph{Algorithm Overview.} Comparing to the regular buyers setting, one extra challenge for general buyers setting is that the revenue function in this case is no longer single-peaked, leading to the failure of the idea of confidence intervals. Therefore, the first step of our algorithm is to discretize the value space: it can be shown that rounding buyers' valuations to multiples of $\epsilon$ only brings an extra $\epsilon$ loss. The idea of discretization was first used in \cite{DBLP:conf/focs/KleinbergL03} for single buyer setting. We will show that the same idea also works for sequential buyers. In this section, we take the discretization parameter $\epsilon = n^{5/3} T^{-1/3}$, and the accumulated regret from the discretization step would be $T \cdot O(\epsilon) = O(n^{5/3} T^{2/3})$.

After discretizing the value space, it suffices to solve BSPP for general distributions with discrete support. The main idea of our algorithm is almost identical to the algorithm for regular buyers - we design a subroutine that shrinks the possible candidates of $p^*_i$ while keeping a low regret inside the new set of candidates. Calling this subroutine $O(\log T)$ times with a halving error parameter gives the desired algorithm. Comparing to the regular distributions setting, the idea of confidence intervals does not work for discrete distributions, because the candidates may be discontinuous. To fix this issue, we use $n$ price sets to maintain possible candidate prices. 

The detailed steps of our algorithm also follows the ideas for regular distributions: we first find a group of approximately prices $\hat p_1, \cdots, \hat p_n$, and then use these prices as a benchmark to update the candidate prices sets. Since the single-peakedness no longer holds for general distributions, instead of doing binary search, we simply enumerate all remaining prices to update the candidate sets.

\subsection{BSPP is Discretizable}

We start our proof by first showing BSPP is discretizable. Specifically, we have the following lemma:

\begin{restatable}{Lemma}{GGDisc} \label{lma:GGDisc}
Consider a sequential posted pricing problem. Let $k$ be an integer and $\epsilon = \frac{1}{k}$ be the discretization error. If we restrict the space of the prices to be $P = \{(p_1, \cdots, p_n): p_i = j\cdot \epsilon, j \in \mathbb{Z}^+\}$, then 
\[
\max_{(p_1, \cdots, p_n) \in [0, 1]^n} R(p_1, \cdots, p_n) ~-~ \max_{(p_1, \cdots, p_n) \in P} R(p_1, \cdots, p_n) ~\leq~ \epsilon.
\]
\end{restatable}

\begin{proof}[Proof of \Cref{lma:GGDisc}]
    Let $p^*_1, \cdots, p^*_n$ be the optimal prices for $R(p_1, \cdots, p_n)$, and let $p'_i = \lfloor \frac{p^*_i}{\epsilon}\rfloor \cdot \epsilon$. To prove \Cref{lma:GGDisc}, we only need to show
    \[
    R(p^*_1, \cdots, p^*_n) ~-~ R(p'_1, \cdots, p'_n) ~\leq~ \epsilon.
    \]
    Let $E_i$ be the event that an item is sold to buyer $i$ with price vector $(p'_1, \cdots, p'_n)$, but not sold to buyer $i$ with price vector $(p^*_1, \cdots, p^*_n)$. Since $p'_i \leq p^*_i$ holds for all $i$, this is the only way to incur a loss when running $(p'_1, \cdots, p'_n)$. Besides, notice that $\sum_{i \in [n]} \Pr{E_i} \leq 1$, because at most one event $E_i$ can happen. Therefore,  it's sufficient to show that whenever $E_i$ happens, the expected loss is bounded by $\epsilon$.

    Let $\opt(S_{i})$ be the expected revenue from the buyers $i, i+1, \cdots, n$ when playing prices $(p^*_1, \cdots, p^*_n)$. When $E_i$ happens, we get expected revenue $\opt(S_{i+1})$ when playing $(p^*_1, \cdots, p^*_n)$ and get $p'_i$ when playing $(p'_1, \cdots, p'_n)$. Notice that we have
\[
\opt(S_{i+1}) - p'_i ~\leq~ p^*_i - p'_i ~\leq~ \epsilon,
\]
where the first inequality is from the fact that $p^*_i \geq \opt(S_{i+1})$,  because 
\[
R(1, \cdots, 1, p, p^*_{i+1}, \cdots, p^*_n) ~<~ \opt(S_{i+1}) ~=~ R(1, \cdots, 1, \opt(S_{i+1}), p^*_{i+1}, \cdots, p^*_n)
\] 
holds for any $p < \opt(S_{i+1})$, i.e., $p$ can't be the maximizer of $R(1, \cdots, 1, p, p^*_{i+1}, \cdots, p^*_n)$ when $p < \opt(S_{i+1})$. Therefore, when $E_i$ happens, the expected loss is bounded by $\epsilon$.
\end{proof}

\subsection{Algorithm for Discrete Buyers}

Now we give our algorithm for buyers with discrete support. Our main goal is to prove the following \Cref{thm:ggdiscmain}:

\begin{restatable}{Theorem}{ggdiscmain}
\label{thm:ggdiscmain}
    For BSPP problem with $n$ buyers having valuations in $V = \{v_1,v_2,\cdots, v_k\} \subseteq [0, 1]$, there exists an $O(n^{2.5}\sqrt{k T} \log T + n^5 k \log^2 T)$ regret algorithm.
\end{restatable}

With the help of \Cref{thm:ggdiscmain}, combining \Cref{lma:GGDisc} with \Cref{thm:ggdiscmain} proves the main\Cref{thm:ggmain}. Therefore, proving \Cref{thm:ggdiscmain} is sufficient to give the desired $T^{2/3}$ regret bound for general buyers.

\begin{proof}[Proof of \Cref{thm:ggmain}]
    We take $k = n^{-5/3}T^{1/3}$ and uniformly discretize $[0, 1]$ into $k$ prices. \Cref{lma:GGDisc} suggests that this leads to an extra $\frac{1}{k} = n^{5/3}T^{-1/3}$ error, so the total regret from the discretization is $T \cdot \frac{1}{k} = O(n^{5/3} T^{2/3})$. On the other hand, the total regret from the algorithm given by \Cref{thm:ggdiscmain} is $O(n^{2.5}\sqrt{k T} \log T + n^5 k \log^2 T) = O(n^{5/3}T^{2/3} \log T + n^{10/3} T^{1/3} \log^2 T)$. In this paper, we further assume that $n^{5/3} \log T = o(T^{1/3})$, i.e., $n^5 \log^3 T = o(T)$, otherwise an $\OTild(\poly(n))$ algorithm is trivial. Then, the total regret is $O(n^{5/3}T^{2/3} \log T)$.
\end{proof}

\subsubsection{Notation}

Our proof of \Cref{thm:ggdiscmain} starts from introducing the notations. We use the same set of the notations as the regular distributions algorithm, which are restated for convenience:
\begin{itemize}
    \item $\epsprm$: We define $\epsprm = \frac{\epsilon}{100n^2}$ to be the scaled error parameter. $\epsprm$ also represents the unit length of those confidence intervals given by concentration inequalities.
     \item $R(p_1, \cdots, p_n)$ and $p^*_1, \cdots, p^*_n$: $R(p_1, \cdots, p_n)$ represents the expected revenue with prices $(p_1, \cdots, p_n)$, and $(p^*_1, \cdots, p^*_n)$ represents the optimal prices that maximizes $R(p_1, \cdots, p_n)$.
     \item $(\hat p_1, \cdots, \hat p_n)$: The approximately optimal price vector we find in first-step algorithm. 
     \item $S_i$: The suffix buyers $i, i+1, \cdots, n$. When finding $\hat p_{i-1}$, the performance of buyers $i, i+1, \cdots, n$ can be seen as a whole. $S_i$ represents the whole of this suffix starting from buyer $i$. 
     \item $P_i$: We define $ P_i ~:=~ \prod_{j < i} F_j(\ucb_j),$ to be the maximum probability of testing buyer $i$.
     \item $\pre(S_i)$: We define $\pre(S_i) := R(\ucb_1, \cdots, \ucb_{i-1}, 1, \cdots, 1)$ to be the expected revenue from the buyers before $S_i$ with prices $\ucb_1, \cdots, \ucb_{i-1}$.
     \item $\opt(S_i)$: We define
     $\opt(S_i):=R(1, \cdots, 1, p^*_i, \cdots, p^*_n)$
     to be the optimal revenue from $S_i$.
     \item $\rev(S_i)$: We define
     $\rev(S_i):=R(1, \cdots,1,\hat p_i, \cdots, \hat p_n)$
     to be the near-optimal revenue from $S_i$ with near-optimal prices $\hat p_i, \cdots, \hat p_n$.
     \item $\loss(S_i)$: We define $\loss(S_i) := \opt(S_i) - \rev(S_i)$, i.e., $\loss(S_i)$ represents the loss from $S_i$ when playing prices $\hat p_i, \cdots, \hat p_n$.
     \item $R_i(p)$: We define $R_i(p)~:=~R\big(\ucb_1, \cdot, \ucb_{i-1}, p_i = p, \hat p_{i+1}, \cdots, \hat p_n\big).$ It should be noted that $R_i(p)$ can be also written as
$R_i(p) ~=~ \pre(S_i) + P_i\big(p \cdot (1 - F_i(p)) + \rev(S_{i+1}) \cdot F_i(p)\big).$
     \item $R^*_i(p)$:  We define $R^*_i(p) ~:=~ R(\ucb^s_1, \cdots, \ucb^s_{i-1}, p_i = p, p^*_{i+1}, \cdots, p^*_n)$. It should be noted that $R^*_i(p)$ can be also written as 
     \[
     \textstyle R^*_i(p)~=~ \pre(S_i) + P_i(p\cdot (1 - F_i(p)) + \opt(S_{i+1}) \cdot F_i(p)) ~=~ R_i(p) + P_i\cdot F_i(p) \cdot\loss(S_{i+1}).
     \]
\end{itemize}
We also introduce the following notations for the discretized setting:
\begin{itemize}
    \item $V = \{v_1, \cdots, v_k\}$: The discrete values of buyers. Every buyer's value must fall in $V$.
    \item $\cp_i$: The candidate prices set of buyer $i$. Initially $\cp_i$ is set to be $V$.
     \item $\ucb_i$: we define $\ucb_i = \max \cp_i$ to be the maximum feasible price for buyer $i$.
\end{itemize}

\subsubsection{High-Level Structure: Shrinking Candidate Prices Sets}

The high-level structure of our algorithm is to gradually shrink the sets of candidate prices. Specifically, we first give the following main sub-routine:

\begin{Lemma}
\label{lma:mainsub-gg}
    Given a BSPP instance with $R(p_1, \cdots, p_n)$ being the expected revenue function and $(p^*_1, \cdots, p^*_n)$ being the optimal prices. Assume the valuation of each single buyer falls in the set $V = \{v_1, \cdots, v_k\}$. Given error parameter $\epsilon \geq 1/\sqrt{T}$ and candidate sets $\cp_1 \subseteq V, \cdots, \cp_n \subseteq V$ satisfying $p^*_i \in \cp_i$ for all $i \in [n]$, there exists an algorithm that tests $O(\frac{n^5 k \log T}{\epsilon^2})$ rounds, where each tested price vector $(p_1, \cdots, p_n)$ satisfies $p_i \in \cp_i$ for all $i \in [n]$, and outputs with probability $1 - T^{-3}$ new candidate sets $\cp'_1 \subseteq \cp_1, \cdots, \cp'_n \subseteq \cp_n$ satisfying
    \begin{itemize}
        \item For $i \in [n]$, $p^*_i \in \cp'_i$.
        \item For every price vector $(p_1, \cdots, p_n)$ that satisfies $p_i \in \cp'_i$ for all $i \in [n]$, we have
        \[
        R(p^*_1, \cdots, p^*_n) - R(p_1, \cdots, p_n) \leq \epsilon.
        \]
    \end{itemize}
\end{Lemma}

We defer the detailed algorithm and proofs for \Cref{lma:mainsub-gg} to \Cref{sec:gg-main-sub} and first give the proof for \Cref{thm:ggdiscmain}. The main idea is to use \Cref{lma:mainsub-gg} for $O(\log T)$ times and gradually halve the one-round error when playing with prices inside the candidate sets.

\begin{algorithm}[tbh]
\caption{$O(n^{2.5}\sqrt{kT} \log T)$ Regret Algorithm}
\label{alg:gg-gen}
\KwIn{Hidden Revenue Function $R(p_1, \cdots, p_n)$, time horizon $T$}
Let $\epsilon_1 \gets 1, \cp^{(1)}_j \gets V = \{v_1, \cdots, v_k\}$ for all $j \in [n]$, and $i \gets 1$. \\
\While{$\epsilon_i > \frac{n^{2.5} \sqrt{k}\log T}{\sqrt{T}}$}
{
    Run the algorithm described in \Cref{lma:mainsub-gg} with $\epsilon_i, \cp^{(i)}_1, \cdots, \cp^{(i)}_n$ as input, and get $\cp'_1, \cdots, \cp'_n$. Let $\cp^{(i+1)}_j \gets \cp'_j$ for all $j \in [n]$. \\
    Let $\epsilon_{i+1} \gets \frac{1}{2}\epsilon_i$ and $i \gets i+1$. \\
}
Finish remaining rounds with any $(p_1, \cdots, p_n)$ satisfying $ p_j \in \cp_j$ for all $j \in [n]$.
\end{algorithm}

\begin{proof}[Proof of \Cref{thm:ggdiscmain}]
   We claim that the regret of \Cref{alg:gg-gen} is $O(n^{2.5}\sqrt{kT} \log T)$ with probability $1 - T^{-2}$. \Cref{alg:gg-gen} uses \Cref{lma:mainsub-gg} for multiple times with a halving error parameter. Assume \Cref{alg:gg-gen} ends with $i = q +1$, i.e., the while loop runs $q$ times. Since we obtain $\epsilon_{i+1} = \epsilon_i /2$ and $\epsilon_q > \frac{n^{2.5} \sqrt{k}\log T}{\sqrt{T}}$, there must be $q = O(\log T)$.
    
    For $i \in [q]$, let $\alg_i$ represent the corresponding algorithm we call when using \Cref{lma:mainsub-gg} with $\epsilon_i, \{\cp^{(i)}_j\}$ as the input. To use \Cref{lma:mainsub-gg} we need to verify that $p^*_j \in \cp^{(i)}_j$ for all $i \in [q], j \in [n]$. The condition $p^*_j \in \cp^{(1)}_j = V$ holds because every buyer only has valuation in $V$. For $i =2, 3, \cdots, k$, the condition $p^*_j \in \cp^{(i)}_j$ is guaranteed by \Cref{lma:mainsub-gg} when calling $\alg_{i-1}$. The failing probability of \Cref{lma:mainsub-gg} is $T^{-3}$. By the union bound, with probability $1 - T^{-2}$, \Cref{lma:mainsub-gg} always holds.

    Now we give the regret of \Cref{alg:gg-gen}. \Cref{lma:mainsub-gg}  guarantees that for $i \geq 2$, every price vector we test in $\alg_i$ has regret bounded by $\epsilon_{i-1}$. Therefore, the total regret of \Cref{alg:gg-gen} can be bounded by
    \[
    \textstyle 1\cdot O\big(\epsilon_1^{-2} n^{5} k \log^2 T\big) ~+~ \sum_{i = 2}^q \epsilon_{i-1} \cdot O\big(\epsilon^{-2}_i n^{5} k \log^2 T\big) ~+~ T \cdot \epsilon_q  ~=~ O(n^5 k\log^2 T + n^{2.5}\sqrt{kT} \log T).
    \]
    
    Finally, we also need to check that \Cref{alg:gg-gen} uses no more than $T$ rounds. If $n^5 k \log^2 T > T$, this implies that $\epsilon_1 < \frac{n^{2.5} \sqrt{k} \log T}{\sqrt{T}}$, so the while loop will not be processed. Otherwise, \Cref{lma:mainsub-gg} suggests that $\alg_i$ runs $O(\epsilon^{-2}_i n^5 k\log^2 T)$ rounds. So the total number of rounds in the while loop is
    $
    \sum_{i \in [k]} O(\epsilon^{-2}_i n^5 k\log^2 T) ~=~ O(T).
    $
    Therefore, \Cref{alg:gg-gen} is feasible.
\end{proof}

\subsubsection{Main Sub-routine: Proof of \Cref{lma:mainsub-gg}}
\label{sec:gg-main-sub}

We first present our main sub-routine to prove \Cref{lma:mainsub-gg}. Similar to the regular distributions setting, there are two major steps in the sub-routine: we first find $(\hat p_1, \cdots, \hat p_n)$ that approximates $(p^*_1, \cdots, p^*_n)$, and the second step is to use the near-optimal prices as a benchmark to update the candidate price sets. Specifically, we have the following two separate lemmas for these two steps:

\begin{Lemma}
\label{lma:step1-gg}
Given $\epsilon \geq \frac{1}{\sqrt{T}}$, candidate sets $\cp_1, \cdots, \cp_n$ satisfying $p^*_i \in \cp_i$, there exists an algorithm that tests $O(\frac{n^5 k \log T}{\epsilon^2})$ rounds with price vectors $(p_1, \cdots, p_n)$ that satisfy $p_i \in \cp_i$ for all $i \in [n]$ and outputs $(\hat p_1, \cdots, \hat p_n)$. With probability $1 - T^{-4}$, we have $R_i(p^*_i) - R_i(\hat p_i) \leq 2\epsprm$ for all $i \in [n]$, where $\epsprm := \frac{\epsilon}{100n^2}$ is the scaled error parameter.
\end{Lemma}

\begin{Lemma}
\label{lma:step2-gg}
Given $\epsilon \geq \frac{1}{\sqrt{T}}$, candidate sets $\cp_1, \cdots, \cp_n$ satisfying $p^*_i \in \cp_i$, and $(\hat p_1, \cdots, \hat p_n)$ that satisfies $R_i(p^*_i) - R_i(\hat p_i) \geq 2\epsprm$, there exists an algorithm that tests $O(\frac{n^5 k \log T}{\epsilon^2})$ rounds with price vectors $(p_1, \cdots, p_n)$ that satisfy $p_i \in \cp_i$ for all $i \in [n]$ and outputs with probability $1 - T^{-4}$ $\cp'_1 \subseteq \cp_1, \cdots, \cp'_n \subseteq \cp_n$ that satisfies:
\begin{itemize}
    \item $p^*_i \in \cp'_i$.
    \item For $p_1, \cdots, p_n$ that satisfies $p_i \in \cp'_i$ for $i \in [n]$, we have $R(p^*_1, \cdots, p^*_n) - R(p_1, \cdots, p_n) \leq \epsilon$.
\end{itemize}
\end{Lemma}

Combining \Cref{lma:step1-gg} and \Cref{lma:step2-gg} with a union bound directly proves \Cref{lma:mainsub-gg}. It only remains to prove the two lemmas. Now we prove them separately in the following subsections.

\subsubsection{Proof of \Cref{lma:step1-gg}: Finding Approximately Good Prices}

The key idea of \Cref{lma:step1-gg} is simply enumerating all prices in the candidate set and picking the best one during the test. Specifically, \Cref{alg:gg-step1} is the desired algorithm for \Cref{lma:step1-gg}.

\begin{algorithm}[tbh]
\caption{Finding Approximately Good Prices for General Distributions}
\label{alg:gg-step1}
\KwIn{Hidden function $R(p_1, \cdots, p_n)$, candidate sets $\cp_1, \cdots, \cp_n$, scaled error parameter $\epsprm$}
For $i \in [n]$, let $r_i$ be the largest price in $\cp_i$. \\
\For{$i = n \to 1$}
{
    For $j < i$, fix $p_j = r_j$; for $j > i$, fix $p_j = \hat p_j$. \\
    \For{$p \in \cp_i$}
    {
        Test $C \cdot \frac{\log T}{\epsprm^2}$ rounds with $p_i = p$. Let $\hat R_i(p)$ be the average reward.
    }
    Let $\hat p_i = \arg \max_{p \in \cp_i} \hat R_i(p)$.
}
\KwOut{$\hat p_i$}
\end{algorithm}

\begin{proof}[Proof of \Cref{lma:step1-gg}]
For the number of rounds of \Cref{alg:gg-step1}, our algorithm enumerates all prices in the candidate sets, so there are $O(nk)$ prices. Each price is tested for $O(\frac{\log T}{\epsprm^2}) = O(\frac{n^4 \log T}{\epsilon^2})$ rounds, so \Cref{alg:gg-step1} tests $O(\frac{n^5 k \log T}{\epsilon^2})$ rounds. It remains to show that $R_i(\hat p_i) \geq R_i(p^*_i) - 2\epsprm$. We start from following claim:
\begin{restatable}{Claim}{clmgga}
    \label{clm:gg-step1-con}
    In \Cref{alg:gg-step1}, $|\hat R_i(p) - R_i(p)| \leq \epsprm$ simultaneously holds for every $i \in [n]$ and $p \in \cp_i$ with probability at least $1 - T^{-4}$.
    \end{restatable}
    \begin{proof}
In \Cref{alg:gg-step1}, $\hat R_i(p)$ is estimated via $N = C \cdot \frac{\log T}{\epsprm^2}$ samples. By Hoeffding's Inequality,
\[
\Pr{|\hat R_i(p) - R_i(p)| > \epsprm} ~\leq~ 2\exp(-2N\epsprm^2) ~=~ 2T^{-2C} ~<~ T^{-5}.
\]
The last inequality holds when $C$ is a constant greater than $3$. Then, we have $|\hat R_i(p) - R_i(p)| \leq \epsprm$ holds with probability $1 - T^{-5}$ for a single tested price $p$.

\Cref{alg:gg-step1} tests $\sum_i |\cp_i| \leq nk$ prices. In this paper, we assume that $n\cdot k \leq T$, otherwise it's not reasonable for not being able to test every price in $T$ rounds. Therefore, by the union bound, we have  $|\hat R_i(p) - R_i(p)| \leq \epsprm$ holds for all tested price $p$ with probability $1 - T^{-4}$.
\end{proof}
    With \Cref{clm:gg-step1-con}, it is straight forward to show $\hat p_i$ is approximately optimal: we have
    \begin{align*}
        R_i(\hat p_i) ~\geq~ \hat R_i(\hat p_i) - \epsprm ~\geq~ \hat R_i(p^*_i) - \epsprm ~\geq~ R_i(p^*_i) - 2\epsprm.
    \end{align*}
    Here we use \Cref{clm:gg-step1-con} in the first and the third inequality, and the second inequality is from the fact that $\hat p_i$ maximizes $\hat R_i(p)$.
\end{proof}

\subsubsection{Proof of \Cref{lma:step2-gg}: Shrinking Candidate Prices Sets}

Before proving \Cref{lma:step2-gg}, we first introduce the following claim, which shows the loss of playing prices $\hat p_i, \cdots, \hat p_n$ for the suffix $S_i$ is small:

\begin{Claim}
    \label{clm:gg-suf-loss}
    Given $(\hat p_1, \cdots, \hat p_n)$ that satisfies $R_i(p^*_i) - R_i(\hat p_i) \leq 2\epsprm$ for $i \in [n]$. Recall that $\loss_i := R(1, \cdots, 1, p^*_i, \cdots, p^*_n) - R(1, \cdots, 1, \hat p_i, \cdots, \hat p_n)$. We have $P_i \cdot \loss_i \leq 2(n-i+1) \epsprm$.
\end{Claim}

\begin{proof}
    We prove \Cref{clm:gg-suf-loss} via induction. The base case $i = n$ is true because of  the assumption $P_n \cdot \loss_n = R_n(p^*_n) - R_n(\hat p_n) \leq 2\epsprm$. 

    For the induction step, we will show $P_i \cdot \loss_i \leq 2(n-i+1) \epsprm$ under the induction hypothesis $P_{i+1} \loss_{i+1} \leq 2(n-i) \epsprm$. We have
    \begin{align*}
         P_i \cdot \loss(S_i) ~=~& R(\ucb_1, \cdots, \ucb_{i-1}, p^*_{i}, p^*_{i+1} \cdots, p^*_n) - R(\ucb_1, \cdots, \ucb_{i-1}, \hat p_{i}, \hat p_{i+1} \cdots, \hat p_n) \\
            =~& R(\ucb_1, \cdots, \ucb_{i-1}, p^*_{i},p^*_{i+1} \cdots, p^*_n) -  R(\ucb_1, \cdots, \ucb_{i-1}, p^*_{i}, \hat p_{i+1} \cdots, \hat p_n)\\
            &+ R(\ucb_1, \cdots, \ucb_{i-1}, p^*_{i}, \hat p_{i+1} \cdots, \hat p_n) -R(\ucb_1, \cdots, \ucb_{i-1}, \hat p_{i}, \hat p_{i+1} \cdots, \hat p_n) \\
            \leq~& F_i(p^*_i) \cdot P_{i} \cdot \loss_{i+1} + (R_i(p^*_i) - R_i(\hat p_i)) \\
            \leq~& P_{i+1} \loss_{i+1} + 2\epsprm ~\leq~ 2(n-i+1) \epsprm.
    \end{align*}
    Here, the first inequality uses the fact that $F_i(p^*_i) \cdot P_i \leq F_i(\ucb_i) \cdot P_i = P_{i+1}$, and the second inequality uses the assumption that $P_{i+1} \loss_{i+1} \leq 2(n-i) \epsprm$, and $R_i(p^*_i) - R_i(\hat p_i) \leq 2\epsprm$.
\end{proof}

Next, similar to the regular distributions setting, we use $R^*_i(p) := R(\ucb_1, \cdots, \ucb_{i-1}, p, p^*_{i+1}, \cdots, p^*_n)$ as a bridge. With the help of \Cref{clm:gg-suf-loss}, we have the following claim:

\begin{Claim}
    \label{clm:gg-bridge}
    Given $(\hat p_1, \cdots, \hat p_n)$ that satisfies $R_i(p^*_i) - R_i(\hat p_i) \leq 2\epsprm$ for $i \in [n]$. We have $R_i(p) \leq R^*_i(p) \leq R_i(p) + 2(n-i) \epsprm$ for all $i \in [n]$.
\end{Claim}

\begin{proof}
    The first inequality $R_i(p) \leq R^*_i(p)$ holds because of the optimality of prices $p^*_{i+1}, \cdots, p^*_n$. The second inequality $R^*_i(p) \leq R_i(p) + 2(n-i+1) \epsprm$ holds from the fact that $R^*_i(p) - R_i(p) = F_i(p) \cdot P_i \cdot \loss_{i+1} \leq P_{i+1} \cdot \loss_{i+1} \leq 2(n-i)\epsprm$, where the last inequality is from \Cref{clm:gg-suf-loss}.
\end{proof}

The key idea of \Cref{lma:step2-gg} is to keep every price which is possible to become the maximzer of $R^*_i(p)$. However, directly estimating $R^*_i(p)$ is not possible. Therefore, we first estimate $R_i(p)$, and use \Cref{clm:gg-bridge} to bound the value of $R^*_i(p)$. Specifically, \Cref{alg:gg-step2} is the desired algorithm for \Cref{lma:step2-gg}.

\begin{algorithm}[tbh]
\caption{Shrinking Candidate Sets via Approximately Good Prices}
\label{alg:gg-step2}
\KwIn{Hidden function $R(p_1, \cdots, p_n)$, candidate sets $\cp_1, \cdots, \cp_n$, scaled error parameter $\epsprm$, and approximately good prices $(\hat p_1, \cdots, \hat p_n)$ from \Cref{alg:gg-step1}}
For $i \in [n]$, let $r_i$ be the largest price in $\cp_i$. \\
For $i \in [n]$, initiate $\cp'_i = \cp_i$.
\For{$i = n \to 1$}
{
    For $j < i$, fix $p_j = r_j$; for $j > i$, fix $p_j = \hat p_j$. \\
    Test $C \cdot \frac{\log T}{\epsprm^2}$ rounds with $p_i = \hat p_i$. Let $\hat R_i(\hat p_i)$ be the average reward.
    \For{$p \in \cp'_i \setminus \{\hat p_i\}$}
    {
        Test $C \cdot \frac{\log T}{\epsprm^2}$ rounds with $p_i = p$. Let $\hat R_i(p)$ be the average reward. \\
        Remove $p$ from $\cp'_i$ if $\hat R_i(p) < \hat R_i(\hat p_i) - 2(n-i+1)\epsprm - 2\epsprm$.
    }
}
\KwOut{Updated sets $\cp'_1, \cdots, \cp'_n$}
\end{algorithm}

\begin{proof}[Proof of \Cref{lma:step2-gg}]
For the number of rounds of \Cref{alg:gg-step1}, our algorithm enumerates all prices in the candidate sets, so there are $O(nk)$ prices. Each price is tested for $O(\frac{\log T}{\epsprm^2}) = O(\frac{n^4 \log T}{\epsilon^2})$ rounds, so \Cref{alg:gg-step1} tests $O(\frac{n^5 k \log T}{\epsilon^2})$ rounds. It remains to show $p^*_i \in \cp'_i$ and give the loss bound for playing prices in $\cp'_1 \cdots, \cp'_n$. We start from the following claim:
    \begin{restatable}{Claim}{clmggb}
    \label{clm:gg-step2-con}
    In \Cref{alg:gg-step2}, $|\hat R_i(p) - R_i(p)| \leq \epsprm$ simultaneously holds for every $i$ and $p$ with probability at least $1 - T^{-4}$.
    \end{restatable}
    \begin{proof}
In \Cref{alg:gg-step2}, $\hat R_i(p)$ is estimated via $N = C \cdot \frac{\log T}{\epsprm^2}$ samples. By Hoeffding's Inequality,
\[
\Pr{|\hat R_i(p) - R_i(p)| > \epsprm} ~\leq~ 2\exp(-2N\epsprm^2) ~=~ 2T^{-2C} ~<~ T^{-5}.
\]
The last inequality holds when $C$ is a constant greater than $3$. Then, we have $|\hat R_i(p) - R_i(p)| \leq \epsprm$ holds with probability $1 - T^{-5}$ for a single tested price $p$.

\Cref{alg:gg-step2} tests $\sum_i |\cp_i| \leq nk$ prices. In this paper, we assume that $n\cdot k \leq T$, otherwise it is not reasonable for not being able to test every price in $T$ rounds. Therefore, by the union bound, we have  $|\hat R_i(p) - R_i(p)| \leq \epsprm$ holds for all tested price $p$ with probability $1 - T^{-4}$.
\end{proof}
    With \Cref{clm:gg-step2-con}, it is sufficient to show that a price $p$ is not possible to become $p^*_i$ if $\hat R_i(p) < \hat R_i(\hat p_i) - 2(n-i+1)\epsprm - 2\epsprm$ holds: Notice that
    \begin{align*}
        \hat R_i(p^*_i) ~&\geq~ R_i(p^*_i) - \epsprm \\
        ~&\geq~ R^*_i(p^*_i) - 2(n-i+1) \epsprm - \epsprm \\
        ~&\geq~ R^*_i(\hat p_i) - 2(n-i+1) \epsprm - \epsprm \\
        ~&\geq~ R_i(\hat p_i) - 2(n-i+1) \epsprm - \epsprm \\
        ~&\geq~ \hat R_i(\hat p_i) - 2(n-i+1) \epsprm - 2\epsprm.
    \end{align*}
    Here, the first inequality and the last inequality is from \Cref{clm:gg-step2-con}. The second and the fourth inequality is from \Cref{clm:gg-bridge}. The third inequality is from the fact that $p^*_i$ maximizes $R^*_i(p)$. Therefore, if a price satisfies $\hat R_i(p) < \hat R_i(\hat p_i) - 2(n-i+1)\epsprm - 2\epsprm$, this price $p$ can't be a candidate of $p^*_i$.

    After \Cref{alg:gg-step2}, every price $p$ in $\cp'_i$ must satisfy $\hat R_i(p) \geq \hat R_i(\hat p_i) - 2(n - i)\epsprm - 4\epsprm$. Then, \Cref{clm:gg-step2-con} and the assumption $R_i(\hat p_i) \geq R_i(p^*_i) - 2\epsprm$ gives
    \[
    R_i(p) ~\geq~ R_i(p) - \epsprm ~\geq~ R_i(\hat p_i) - 2(n-i)\epsprm - 6\epsprm ~\geq~ R_i(p^*_i) - 2(n-i)\epsprm - 8\epsprm.
    \]
    Combining the above inequality with \Cref{clm:gg-bridge}, we have
    \[
    R^*_i(p) ~\geq~ R_i(p) ~\geq~ R_i(p^*_i) - 2(n-i)\epsprm - 8\epsprm ~\geq~ R^*_i(p^*_i) - 4(n-i)\epsprm - 10\epsprm.
    \]
    Now we show that this is sufficient to bound the loss of playing prices in sets $\cp'_1, \cdots, \cp'_n$. Specifically, we will prove the following statement via induction: Given $p_1, p_2, \cdots, p_i$, such that $p_j \in \cp'_j$ for  $j \in [i]$, we have
    \[
    R(p^*_1, \cdots, p^*_i, p^*_{i+1}, \cdots, p^*_n) - R(p_1, \cdots, p_i, p^*_{i+1}, \cdots, p^*_n) ~\leq~ 20in \cdot \epsprm.
    \]
    The base case is $i = 1$. The statement holds because we have 
    \[
    R(p^*_1, p^*_2, \cdots, p^*_n) - R(p_1, p^*_2, \cdots, p^*_n) ~=~ R_1(p^*_1) - R_1(p_1) ~\leq~ (4n + 6)\epsprm \leq 20n\epsprm.
    \]
    For the induction step, we have
    \begin{align*}
        &R(p^*_1, \cdots, p^*_{i-1}, p^*_i, p^*_{i+1}, \cdots, p^*_n) - R(p_1, \cdots, p_{i-1}, p_i, p^*_{i+1}, \cdots, p^*_n) \\
        ~=~& R(p^*_1, \cdots, p^*_{i-1}, p^*_i, p^*_{i+1}, \cdots, p^*_n) -  R(p_1, \cdots, p_{i-1}, p^*_i, p^*_{i+1}, \cdots, p^*_n)\\
        &+ R(p_1, \cdots, p_{i-1}, p^*_i, p^*_{i+1}, \cdots, p^*_n) - R(p_1, \cdots, p_{i-1}, p_i, p^*_{i+1}, \cdots, p^*_n). \\
        ~\leq~& 20(i-1)n \cdot \epsprm + \prod_{j = 1}^{i-1} F_j(p_j) \cdot \left( \frac{R^*_i(p^*_i)}{P_i} -  \frac{R^*_i(p_i)}{P_i} \right) \\
        ~\leq~& 20(i-1)n \cdot \epsprm + \left(R^*_i(p^*_i) - R^*_i(p_i) \right) \\
        ~\leq~& 20(i-1)n \cdot \epsprm + 20n \cdot \epsprm ~=~ 20in \cdot \epsprm.
    \end{align*}
    Here, the first inequality is from the induction hypothesis. The second inequality is from the fact that $p_j \leq \ucb_j$, and therefore $\prod_{j < i} F_j(p_j) \leq \prod_{j < i} F_j(\ucb_j) = P_i$. It should also be noted that the inequality can be applied only when we have $R^*_i(p^*_i) \geq R^*_i(p_i)$, which is true because $p^*_i$ maximizes $R^*_i(p)$. The last inequality is from the fact that $R^*_i(p) \geq R^*_i(p^*_i) - 4(n-i)\epsprm - 10\epsprm \geq R^*_i(p^*_i) - 20n \epsprm$.

    Finally, taking $i = n$ in the above statement together with the fact that $20n^2 \epsprm \leq \epsilon$ finishes the proof of \Cref{lma:step2-gg}.
\end{proof}

\section{Linear Regret Lower Bound for Adversarial Valuations} \label{sec:lowerBound}

We show an $\Omega(T)$ regret lower bound for learning sequential posted pricing with adversarial buyer values. \cite{DBLP:conf/focs/KleinbergL03} gave an $\OTild(T^{{2}/{3}})$ regret upper bound for bandit sequential posted pricing with adversarial values when there is only a single buyer. The following result shows that it is not possible to generalize their result to multiple buyers.

\begin{Theorem}
\label{thm:LB}
For Online Sequential Posted Pricing problem with oblivious adversarial inputs, there exists an instance with $n = 2$ buyers such that the optimal fixed-threshold strategy has total revenue $\frac{3}{4}T$ but no online algorithm can obtain total value more than $\frac{1}{2}T$.
\end{Theorem}

\begin{proof}
The proof follows the hardness example for Online Prophet Inequality problem in \cite{GKSW-arXiv22}. Here we restate the example for completeness.

Let $s$ be a binary string in $\{0,1\}^T$. Define $Bin(s)$ to be the binary decimal corresponding to $s$. For example, $Bin(1000) = (0.1000)_2 = \frac{1}{2}$ and $Bin(0101) = (0.0101)_2 = \frac{5}{16}$.

At the beginning, the adversary chooses a $T$-bits string $s=s_1s_2,\cdots, s_T$ uniformly at random, i.e., $s_i$ is set to be $0$ or $1$ independently with probability $\frac{1}{2}$. In the $i$-th round, the value of the first buyer will be $v_1 = \frac{1}{2} + \varepsilon\cdot \alpha_i$, where $\alpha_i$ is set to be $Bin(s_1s_2...s_{i-1}+0+1^{T-i+1})$, and $\varepsilon$ is an arbitrarily small constant that doesn't effect the value. The value of the second buyer only depends on $s_i$: $v_2$ is set to be $0$ when $s_i = 1$, while $v_i = 1$ when $s_i = 0$.

The key idea of this example is that we have $Bin(s) > \alpha_i$ when $s_i$ is $0$, and $Bin(s) < \alpha_i$ when $s_i = 1$. Therefore, if we set $p_1 = Bin(s)$ and $p_2 = 1$, we can receive $v_2$ when $v_2$ is $1$, and otherwise $v_1$. Since $s$ is generated uniformly at random, the expected revenue is $\frac{3}{4}T$. However, for any online algorithm, it only knows that the value of $v_2$ is $0$ or $1$ with probability $\frac{1}{2}$. Therefore, it can only get revenue $\frac{1}{2}$ in expectation and the maximum total revenue is $\frac{1}{2}T$. 
\end{proof}

\subsection*{Acknowledgements}
We are thankful to Thomas Kesselheim as the project was initiated in discussions with him. We also thank the anonymous reviewers who helped greatly improve the presentation of the paper.








\clearpage
\begin{small}
\bibliographystyle{alpha}
\bibliography{ref.bib,bib.bib}
\end{small}

\end{document}